\documentclass[11pt, onecolumn, journal,compsoc]{IEEEtran}

\usepackage{hyperref}
\hypersetup{%
   bookmarksnumbered, bookmarksopen=true, bookmarksopenlevel=1,%
}

\usepackage{amsmath,amssymb,amsthm,bm}

\usepackage{stmaryrd}
\usepackage{epsfig,calc,graphicx}
\usepackage{wasysym,pifont}
\usepackage[usenames]{color}
\usepackage{cleveref}
\usepackage{enumerate}
\usepackage{mathtools}

\usepackage{makecell}

\usepackage{rotating}

\usepackage{scalerel}

\usepackage[lined,algoruled,algosection]{algorithm2e}

\usepackage[printonlyused]{acronym}

\usepackage{dcolumn}
\usepackage{longtable}
\usepackage{multirow}

\usepackage{booktabs}

\usepackage{subfig}

\usepackage{tikz}
\usepackage{pgfplots}
\usepackage{grffile}
\pgfplotsset{compat=newest}
\usetikzlibrary{plotmarks}
\usetikzlibrary{arrows.meta}
\usepgfplotslibrary{patchplots}

\title{Compressed Dictionary Learning}

\author{Karin Schnass and Flavio Teixeira
\thanks{Department of Mathematics, University of Innsbruck, Technikerstra\ss e 13, 6020 Innsbruck, Austria. \newline Email: karin.schnass@uibk.ac.at, eng.flavio.teixeira@gmail.com}
}

\newcommand\ip[2]{\langle #1, #2\rangle}

\newcommand\natoms{K}

\newcommand\coherence{\mu}

\newcommand\nsig{N}
\newcommand\sparsity{S}
\newcommand\ddim{d}

\newcommand\trainsignal{\bm{y}}
\newcommand\trainmatrix{\bm{Y}}
\newcommand\ctrainmatrix{\widetilde{\bm{Y}}}

\newcommand\eps{\varepsilon}

\newcommand\targeterror{\tilde{\eps}}

\newcommand\dico{\bm{\Phi}}

\newcommand\atom{\bm{\phi}}

\newcommand\Pert{\Delta}

\newcommand\pdico{\bm{\Psi}}
\newcommand\ppdico{\bar{\bm{\Psi}}}
\newcommand\cdico{\widetilde{\bm{\Psi}}}

\newcommand\patom{\bm{\psi}}
\newcommand\patomn{\bar{\bm{\psi}}}
\newcommand\ppatom{\bar{\bm{\psi}}}

\newcommand\noise{\bm{r}}

\newcommand\nsigma{\rho}

\newcommand\csum{{\beta}}

\newcommand\epsopt{\eps_{\coherence,\nsigma,\jldistorterm}}
\newcommand\jlmatrix{\bm{\Gamma}}
\newcommand\jlmatrixnormfactor{\varrho}
\newcommand\jlprobterm{\eta}

\newcommand\jldistorterm{\delta}
\newcommand\jlapproxmultfactor{\left(1 \pm \jldistorterm \right)}
\newcommand{\rademachermatrix}{\bm{\Pi}}

\newcommand{\permmatrix}{\bm{P}}
\newcommand{\jlnumpoints}{p}
\newcommand{\jlembeddingdim}{m}

\newcommand\fsttestnumber{u}
\newcommand\sndtestnumber{v}
\newcommand\thirdtestnumber{w}

\newcommand\fsttestvector{\bm{u}}
\newcommand\sndtestvector{\bm{v}}

\newcommand\fsttestmatrix{\bm{U}}

\newcommand\ripmatrix{\bm{\Upsilon}}
\newcommand\isometryconst{\upsilon}

\newcommand\orthogonalmatrix{\bm{Q}}

\newcommand\coeffsequence{\bm{c}}
\newcommand\coeffsequenceterm{c}

\newcommand\coeffvector{\bm{x}}
\newcommand\coeffmatrix{\bm{X}}

\newcommand\permmapping{p}

\newcommand\signsequence{\bm{\sigma}}

\newcommand\snrinputsig{C_{\noise}}
\newcommand\statinputone{\gamma_{1,\sparsity}}
\newcommand\statinputotwo{\gamma_{2,\sparsity}}

\newcommand{\dicodecompfstvectorcoord}{\alpha}

\newcommand{\dicodecompsndvectorcoord}{\omega}

\newcommand{\unitatom}{\bm{z}}
\newcommand{\pertdico}{\bm{Z}}

\newcommand{\indexset}{\mathcal{I}}
\newcommand{\indicator}{1}

\newcommand{\threshpdicosupp}{\indexset^{t}_{n}}

\newcommand{\compthreshpdicosuppn}{\indexset^{ct}_{n}}
\newcommand{\compthreshpdicosupp}{\indexset^{ct}}

\newcommand{\oraclesuppn}{\indexset^{o}_n}
\newcommand{\oraclesupp}{\indexset^{o}}

\newcommand{\threshsetatom}{\mathcal{E}_{\jlmatrix \atom}}
\newcommand{\threshsetunitatom}{\mathcal{E}_{\jlmatrix \unitatom}}
\newcommand{\signsetatom}{\mathcal{E}_{\atom}}
\newcommand{\signsetunitatom}{\mathcal{E}_{\unitatom}}

\newcommand{\thresholdbad}{\mathcal{F}_{\jlmatrix}}
\newcommand{\jlgood}{\mathcal{G}_{\jldistorterm}}
		
\newcommand{\paramprobestimate}{\tau}

\newcommand\frameupbound{B}

\newcommand\mask{\bm{M}}

\newcommand\signop{\operatorname{sign}}
\newcommand\projop{\operatorname{P}}

\newcommand{\R}{{\mathbb{R}}}

\newcommand{\E}{{\mathbb{E}}}
\newcommand{\I}{{\mathbb{I}}}
\newcommand{\K}{\mathbb{K}}

\renewcommand{\P}{{\mathbb{P}}}

\theoremstyle{plain}
\newtheorem{theorem}{Theorem}[section]

\newtheorem{lemma}[theorem]{Lemma}

\newtheorem{example}{Example}[section]
\newtheorem{remark}[example]{Remark}

\newlength\figureheight 
\newlength\figurewidth 

\newcommand{\defaultheight}{4cm}
\newcommand{\defaultwidth}{6cm}

\DeclarePairedDelimiter\ceil{\lceil}{\rceil}

\acrodef{itkm}[ITKM]{Iterative Thresholding and K-Means}
\acrodef{itkrm}[ITKrM]{Iterative Thresholding and K-residual-Means}
\acrodef{ictkm}[IcTKM]{Iterative compressed-Thresholding and K-Means}
\acrodef{jl}[JL]{Johnson-Lindenstrauss}
\acrodef{rip}[RIP]{Restricted Isometry Property}
\acrodef{jlt}[JLT]{Johnson-Lindenstrauss Transform}
\acrodef{ct}[CT]{Compressed Thresholding}
\acrodef{dct}[DCT]{Discrete-Cosine Transform}
\acrodef{dft}[DFT]{Discrete-Fourier Transform}
\acrodef{fft}[FFT]{Fast-Fourier Transform}
\acrodef{crt}[CRT]{Circulant-Rademacher Transform}
\acrodef{cgt}[CRT]{Circulant-Gaussian Transform}
\acrodef{snr}[SNR]{Signal-to-Noise Ratio} 
\acrodef{cpu}[CPU]{Central Processing Unit} 
\acrodef{gpu}[GPU]{Graphics Processing Unit} 
\acrodef{rwc}[RWC]{Real World Computing}
\acrodef{nmf}[NMF]{Non-Negative Matrix Factorization}
\acrodef{mdl}[MDL]{Masked Dictionary Learning}

\acused{itkm}
\acused{rwc}
\acused{cpu}
\acused{gpu}
\acused{crt}
\acused{dft}
\acused{fft}
\acused{dct}

\crefname{secinapp}{appendix}{appendices}
\Crefname{secinapp}{Appendix}{Appendices}

\crefname{algocf}{alg.}{algs.}
\Crefname{algocf}{Algorithm}{Algorithms}

\makeatletter
\newcommand{\strong}[1]{\@strong{#1}}
\newcommand{\@@strong}[1]{\textit{\textbf{\let\@strong\@@@strong#1}}}
\newcommand{\@@@strong}[1]{\textnormal{\let\@strong\@@strong#1}}
\let\@strong\@@strong
\makeatother

\begin{document}

\maketitle

%%%%%%%%%%%%%%%%%%%%%%%%%%%%%%%%%%%%%%%%%%%%%%%%%%%%
%%%%%%%%%%%%%%%%%%%%%ABSTRACT%%%%%%%%%%%%%%%%%%%%%%%
%%%%%%%%%%%%%%%%%%%%%%%%%%%%%%%%%%%%%%%%%%%%%%%%%%%%
\begin{abstract}
In this paper we show that the computational complexity of the \ac{itkrm} algorithm for dictionary learning can be significantly reduced by using dimensionality-reduction techniques based on the Johnson-Lindenstrauss lemma. The dimensionality reduction is efficiently carried out with the fast Fourier transform. We introduce the \ac{ictkm} algorithm for fast dictionary learning and study its convergence properties. We show that \ac{ictkm} can locally recover an incoherent, overcomplete generating dictionary of $\natoms$ atoms from training signals of sparsity level $\sparsity$ with high probability. Fast dictionary learning is achieved by embedding the training data and the dictionary into $\jlembeddingdim < \ddim$ dimensions, and recovery is shown to be locally stable with an embedding dimension which scales as low as $m = O(\sparsity \log^4\sparsity \log^3 \natoms)$. The compression effectively shatters the data dimension bottleneck in the computational cost of \ac{itkrm}, reducing it by a factor $O(m/d)$.
Our theoretical results are complemented with numerical simulations which demonstrate that \ac{ictkm} is a powerful, low-cost algorithm for learning dictionaries from high-dimensional data sets.
 \end{abstract}
 

%%%%%%%%%%%%%%%%%%%%%%%%%%%%%%%%%%%%%%%%%%%%%%%%%%%%
%%%%%%%%%%%%%%%%%%%%%INTRODUCTION%%%%%%%%%%%%%%%%%%%
%%%%%%%%%%%%%%%%%%%%%%%%%%%%%%%%%%%%%%%%%%%%%%%%%%%%
\section{Introduction}\label{sec:intro}

%%%% Organization

Low complexity models of high-dimensional data lie at the core of many efficient solutions in modern signal processing. One such model is that of \emph{sparsity} in a dictionary, where every signal in the data class at hand has a sparse expansion in a predefined basis or frame. In mathematical terms we say that there exists a set of $\natoms$ unit-norm vectors $\atom_k \in \mathbb{R}^{\ddim}$ referred to as \emph{atoms}, such that every signal $\trainsignal \in \mathbb{R}^{\ddim}$ can be approximately represented in the \emph{dictionary} $\dico = (\atom_1, \ldots, \atom_\natoms)$ as $\trainsignal \approx \sum_{i \in \indexset} \coeffvector(i) \atom_i,$ where $\indexset$ is an index set and $\coeffvector \in \R^{\natoms}$ is a sparse coefficient vector with $|\indexset| = \sparsity$ and $\sparsity \ll \ddim$. 

A fundamental question associated with the sparse model is how to find a suitable dictionary providing sparse representations. When taking a learning rather than a design approach this problem is known as \emph{dictionary learning} or sparse component analysis. In its most general form, dictionary learning can be seen as a matrix factorization problem. Given a set of $\nsig$ signals represented by the $\ddim \times \nsig$ data matrix $\trainmatrix = (\trainsignal_1,\ldots,\trainsignal_\nsig)$, decompose it into a $\ddim \times \natoms$ dictionary matrix $\dico$ and a $\natoms \times \nsig$ coefficient matrix $\coeffmatrix = (\coeffvector_1, \ldots, \coeffvector_{\nsig})$; in other words, find $\trainmatrix =\dico \coeffmatrix$ where every coefficient vector $\coeffvector_k$ is sparse. Since the seminal paper by Olshausen and Field, \cite{olsfield96}, a plethora of dictionary learning algorithms have emerged, see \cite{ahelbr06, enaahu99, kreutz03, lese00, mabaposa10, sken10, mabapo12, memathva16}, and also theory on the problem has become available, \cite{grsc10, spwawr12, argemo13, aganjaneta13, sc14, sc14b, bagrje14, bakest14, argemamo15, sc15, suquwr17a, suquwr17b}. For an introduction to the origins of dictionary learning including well-known algorithms see \cite{rubrel10}, while pointers to the main theoretical results can be found in \cite{sc15imn}.

Fast dictionary learning algorithms that are applicable to large-scale problems have been proposed in \cite{memathva18,memathva16,argemamo15,sc15,mabaposa10}. In large-scale problems we are typically required to learn dictionaries from a massive number of training signals, potentially in high dimensions. When $\nsig$ is large, fast learning algorithms avoid operating explicitly with the data matrix $\trainmatrix$ by following an \emph{online-learning model}; in other words, they process the data set $\trainmatrix$ column-by-column in a serial fashion without having the entire matrix $\trainmatrix$ available from the start of the learning task. When $\ddim$ is large, the computational cost of processing a single high-dimensional signal can get prohibitively expensive, even in an online-learning model. Random subsampling has been proposed in \cite{memathva16,memathva18} to deal with the increase in computational cost of dictionary learning in high dimensions. In random-subsampling-based learning algorithms, the computational cost of processing high-dimensional data is reduced by applying a random mask to each column of the data matrix $\trainmatrix$, such that a random subset of the entries in the column is observed instead of all the entries. Subsampling schemes require the target dictionary to be incoherent with the Dirac basis. Further, despite the emergence of practical learning algorithms that are applicable to large-scale problems, so far there exist no identification results for fast dictionary learning algorithms which scale well both in the large $\ddim$ and $\nsig$ parameter regime.

Here we take a step towards increasing the computational efficiency of dictionary learning for high-dimensional training data. We introduce an online learning algorithm that scales well with both $\ddim$ and $\nsig$, and can identify any dictionary with high probability, not only those incoherent with the Dirac basis. On top of that, we provide a theoretical analysis of its local convergence behavior. We will use the \ac{itkrm} algorithm introduced in \cite{sc15} as our starting point. For a target recovery error $\targeterror$ and noisy signals with sparsity levels $\sparsity \leq O(\ddim/(\log (\natoms^2/\targeterror) )$ and \ac{snr} of order $O(1)$, it has been shown that except with probability $O(\jlprobterm \log(1/\targeterror))$ \ac{itkrm} will recover the generating dictionary up to $\targeterror$ from any input dictionary within radius $O(1/\sqrt{\max \{ \sparsity, \log \natoms \}})$ in $O(\log(1/\targeterror))$ iterations, as long as in each iteration a new batch of $\nsig = O(\targeterror^{-2}\natoms \log(\natoms/\jlprobterm))$ training signals is used. In particular, for sparsity levels $\sparsity \leq \natoms^{2/3}$ the computational cost of dictionary recovery with \ac{itkrm} scales as $O(\ddim \natoms \nsig \log(\targeterror^{-1}))$.
 
\textbf{Main Contribution.} In this work we show that the signal dimension bottleneck in the computational cost of \ac{itkrm} can be shattered by using randomized dimensionality reduction techniques. Our main technical tool is a result due to Johnson and Lindenstrauss in \cite{joli84}, which shows that it is possible to map a fixed set of data points in a high-dimensional space to a space with lower dimension with high probability, while preserving the pairwise distances between the points up to a prescribed distortion level. We introduce the \ac{ictkm} algorithm for fast dictionary learning, where we exploit recent constructions of the Johnson-Lindenstrauss embeddings using Fourier or circulant matrices, \cite{aili08,aich09,krwa11,hare17,krmera14}, to efficiently embed both the training signals and the current dictionary. The inner products between embedded signals and atoms are then used as approximations to the exact inner products in the thresholding step.\\
We study the local convergence behavior of \ac{ictkm} and prove a similar result as for \ac{itkrm}.
For a target recovery error $\targeterror$ and noisy signals with sparsity levels $\sparsity \leq O(\ddim/(\log (\natoms^2/\targeterror) )$ and \ac{snr} of order $O(1)$, we will show that except with probability $O(\jlprobterm \log(1/\targeterror))$ \ac{ictkm} will recover the generating dictionary up to $\targeterror$ from any input dictionary within radius $O(1/\sqrt{\max \{ \sparsity, \log \natoms \}})$ in $O(\log(1/\targeterror))$ iterations, as long as in each iteration a new batch of $\nsig = O(\targeterror^{-2}\natoms \log(\natoms/\jlprobterm))$ training signals is used and a new embedding is drawn with the embedding dimension $\jlembeddingdim$ satisfying $\jlembeddingdim\geq O(\sparsity \log^4(\sparsity) \log^3(\natoms/\jlprobterm))$. 
This means that the computational cost of dictionary recovery with \ac{ictkm} can be as small as $O(\sparsity \natoms \nsig \log(\targeterror^{-1})\log^4(\sparsity) \log^3(\natoms/\jlprobterm))$. \\
We further show with several numerical experiments on synthetic and audio data that \ac{ictkm} is a powerful, low-cost algorithm for learning dictionaries from high-dimensional data sets.

\textbf{Outline.} The paper is organized as follows. In \Cref{sec:notation_preliminaries} we introduce notation and define the sparse signal model used to derive our convergence results. The proposed algorithm is presented in \Cref{sec:ictkm}, where its convergence properties and computational complexity are studied in detail. In \Cref{sec:simulations} we present numerical simulations on synthetic and audio training data to illustrate the ability of our proposed algorithm for learning dictionaries with fairly low computational cost on realistic, high-dimensional data sets. Lastly, in \Cref{sec:conclusions} we present conclusions and discuss possible future research directions.

%%%%%%%%%%%%%%%%%%%%%%%%%%%%%%%%%%%%%%%%%%%%%%%%%%%%
%%%%%%%%%%%%%%%%%%%%%NOTATIONS%%%%%%%%%%%%%%%%%%%%%%
%%%%%%%%%%%%%%%%%%%%%%%%%%%%%%%%%%%%%%%%%%%%%%%%%%%%
\section{Notation and Signal Model}\label{sec:notation_preliminaries} 

Before we hit the strings, we will fine tune the notation and introduce some definitions. Regular letters will denote numbers as in $\fsttestnumber \in \R$. For real numbers $\fsttestnumber, \sndtestnumber, \thirdtestnumber \geq 0$, we use the notation $\fsttestnumber \lessgtr \left(1 \pm \thirdtestnumber \right)\sndtestnumber$ to convey that $\fsttestnumber \in \left[ (1-\thirdtestnumber)v, (1+\thirdtestnumber)\sndtestnumber \right]$. Lower-case bold letters denote vectors while upper-case bold letters are reserved for matrices, e.g., $\fsttestvector \in \R^{\ddim}$ vs. $\fsttestmatrix \in \R^{\ddim \times \ddim}$. For a vector $\fsttestvector$ we use $\fsttestvector(k)$ to denote its $k$th coordinate. The supremum norm of $\fsttestvector$ is defined by $\left\| \fsttestvector \right\|_{\infty} = \max_{k} | \fsttestvector(k) | $. For a matrix $\fsttestmatrix$, we denote its conjugate transpose by $\fsttestmatrix^*$ and its Moore-Penrose pseudo-inverse by $\fsttestmatrix^\dagger$. The operator norm of $\fsttestmatrix$ is defined by $\left\| \fsttestmatrix \right\|_{2,2} = \max_{\left\|\sndtestvector \right\|_2 = 1} \left\| \fsttestmatrix \sndtestvector \right\|_2$. We say that a matrix $\fsttestmatrix \in \R^{\jlembeddingdim \times \ddim}$ with $\jlembeddingdim < \ddim$ has the \ac{rip} of order $k$ and level $\isometryconst$, or the shorthand $(k,\isometryconst)$-\ac{rip}, if for all $k$-sparse vectors $\sndtestvector \in \R^{\ddim}$ we have $\left\| \fsttestmatrix \sndtestvector \right\|_2^2 \lessgtr (1 \pm \isometryconst) \left\| \sndtestvector \right\|_2^2$, see \cite{cata05} for details. Upper-case calligraphy letters will denote sets; specifically, we let $\indexset$ denote an index set and use the notation $\fsttestmatrix_{\indexset}$ to convey the restriction of matrix $\fsttestmatrix$ to the columns indexed by $\indexset$, e.g., $\fsttestmatrix_{\indexset} = (\fsttestvector_{i_1}, \fsttestvector_{i_2}, \ldots, \fsttestvector_{i_n})$ with $i_{j} \in \indexset$ and $|\indexset| = n$ for some integer $n$. Further, we denote by $\projop(\fsttestmatrix_\indexset)$ the orthogonal projection onto the span of the columns indexed by $\indexset$, i.e., $\projop(\fsttestmatrix_\indexset) = \fsttestmatrix_\indexset \fsttestmatrix_\indexset^\dagger$. For a set $\mathcal{V}$, we refer to $\indicator_{\mathcal{V}}(\cdot)$ as its indicator function, such that $\indicator_{\mathcal{V}}(\sndtestnumber)$ is equal to one if $\sndtestnumber \in \mathcal{V}$ and zero otherwise.

The $\ddim \times \natoms$ matrix $\dico = (\atom_1, \atom_2, \ldots, \atom_\natoms)$ denotes the \emph{generating dictionary}, which we define as a collection of $\natoms$ unit-norm vectors $\atom_k \in \R^{\ddim}$, also referred to as \emph{atoms}, and $\ddim$ is the \emph{ambient} dimension. The rank of $\dico$ is called the \emph{intrinsic} dimension $\tilde{\ddim}$ of the dictionary, and we have $\tilde{\ddim} \leq K$. Since we are dealing with high dimensional data, we will not only consider the usual case of overcomplete dictionaries with $d < K$ but, especially in the numerical simulations, also the undercomplete dictionaries with $K < d$. The maximal inner-product (in magnitude) between two different atoms of the generating dictionary is called the \emph{coherence} $\coherence := \max_{k \neq j} | \langle \atom_k, \atom_j \rangle |$. The dictionary $\dico$ will be used to generate our training signals as follows
\begin{equation}
\trainsignal = \frac{\dico \coeffvector + \noise}{\sqrt{1+ \left\| \noise \right\|_2^2}},
\label{eq:signal_model}
\end{equation}
where $\coeffvector \in \R^{\natoms}$ is a sparse coefficient vector and $\noise \in \R^{\ddim}$ represents noise. By collecting a set of $\nsig$ training signals $\trainsignal_n \in \R^{\ddim}$ generated as in \labelcref{eq:signal_model}, we form our training data set as $\trainmatrix = (\trainsignal_1, \trainsignal_2, \ldots, \trainsignal_\nsig)$. We refer to the number of training signals $\nsig$ required to recover a dictionary with high probability as the \emph{sample complexity}. 

We will model the sparse coefficient vector $\coeffvector$ as a random permutation of a randomly chosen sequence provided with random signs, where the $\sparsity$-largest entry of the randomly chosen sequence is greater than the $(\sparsity+1)$-largest entry in magnitude. Formally, let $\mathcal C$ denote a subset of all positive non-increasing, unit-norm sequences, i.e., for a $\coeffsequence \in \mathcal{C}$ we have $\coeffsequence(1) \geq \coeffsequence(2) \geq \ldots \coeffsequence(\natoms) \geq 0$ with $\left\| \coeffsequence \right\|_2 = 1$, endowed with a probability measure $\nu_c$. To model $S$-sparsity we assume that almost $\nu_c$-surely we have 
\begin{equation}
\coeffsequence(\sparsity) - \coeffsequence(\sparsity + 1) \geq \csum_{\sparsity} \hspace{2em} \text{and} \hspace{2em} \frac{\coeffsequence(\sparsity) - \coeffsequence(\sparsity+ 1)}{\coeffsequence(1)} \geq \Pert_\sparsity,
\label{eq:relative_absolute_gap_definition}	
\end{equation}
where $\csum_{\sparsity} > 0$ is called the \emph{absolute gap} and $\Pert_\sparsity$ the \emph{relative gap} of our coefficient sequence. To choose the coefficients $\coeffsequence$, we first draw them according to $\nu_c$, then draw a permutation $\permmapping$ as well as a sign sequence $\signsequence \in \left\{ -1, 1 \right\}^{\natoms}$ uniformly at random and set $\coeffvector = \coeffvector_{\coeffsequence,\permmapping,\signsequence}$ in \labelcref{eq:signal_model}, where $\coeffvector_{\coeffsequence,\permmapping,\signsequence}(k) = \signsequence(k) \coeffsequence(\permmapping(k))$. With this notation the signal model then takes the form
\begin{equation}
\trainsignal = \frac{\dico \coeffvector_{\coeffsequence,\permmapping,\signsequence} + \noise}{\sqrt{1 + \left\| \noise \right\|_2^2}}.
\label{eq:training_signal_with_model}	
\end{equation}
For most of our derivations it will suffice to think of the sparse coefficient vector $\coeffvector$ as having exactly $\sparsity$ randomly distributed and signed, equally sized non-zero entries; in other words, $\mathcal C$ contains one sequence $\coeffsequence$ with $\coeffsequence(k) = 1/\sqrt{S}$ for $k\leq S$ and $\coeffsequence(k) = 0$ for $k>S$, so we have $\csum_{\sparsity} = 1/\sqrt{S}$ and $\Pert_\sparsity=1$. In particular, the main theorem of the paper will be a specialization to this particular coefficient distribution, while the more detailed result that addresses general sequence sets $\mathcal C$ is deferred to the appendix. For the numerical simulations we will again use only exactly $\sparsity$-sparse sequences with $\coeffsequence(k) = 0$ for $k>S$. However, for $k\leq S$ they will form a geometric sequence whose exact generation will be discussed in the relevant section.

The noise vector $\noise$ is assumed to be a centered subgaussian vector with parameter $\nsigma$ independent of $x$, that is, $\E(\noise) = 0$, and for all vectors $\fsttestvector \in \R^{d}$ the marginals $\langle \fsttestvector, \noise \rangle$ are subgaussian with parameter $\nsigma$ such that we have $\E(e^{t \langle \fsttestvector, \noise \rangle}) \leq \exp(\tfrac{t^2 \nsigma^2 \| \fsttestvector \|^2}{2})$ for all $t > 0$. Since $\E\bigl(\left\| \noise \right\|_2^2\bigr) \leq \ddim \nsigma^2$, with equality holding in the case of Gaussian noise, and $\E\left(\| \dico \coeffvector_{\coeffsequence,\permmapping,\signsequence} \|_2^2\right) = 1$, we have the following relation between the noise level $\nsigma$ and the \acl{snr}, $\operatorname{SNR} \geq (d\nsigma^2)^{-1}$. 
As before, for most of the paper it suffices to think of $\noise$ as a Gaussian random vector with mean zero and variance $\nsigma^2= 1/(d\cdot\operatorname{SNR})$. 

The scaling factor $(1 + \|\noise\|^2_2)^{-1/2}$ in the signal model might seem strange at first glance, but can be thought of as a first moment approximation to the normalization factor $\|\dico \coeffvector_{\coeffsequence,\permmapping,\signsequence} + \noise\|_2^{-1}$, often used in practice. It allows us
to handle unbounded noise distributions, while still being relatively simple. Note that the main argument used for the result presented here does not depend on the scaling factor. The corresponding lemma as well as the other lemmata from \cite{sc15}, which we require for the proof of our main result, can be straightforwardly modified to handle a signal model without the scaling factor, when additionally assuming boundedness of the noise. Further, with more technical effort, e.g., by using concentration of $\|\dico \coeffvector_{\coeffsequence,\permmapping,\signsequence} + \noise\|_2$ around $\sqrt{1+ \|\noise\|_2}$, the Taylor expansion of the square root, and keeping track of the small distortion coming from higher order moments, it is possible to extend the necessary lemmata in \cite{sc15} to the normalized signal model.

We will refer to any other dictionary $\pdico = (\patom_1, \patom_2, \ldots, \patom_\natoms)$ as our \emph{perturbed dictionary}, meaning that it can be decomposed into the generating dictionary $\dico$ and a perturbation dictionary $\pertdico = (\unitatom_1, \unitatom_2, \ldots, \unitatom_\natoms)$. To define this decomposition, we first consider the (asymmetric) distance of $\pdico$ to $\dico$ defined by
\begin{equation}
d(\pdico, \dico) := \max_k \min_j \left\| \atom_k \pm \patom_j \right\|_2 = \max_k \min_j \sqrt{2 - 2 \left| \langle \atom_k, \patom_j \rangle \right|}.
\label{eq:distance_dictionary_definition}	
\end{equation}
Although $d(\pdico, \dico)$ is not a metric, a locally equivalent (symmetric) version may be defined in terms of the maximal distance between two corresponding atoms, see \cite{sc15} for details. Since the asymmetric distance is easier to calculate and our results are local, we will refer to distances  between dictionaries in terms of \labelcref{eq:distance_dictionary_definition} and assume that $\pdico$ is already signed and rearranged in a way that $d(\pdico, \dico) = \max_k \left\| \atom_k - \patom_k \right\|_2$. With this distance in hand, let $\left\| \patom_k - \atom_k \right\|_2 = \eps_k$ and $d(\pdico, \dico) = \eps$, where $\max_k \eps_k = \eps$  by definition. We can write our perturbed dictionary by finding unit vectors $\unitatom_k$ with $\langle \atom_k, \unitatom_k \rangle = 0$ such that we have the decomposition
\begin{equation}
\patom_k = \dicodecompfstvectorcoord_k \atom_k + \dicodecompsndvectorcoord_k \unitatom_k, \hspace{1em} \text{for} \hspace{1em} \dicodecompfstvectorcoord_k = : 1 - \eps_k^2/2 \hspace{1em} \text{and} \hspace{1em} \dicodecompsndvectorcoord_k := \sqrt{\eps_k^2 - \eps_k^4/4}.
\label{eq:perturbed_dico_definition}	
\end{equation}
Lastly, we use the Landau symbol $O(f)$ to describe the growth of a function $f$. We have $f(t) = O(g(t))$ if $\lim_{t\rightarrow 0/\infty}{f(t)/g(t)} = C < \infty$, where $C>0$ is a constant. 
%%%%%%%%%%%%%%%%%%%%%%%%%%%%%%%%%%%%%%%%%%%%%%%%%%%%
%%%%%%%%%%%%%%%%%%%%%ICTKM ALGORITHM%%%%%%%%%%%%%%%%
%%%%%%%%%%%%%%%%%%%%%%%%%%%%%%%%%%%%%%%%%%%%%%%%%%%%

\section{Fast Dictionary Learning via \acs{ictkm}} \label{sec:ictkm}

The \ac{itkrm} algorithm is an alternating-minimization algorithm for dictionary learning which can also be interpreted as a fixed-point iteration. \ac{itkrm} alternates between $2$ steps: $(1)$ updating the sparse coefficients based on the current version of the dictionary, and $(2)$ updating the dictionary based on the current version of the coefficients. The sparse coefficients update is achieved via \emph{thresholding}, which computes the sparse support $\threshpdicosupp$ of each point $\trainsignal_n$ in the data set $\trainmatrix$ by finding the $\sparsity$-largest inner products (in magnitude) between the atoms of a perturbed dictionary and the data point as follows
\begin{equation}
\threshpdicosupp  := \arg \max_{\left| \indexset \right| = \sparsity} { \left\| \pdico_{\indexset}^{*} \trainsignal_n \right\|_1 } = \arg \max_{|\indexset|=\sparsity}{\sum_{k \in \indexset}{| \langle \patom_k, \trainsignal_n \rangle |}}.
\label{eq:thresholding}	
\end{equation}
The dictionary update, on the other hand, is achieved by computing $\natoms$ \emph{residual means} given by
\begin{equation}
\ppatom_k = \frac{1}{\nsig} \sum_{n}{\signop{(\langle \patom_k, \trainsignal_n \rangle)} \cdot   \indicator_{\threshpdicosupp}(k)     \cdot \left( \trainsignal_n - P(\pdico_{\threshpdicosupp}) \trainsignal_n + P ( \patom_k ) \trainsignal_n \right) }.
\label{eq:k_residual_means_definition}
\end{equation}
The two most computationally expensive operations in the \ac{itkrm} algorithm are the computation of the sparse support $\threshpdicosupp$ and the projection $P(\pdico_{\threshpdicosupp}) \trainsignal_n$. If we consider one pass of the algorithm on the data set $\trainmatrix$, then finding the sparse support of $\nsig$ signals via thresholding entails the calculation of the matrix-product $\pdico^{*}\trainmatrix$ of cost $O(\ddim \natoms \nsig)$. To compute the $\nsig$ projections, on the other hand, we can use the eigenvalue decomposition of $\pdico^{*}_{\threshpdicosupp} \pdico_{\threshpdicosupp}$ with total cost $O(\sparsity^3 \nsig)$. Stable dictionary recovery with \ac{itkrm} can be achieved for sparsity levels up to $\sparsity = O(\coherence^{-2}/\log\natoms) \approx O (\ddim / \log \natoms)$, see \cite{sc15} for details, but in practice recovery is carried out with much lower sparsity levels where thresholding becomes the determining complexity factor. Conversely, the projections would dominate the computations only for impractical sparsity levels $\sparsity \geq  \natoms^{2/3}$. We will concentrate our efforts on the common parameter regime where thresholding is the computational bottleneck for stable dictionary recovery.

Although the cost $O(\ddim \natoms \nsig)$ incurred by thresholding $\nsig$ signals is quite low compared to the computational cost incurred by other popular algorithms such as the K-SVD algorithm \cite{ahelbr06}, learning dictionaries can still be prohibitively expensive from a computational point of view when the ambient dimension is large. Our goal here is to shatter the ambient dimension bottleneck in thresholding by focusing on dimensionality-reduction techniques, which will allow us to address real-world scenarios that require handling high-dimensional data in the learning process. 

\subsection{Speeding-up Dictionary Learning}\label{sec:comp_dico_learn}

Our main technical tool for speeding-up \ac{itkrm} is a \emph{dimensionality reduction} result due to Johnson and Lindenstrauss \cite{joli84}. This key result tells us that it is possible to embed a finite number of points from a high-dimensional space into a lower-dimensional one, while preserving the relative distances between any two of these points by a constant distortion factor. Say we want to embed a set $\mathcal{X} \in \R^{\ddim}$ of $|\mathcal{X}| = \jlnumpoints$ points into $\jlembeddingdim < \ddim$, where $\jlembeddingdim$ is the embedding dimension. By Lemma 4 in \cite{joli84}, there exists a \ac{jl} mapping $\jlmatrix : \R^{\ddim} \rightarrow \R^{\jlembeddingdim}$ with $\jlembeddingdim \geq O \left( \delta^{-2} \log \jlnumpoints \right)$, where $\delta \in (0, 1/2)$ is the embedding distortion, such that 
\begin{equation}
\left\| \jlmatrix(\fsttestvector)-\jlmatrix(\sndtestvector) \right\|_2^2 \lessgtr \jlapproxmultfactor \left\| \fsttestvector-\sndtestvector \right\|_2^2, \qquad \forall \fsttestvector,\sndtestvector \in \mathcal{X}.
\label{eq:jl_lemma}
\end{equation}
Further, we know from \cite{aili08,aich09,krwa11,hare17,krmera14} that the \ac{jl} mapping $\jlmatrix : \R^{\ddim} \rightarrow \R^{\jlembeddingdim}$ in \labelcref{eq:jl_lemma} can be realized with probabilistic matrix constructions where the embedding dimension is on par with the bound $\jlembeddingdim \geq O \left( \delta^{-2} \log \jlnumpoints \right)$ up to logarithmic factors, and fast algorithms for matrix-vector multiplication can be used to reduce the computational cost to embed the points. A precise definition of random matrices with these nice dimensionality-reduction properties is now given.
\begin{theorem}[Fast \ac{jl}-Embeddings \cite{krwa11}] \label{def:fast_jl_embedding}
Let $\jlmatrix \in \mathbb{R}^{\jlembeddingdim \times \ddim}$ be of the form $\jlmatrix = \jlmatrixnormfactor \ripmatrix \rademachermatrix$, where $\jlmatrixnormfactor = \sqrt{\ddim/\jlembeddingdim}$ is a normalization factor, $\rademachermatrix \in \mathbb{R}^{\ddim \times \ddim}$ is a diagonal matrix with entries uniformly distributed on $\left\{ -1, 1 \right\}^\ddim$, and $\ripmatrix \in \R^{\jlembeddingdim \times \ddim}$ is obtained by drawing $\jlembeddingdim$ rows uniformly at random from a $\ddim \times \ddim$ orthogonal matrix, e.g., discrete Fourier/cosine unitary matrices, or a $\ddim \times \ddim$ circulant matrix in which the first row is a Rademacher random vector multiplied by a $1/\sqrt{\ddim}$ normalization factor, and the subsequent rows are cyclic permutations of this vector. If ${\jlembeddingdim \geq O\bigl( \jldistorterm^{-2} \cdot \log^2(\jldistorterm^{-1}) \cdot \log(\jlnumpoints/\jlprobterm) \cdot \log^2 \bigl( \jldistorterm^{-1} \log(\jlnumpoints/\jlprobterm) \bigr) \cdot \log {\ddim} \bigr)}$ or ${\jlembeddingdim \geq O\bigl( \jldistorterm^{-2} \cdot \log(\jlnumpoints/\jlprobterm) \cdot \log^2\bigl(\log(\jlnumpoints/\jlprobterm)\bigr) \cdot \log^2{\ddim} \bigr)}$ for $\ripmatrix$ obtained from orthogonal or circulant matrices, respectively, then \labelcref{eq:jl_lemma} holds with probability exceeding $(1 - \jlprobterm)$.
\begin{proof} \normalfont The proof is a direct consequence of Theorem 3.1 in \cite{krwa11}. Since $\ripmatrix$ has the $(k,\isometryconst)$-\ac{rip} with high probability when drawing $\jlembeddingdim$ rows at random from a Fourier/cosine matrix for $\jlembeddingdim \geq O( \log^2(1/\isometryconst) \cdot \isometryconst^{-2} \cdot k \cdot \log^2(k/\isometryconst) \cdot \log{\ddim})$ (see Theorem 4.5 in \cite{hare17}), and a circulant matrix for  $\jlembeddingdim \geq O( \isometryconst^{-2} \cdot k \cdot \log^2{k} \cdot \log^2{\ddim} )$ (see Theorem 1.1 in \cite{krmera14}), it then follows from Theorem 3.1 in \cite{krwa11} that the \ac{jl} property in \labelcref{eq:jl_lemma} holds for $\jlmatrix$ with probability exceeding $(1-\jlprobterm)$ for ${\jlembeddingdim \geq O\bigl( \jldistorterm^{-2} \cdot \log^2(\jldistorterm^{-1}) \cdot \log(\jlnumpoints/\jlprobterm) \cdot \log^2 \bigl( \jldistorterm^{-1} \log(\jlnumpoints/\jlprobterm) \bigr) \cdot \log {\ddim} \bigr)}$ in the orthogonal case, and ${\jlembeddingdim \geq O\bigl( \jldistorterm^{-2} \cdot \log(\jlnumpoints/\jlprobterm) \cdot \log^2\bigl(\log(\jlnumpoints/\jlprobterm)\bigr) \cdot \log^2{\ddim} \bigr)}$ in the circulant case.
\end{proof}
\end{theorem}
\begin{remark}[Computational Complexity] \label{rem:jl_computational_complexity}
Note that all fast \ac{jl}-embeddings defined above can be decomposed as $\jlmatrix = \jlmatrixnormfactor \permmatrix_{\indexset} \orthogonalmatrix \rademachermatrix$ with $|\indexset| = m$, where $\permmatrix_{\indexset} \in \R^{m \times \ddim}$ is a projection onto the indices in $\indexset$, and $\orthogonalmatrix \in \R^{\ddim \times \ddim}$ is either an orthogonal or a circulant matrix. From this decomposition, we can see that the embedding cost is dominated by the action of $\orthogonalmatrix$. If $\orthogonalmatrix$ is a circulant, or a Fourier/cosine matrix, then $\orthogonalmatrix$ admits fast matrix-vector multiplication via the \ac{fft} and the cost of embedding a point is of order $O(\ddim \log \ddim)$. Therefore, from now on we will always think of $\jlmatrix$ as being based on one of these three constructions.
\end{remark}

\begin{remark}[Operator Norm] \label{rem:jl_operator_norm} The advantage of an embedding based on an orthogonal $\orthogonalmatrix$, meaning a Fourier or cosine matrix, is that the operator norm of $\jlmatrix$ is bounded by $\jlmatrixnormfactor$ since the operator norms of all three factors, $\permmatrix_{\indexset}, \orthogonalmatrix, \rademachermatrix$ are bounded by one. In the case of a circulant $\orthogonalmatrix$, we have that its singular values correspond to the magnitudes of the \ac{dft} of its first row. The operator norm of $\orthogonalmatrix$ is, therefore, bounded by the supremum norm of the \ac{dft} of a normalized Rademacher vector which concentrates around its expectation of order $O(\sqrt{\log \ddim})$, and so with high probability the operator norm of $\jlmatrix$ will be of order $O(\jlmatrixnormfactor \sqrt{\log \ddim})$. We will see that the operator norm of $\orthogonalmatrix$ directly affects our admissible noise level $\nsigma$. In particular, the circulant matrix construction reduces our admissible noise by a factor of at least $O(\log \ddim)$ compared to the orthogonal construction. For simplicity, we will state and prove only the stronger theoretical results for \ac{jl}-embeddings based on an orthogonal $\orthogonalmatrix$, but will point out which part of the proofs needs to be amended for a circulant $\orthogonalmatrix$.
\end{remark}

\subsection{The Proposed Algorithm}

We can reduce the computational cost of \ac{itkrm} by using fast embedding constructions such as those in \Cref{def:fast_jl_embedding}. Consider updating the sparse coefficients in an alternating minimization algorithm via \emph{compressed-thresholding}, which now computes the sparse support $\compthreshpdicosuppn$ of each point $\trainsignal_n$ in the data set $\trainmatrix$ by finding the $\sparsity$-largest inner products (in magnitude) between the \emph{embedded atoms} of a perturbed dictionary and the \emph{embedded data point} as follows
\begin{equation}
\compthreshpdicosuppn := \arg \max_{|\indexset|=\sparsity}{\left\| \pdico^*_{\indexset} \jlmatrix^{*}\jlmatrix \trainsignal_n \right\|_1} = \arg \max_{|\indexset|=\sparsity}{\sum_{k \in \indexset}{| \langle \jlmatrix \patom_k, \jlmatrix \trainsignal_n \rangle |} }.
\label{eq:compressed_thres_def}	
\end{equation}
By replacing the thresholding operation of the \ac{itkrm} algorithm in \labelcref{eq:thresholding} with its compressed version in  \labelcref{eq:compressed_thres_def}, we arrive at the \ac{ictkm} algorithm, see \Cref{alg:ictkm}.
\begin{algorithm}[htp]
\BlankLine
\caption{\ac{ictkm} (one iteration)} \label{alg:ictkm}
\BlankLine

\textit{JL embedding}: draw the random matrices $\ripmatrix$ and $\rademachermatrix$, and form the dimensionality-reduction matrix $\jlmatrix$ \;

\ForEach{training signal $\trainsignal_n$ in the training set $\trainmatrix$}{
\textit{Compressed Thresholding}: $\compthreshpdicosuppn \leftarrow \arg \max_{|\indexset|=\sparsity}{\left\| \pdico^*_{\indexset} \jlmatrix^{*}\jlmatrix \trainsignal_n \right\|_1}$ \;
}
\ForEach{atom $\ppatom_k$ in $\ppdico$ with $k \in \compthreshpdicosuppn$}{
$\natoms$-\textit{residual means}: $\ppatom_{k} \leftarrow \frac{1}{\nsig} \sum_{n}{\left( \trainsignal_n - P(\pdico_{\compthreshpdicosuppn}) \trainsignal_n + P ( \patom_k ) \trainsignal_n \right) \cdot \signop{(\langle \patom_k, \trainsignal_n \rangle)} }$\;
}
\textit{Normalize the atoms}: $\ppdico \leftarrow \left( \ppatom_{1}/\| \ppatom_{1} \|_2, \dots, \ppatom_{\natoms}/\| \ppatom_{\natoms} \|_2 \right)$\;
\BlankLine

\end{algorithm}

\ac{ictkm} inherits all the nice properties of \ac{itkrm} as far as the implementation of the algorithm is concerned. It can be halted after a fixed number of iterations has been reached, and is suitable for online processing and parallelization. In particular, \Cref{alg:ictkm} may be rearranged in a way that the two inner loops are merged into a single loop that goes through the data set. In this implementation, the sparse support $\compthreshpdicosuppn$ is computed for the signal at hand and all the atoms $\ppatom_k$ for which $k \in \compthreshpdicosuppn$ are then updated as in $\ppatom_{k} \leftarrow \left( \trainsignal_n - P(\pdico_{\compthreshpdicosuppn}) \trainsignal_n + P ( \patom_k ) \trainsignal_n \right) \cdot \signop{(\langle \patom_k, \trainsignal_n \rangle)}$. The algorithm proceeds to the next signal, and the dictionary is normalized once all the signals have been processed. Since each signal can be processed independently, the learning process may be carried out in $\nsig$ independent processing nodes and thus we benefit from massive parallelization. Further, we have fairly low storage complexity requirements in this online implementation. We only need to store $O \left(\ddim (\natoms + \jlembeddingdim) \right)$ values which correspond to the input dictionary $\pdico$, the current version of the updated dictionary $\ppdico$, the \ac{jl} embedding $\jlmatrix$, and the current signal $\trainsignal_n$. Note that it is not necessary to store the data set $\trainmatrix$ in the online implementation, as this would have incurred a large storage overhead of $O(d\nsig)$ values in memory.

%TODO: See if any of below can be used 
% We will see that computing $\cdico^{*}\ctrainmatrix$ with cost $O(\jlembeddingdim \natoms \nsig)$ is the determining factor in the computational complexity of compressed thresholding because $\cdico$ and $\ctrainmatrix$ can be obtained with very-low computational cost using the \ac{fft}. Further, the admissible embedding dimension for stable dictionary recovery with \ac{ictkm} can be made as low as $m = O(\log^5 \ddim)$, omitting additional factors that do not play a significant role. Thus, the computational complexity of dictionary learning is greatly improved by using dimensionality reduction on high-dimensional data because we can achieve a reduction of up to the order $O(\log^5 \ddim/ \ddim)$ in the computational cost of stable dictionary recovery. If we set $m = O(\log^5 \ddim)$ which amounts to a compression ratio of $O( \ddim / \log^5 \ddim ) : 1$ in \labelcref{eq:compressed_thres_def}, omitting additional non-leading factors, then we can shatter the ambient dimension bottleneck and stably reduce the \ac{itkrm} dominant cost of $O(\ddim \natoms \nsig)$ down to $O(\log^5 \ddim\natoms \nsig)$. We will now proceed by stating our main convergence result, which will then allow us to address in details the computational complexity of \ac{ictkm}, and the conditions under which we can carry out dictionary learning with the highest compression ratio.

\subsection{Computational Complexity} \label{subsec:computational_complexity}  
Considering one pass of the \ac{ictkm} algorithm on the data set $\trainmatrix$, it can be seen from \labelcref{eq:compressed_thres_def} that to find the sparse support with compressed thresholding, we first need to compress the dictionary and the data set as in $\cdico = \jlmatrix \pdico$ and $\ctrainmatrix = \jlmatrix \trainmatrix$, respectively, and then compute the matrix product $\cdico^{*}\ctrainmatrix$ to find the $\sparsity$-largest inner products $\langle \jlmatrix \patom_k, \jlmatrix \trainsignal_n \rangle$ in magnitude. We know from the decomposition $\jlmatrix = \jlmatrixnormfactor \permmatrix_{\indexset} \orthogonalmatrix \rademachermatrix$ in \Cref{rem:jl_computational_complexity} that the cost of computing $\cdico^{*} = \pdico^{*} \jlmatrix^{*} $ and $\ctrainmatrix = \jlmatrix \trainmatrix$ is dominated by the action of the orthogonal (or circulant) matrix $\orthogonalmatrix$. Since $\orthogonalmatrix$ can be applied with the \ac{fft}, the cost of embedding the dictionary and the data set reduces to $O(\ddim \natoms \log \ddim)$ and $O(\ddim \nsig \log \ddim )$, respectively. Thus, the computational cost of compressed thresholding for one pass on the data set $\trainmatrix$ is of order $O\left( (\ddim \log \ddim + \jlembeddingdim \natoms ) \nsig \right)$. Further, we will see in our convergence result that in order to achieve stable dictionary recovery $\jlembeddingdim$ needs to be larger than $\log \ddim $, and thus the cost of compressed thresholding for overcomplete dictionaries simplifies to $O\left( \jlembeddingdim \natoms \nsig \right)$. \\
Comparing the cost $O (\jlembeddingdim \natoms \nsig)$ of compressed thresholding against the cost $O(\ddim \natoms \nsig)$ of regular thresholding, we can see that a significant speed-up in dictionary learning can be achieved with \ac{ictkm} if $\jlembeddingdim$ can be made small, ideally as in $\jlembeddingdim \ll \ddim$. Next, we address the convergence properties of \ac{ictkm} and answer the question of how much the data can be compressed. In particular, we present data compression ratios $(\ddim/\jlembeddingdim) : 1$ that can be reliably achieved to speed-up the learning process.

\subsection{Convergence Analysis} \label{subsec:convergence_analysis}
We now take a look at the convergence properties of \ac{ictkm} for exactly $\sparsity$-sparse training signals with randomly distributed and signed, equally sized non-zero entries. A precise convergence results for the case of approximately $\sparsity$-sparse signals with more general coefficient distributions can be found in \Cref{th:ictkm_convergence}.

\begin{theorem} \label{th:ictkm_convergence_o_notation}	
Assume that the generating dictionary $\dico \in \R^{\natoms \times \ddim}$ has operator norm $O(K/d) = O(1)$ and coherence $\mu$, and that the training signals $\trainsignal_n$ are generated following the signal model in \labelcref{eq:training_signal_with_model} with $\coeffsequence(k)= 1/\sqrt{\sparsity}$ for $k \leq \sparsity$ and $\coeffsequence(k) = 0$ for $k > \sparsity$, and with Gaussian noise of variance $\nsigma^2$. \\
Fix a target error $\targeterror >0$ and a distortion level $\jldistorterm > 0$, such that 
%\begin{equation}
%%\sparsity \leq O \left( \left(\ell (\coherence +\delta)^2 \log \natoms \right)^{-1}\right),
%(\coherence + \jldistorterm)^2 \leq O\left(\frac{1}{ \sparsity \log(\natoms^2/\targeterror) }\right) \mbox{and} , \nsigma^2 + \ddim \nsigma^2 \jldistorterm^2  \leq O\left(\frac{1}{ \sparsity \log(\natoms^2/\targeterror) }\right),
%\label{eq:admissible_sparsity_ictkm_convergence}
%\end{equation}
\begin{equation}
%\sparsity \leq O \left( \left(\ell (\coherence +\delta)^2 \log \natoms \right)^{-1}\right),
\max\left\{ \coherence,  \jldistorterm, \nsigma, \nsigma \jldistorterm\ddim^{1/2}\right\} \leq O\left(\frac{1}{ \sqrt{\sparsity \log(\natoms^2/\targeterror)} }\right).\label{eq:admissible_sparsity_ictkm_convergence}
\end{equation}
Choose a failure probability parameter $\jlprobterm >0$ and let $\jlmatrix$ be a \ac{jl} embedding constructed according to \Cref{def:fast_jl_embedding} by drawing 
%\begin{equation}
%\jlembeddingdim \geq O\left( \jldistorterm^{-2} \cdot \log^2(\jldistorterm^{-1}) \cdot \log(\natoms/\jlprobterm) \cdot \log^2 ( \jldistorterm^{-1} \log(\natoms/\jlprobterm) ) \cdot \log\ddim \right)
%\label{eq:embedding_dimension}
%\end{equation}
\begin{equation}
m \geq O \Bigl( \jldistorterm^{-2} \cdot \log^4\bigl(1/\jldistorterm\bigr) \cdot \log^2 (K/\jlprobterm)\Bigr)
\label{eq:jl_embedding_o_theorem_complete}	
\end{equation}
rows uniformly at random from a Fourier or cosine unitary matrix. %, or Rademacher (circulant) 
Then for any starting dictionary $\pdico$ within distance $O \left( 1/\sqrt{ \max \{ \sparsity, \log \natoms\} } \right)$
%\begin{equation}
% O \left( 1/\sqrt{ \max \left\{ \sparsity, \log \natoms \right\} } \right)
%\label{eq:convergence_radius_o_notation}
%\end{equation}
to the generating dictionary $\dico$, after $O( \log(1/\targeterror))$ iterations of \ac{ictkm}, each using a new batch of 
%$\nsig = O \left(\max\left\{ 1, d\nsigma^2\targeterror^{-1}\right\}\targeterror^{-1} K \log(K/\jlprobterm)\right)$ 
\begin{equation}
\nsig \geq O \left(\max\left\{ 1, d\nsigma^2/\targeterror\right\}\cdot (\natoms /\targeterror) \cdot \log(\natoms/\jlprobterm)\right)
\label{eq:sample_size}
\end{equation}
training signals $\trainmatrix$ and a new \ac{jl} embedding $\jlmatrix$, the distance of the output dictionary $\ppdico$ to the generating dictionary $\dico$ will be smaller than the target error except with probability $O( \jlprobterm \log(1/\targeterror))$.
\end{theorem}

Our first observation is that for $\delta =0$ corresponding to $m=d$, \ac{ictkm} reduces to \ac{itkrm}, and \Cref{th:ictkm_convergence_o_notation} reduces to the statement for \ac{itkrm}, see Theorem 4.2 in \cite{sc15}. We also see that the convergence radius and sample complexity are essentially not affected by the compressed thresholding operation. This is due to the fact that we use inner products between the embedded sparse signal and embedded atoms, and exploit their excess separation outside the sparse support to prove recovery, but use only the regular inner products in the update-formula. One side effect of compression is noise folding, a recurring issue in compressed sensing algorithms, see \cite{trodar11} for instance. The noise folding issue manifests itself in our admissible noise level $\nsigma^2$, which is reduced by a factor of $O(\ddim \jldistorterm^2) \approx O(\ddim/\jlembeddingdim)$ compared to \ac{itkrm}. \\
To get a better feeling for the best achievable compression, assume that $\coherence^2 = O(1/\ddim)$ and that the \ac{snr} is 1, corresponding to $\nsigma^2 =1/d$. If we are satisfied with a moderate target error $\targeterror^2 =1/\natoms$ and a final failure probability $\log \natoms/\natoms$, meaning $\jlprobterm =1/\natoms$, which results in a reasonable batch size of $\nsig = O(\natoms^2 \log \natoms)$ signals per iteration, the condition on the distortion level essentially becomes $\jldistorterm^2 \leq O \left(1/(\sparsity \log \natoms )\right)$ and we get for the best possible compression ratio an embedding dimension as low as (omitting $\log \log$ factors)
\begin{equation}
\jlembeddingdim = O \Bigl( \sparsity \cdot \log^4\sparsity \cdot \log^3 \natoms \Bigr).
\label{eq:jl_embedding_best_compression}	
\end{equation}
This means that up to logarithmic factors the embedding dimension necessary for local convergence of compressed dictionary learning scales as the sparsity level, and thus is comparable to the embedding dimension for compressed sensing of sparse signals. It also shows that the cost of dictionary learning per signal can be significantly reduced from $O(\ddim \natoms)$ to $O(\sparsity \natoms \log^4\sparsity \log^3 \natoms)$.\\
%{\color{red}
%The side effect of compression is noise folding, a recurring issue in compressed sensing algorithms, see \cite{trodar11} for instance. The noise folding issue manifests itself in our admissible noise level $\nsigma^2$, which is reduced by a factor of $O(\ddim \jldistorterm^2)$ compared to \ac{itkrm}. }\\
%In the complete proof in \Cref{th:ictkm_convergence}, we will further see that the admissible noise level of \ac{ictkm} gets reduced by a factor of %order $O(\ddim/\jlembeddingdim)$ compared to \ac{itkrm}. The reason for this reduction is the \ac{jl} embedding normalization factor $\jlmatrixnormfactor$, which appears in our concentration inequalities for the noise term. In the complete proof, the noise level reduction leads to the extra condition $\jlembeddingdim \geq \jldistorterm^{-2} \ddim \nsigma^2 (1 + \jldistorterm)$, but this condition disappears in the $O$-notation with the assumption $\nsigma^2 = 1/\ddim$ used in \Cref{th:ictkm_convergence_o_notation}. }
% 
%Discussion of noiselevel: hidden in O-notation, this
%blabla who wants to see it look at proof of full theorem.....
To not break the flow of the paper, we present the complete proof together with the exact statement in the appendix, and provide here only a sketch which summarizes the main ideas.
\begin{proof}[Proof Sketch]
To prove this result we will use the \ac{jl} property in \labelcref{eq:jl_lemma} and \Cref{def:fast_jl_embedding}. We first need to ensure that the relative distances between all pairs of \emph{embedded atoms} of the \emph{generating} and \emph{perturbation dictionaries} are preserved up to a distortion $\jldistorterm$ with high probability, which is achieved by enforcing the embedding dimension bound in \labelcref{eq:jl_embedding_o_theorem_complete}. The distance preservation property in conjunction with the assumption that the coefficients have a well balanced distribution in magnitude will ensure that compressed thresholding recovers the generating (oracle) signal support with high probability. With this result in hand, we then make use of the same techniques used in the proof of Theorem 4.2 in \cite{sc15}. Assuming that compressed thresholding recovers the generating signal support, we will apply a triangle inequality argument to the update formula $\ppatom_{k} = \frac{1}{\nsig} \sum_{n}{\left( \trainsignal_n - P(\pdico_{\compthreshpdicosuppn}) \trainsignal_n + P ( \patom_k ) \trainsignal_n \right) \cdot \signop{(\langle \patom_k, \trainsignal_n \rangle)} }$ and show that the difference between the residual based on the oracle signs and supports using $\dico$, and the residual using $\pdico$ concentrates around its expectation, which is small. This concentration property also ensures that the sum of residuals using $\dico$ converges to a scaled version of $\atom_k$. The convergence of the sum of residuals will then be used to show that one iteration of \ac{ictkm} decreases the error, e.g., $d(\ppdico,\dico) \leq \kappa \varepsilon$ for $\kappa < 1$, with high probability. Finally, we iterate the error decreasing property and show that the target error is reached, $d(\ppdico,\dico) \leq \targeterror$, after $L$ iterations. 
%We defer the complete proof to \Cref{sec:proof_of_ictkm_convergence_theorem} to keep the flow of the paper.
\end{proof}

%In the complete proof in \Cref{th:ictkm_convergence}, we will further see that the admissible noise level of \ac{ictkm} gets reduced by a factor of order $O(\ddim/\jlembeddingdim)$ compared to \ac{itkrm}. The reason for this reduction is the \ac{jl} embedding normalization factor $\jlmatrixnormfactor$, which appears in our concentration inequalities for the noise term. In the complete proof, the noise level reduction leads to the extra condition $\jlembeddingdim \geq \jldistorterm^{-2} \ddim \nsigma^2 (1 + \jldistorterm)$, but this condition disappears in the $O$-notation with the assumption $\nsigma^2 = 1/\ddim$ used in \Cref{th:ictkm_convergence_o_notation}. }
%
%%%%%%%%%%%%%%%%%%%%%%%%%%%%%%%%%%%%%%%%%%%%%%%%%%%%
%%%%%%%%%%%%%%%%%%%%%SIMULATIONS%%%%%%%%%%%%%%%%%%%%
%%%%%%%%%%%%%%%%%%%%%%%%%%%%%%%%%%%%%%%%%%%%%%%%%%%%

\section{Numerical Simulations}\label{sec:simulations}

We will now complement our theoretical results with numerical simulations to illustrate the relation between the compression ratio, admissible sparsity level/achievable error, and the computational cost of dictionary learning in a practical setting\footnote{A MATLAB toolbox for reproducing the experiments can be found at \href{https://www.uibk.ac.at/mathematik/personal/schnass/code/ictkm.zip}{www.uibk.ac.at/mathematik/personal/schnass/code/\hspace{.5pt}ictkm.zip}. The toolbox has been designed from the ground up to efficiently learn the dictionaries. In particular, parallelization techniques to process multiple training signals simultaneously and fully utilize multiple \ac{cpu} cores have been implemented and \ac{gpu} acceleration may also be used to speed-up the most computationally demanding vectorized calculations.}. The simulation results will further demonstrate that \ac{ictkm} is a powerful, low-cost algorithm for learning dictionaries, especially when dealing with training signals in dimension $\ddim \geq 100,000$, where it is possible to speed-up the learning process by up to an order of magnitude. We begin with simulations carried out on synthetic training data, and then follow up with simulations on audio training data obtained from the \ac{rwc} music database \cite{ma04}.

\subsection{Synthetic data}

We have generated our synthetic training set with \labelcref{eq:training_signal_with_model}, using both overcomplete and undercomplete generating dictionaries and exactly $\sparsity$-sparse coefficients under Gaussian noise. Details for the signal generation are described next.

\textbf{Generating dictionary.} We will generate two types of dictionaries: with intrinsic dimension $\tilde d$ equal or lower than the ambient dimension $d$. In both cases the dictionary $\dico$ is constructed as a subset of the union of two bases in $\R^{\tilde \ddim}$, the Dirac basis and the first-half elements of the discrete cosine transform basis, meaning the number of atoms in the generating dictionary amounts to $\natoms = (3/2)\tilde \ddim$. In case of $\tilde d < d$ we simply embed the $\tilde \ddim$-dimensional dictionary into $\R^{\ddim \times K}$ by zero-padding the missing entries. \\ 
\indent \textbf{Sparse coefficients.} As coefficient sequences $\coeffsequence$ we use geometric sequences with decay factor $\coeffsequenceterm_b$ uniformly distributed in $[1-b, 1]$ for some $0 < b < 1$; to be more specific, we set $\coeffsequence(k) = \beta \coeffsequenceterm_b^{k-1}$ for $k \leq \sparsity$ and $\coeffsequence(k) = 0$ for $k > \sparsity$ with $\beta = (1-\coeffsequenceterm_b^{2S})^{1/2}$ the normalisation factor ensuring $\|\coeffsequence\|_2=1$. The maximal dynamic range of our coefficients for this particular arrangement is $(1-b)^{1-\sparsity}$, and for a given sparsity level $\sparsity$ we choose $b$ so that the maximal dynamic range is exactly $4$. \\ 
\indent \textbf{Sparsity level.} We have experimented with two parameter regimes for the sparsity level, $\sparsity = O(1)$ and $\sparsity = O({\tilde \ddim}^{1/2})$; or more precisely, $\sparsity = 4$ and $\sparsity = {\tilde \ddim}^{1/2}/2$. We have chosen these two sparsity level regimes to experimentally validate our theoretical findings; that for the lower sparsity levels $\sparsity = O(1)$ the highest compression ratios can be achieved but at the expense of an increased recovery error and, on the other hand, with the higher sparsity levels $\sparsity = O({\tilde \ddim}^{1/2})$ recovery precision is increased but only modest improvements in the computational cost are achievable. \\ 
\indent \textbf{Recovery criteria.} Given an atom $\ppatom_l$ from the output dictionary $\ppdico$, the criteria for declaring the generating atom $\atom_k$ as recovered is $\max_l \left|  \langle \ppatom_l, \atom_k  \rangle \right| \geq 0.99$. To estimate the percentage of recovered atoms we run $100$ iterations of \ac{ictkm} (or \ac{itkrm} where applicable) with $\nsig = 50 \natoms \log \natoms$ using $10$ completely random dictionary initializations, meaning that the atoms $\ppatom_l$ for the initial dictionary $\ppdico$ are chosen uniformly at random from the unit sphere in $\R^\ddim$. \\ 
\indent \textbf{Noise level.} The noise $\noise$ is chosen as a Gaussian random vector with mean zero and variance $\nsigma^2 = 1/(4 \ddim)$. Since $\E ( \left\| \dico \coeffvector  \right\|_2^2) = 1$ for our coefficient sequence and $\E (\left\| \noise \right\|_2^2) = d \nsigma^2$ for Gaussian noise, the \ac{snr} of our training signals is exactly $4$.

Next we present recovery simulations carried out with our synthetic generated training data to evaluate the achievable compression ratio and recovery rates/time. We will also evaluate how \ac{ictkm} scales with increasing ambient dimensions.

\subsubsection{Compression ratio}

In \Cref{tbl:table_highest_compression_ratio} we evaluate the highest achievable compression ratio to recover the generating dictionary $\dico$ with the \acf{dft}, \acf{dct}, and \acf{crt} as \ac{jl} embedding. Here we have used the convention that $\dico$ has been recovered if $90 \%$ of its atoms are recovered. Synthetic signals with ambient dimension $ \ddim \in \{ 256, 512, 1024, 2048 \}$, intrinsic dimension $\tilde\ddim = \ddim$ and compression ratios in $ \{ 1.5, 2, 2.5, 2.9, 3.33, 4, 5, \allowbreak 6.67, 10, 20, 33.33, 40 \}$ have been used in this experiment. As predicted by our theoretical results, we can attain much higher compression ratios when using reduced sparsity levels, compare the results with $\sparsity = O(\sqrt{\ddim})$ in \Cref{tbl:table_highest_compression_ratio}(A) and $\sparsity = O(1)$ in \Cref{tbl:table_highest_compression_ratio}(B). Additionally, although low-dimensional training signals have been used in this experiment, a compression ratio of at least $2:1$ can be attained for sparsity levels $\sparsity = O(\sqrt{\ddim})$ rising to $6.67:1$ for sparsity levels $\sparsity = O(1)$, and these good results indicate that a large constant factor might be present in our compression ratio estimate of $O(\ddim / \log^5 \ddim) : 1$ in the theoretical results. Lastly, note that the \ac{dft} attains consistently higher compression ratios than the \ac{dct} and \ac{crt} as \ac{jl} embeddings.
\newcommand{\tableone}{
\begin{tabular}{@{}ccc@{}}
\toprule
\thead{\ac{jl} \\ embedding \\ type } & \thead{ambient \\ dimension} & \thead{highest \\ compression \\ ratio } \\
\midrule 
\multirow{3}{*}{\ac{dft}} & $256$ & $\hspace{1.3em} 5:1$ \\
 & $512$ & $6.67:1$ \\
 & $1,024$ & $6.67:1$ \\
\midrule 
\multirow{3}{*}{\ac{dct}} & $256$ & $3.33:1$ \\
 & $512$ & $\hspace{1.3em} 4:1$ \\
 & $1,024$ & $\hspace{1.3em} 5:1$ \\
\midrule
 \multirow{3}{*}{\ac{crt}} & $256$ & $\hspace{1.3em} 2:1$ \\
 & $512$ & $\hspace{.6em} 2.9:1$ \\
 & $1,024$ & $3.33:1$ \\
\bottomrule
\end{tabular}
}

\newcommand{\tabletwo}{
\begin{tabular}{@{}ccc@{}}
\toprule
\thead{\ac{jl} \\ embedding \\ type } & \thead{ambient \\ dimension} & \thead{highest \\ compression \\ ratio } \\
\midrule 
\multirow{3}{*}{\ac{dft}} & $512$ & $\hspace{0.7em} 10:1$ \\
 & $1,024$ & $\hspace{0.7em} 20:1$ \\
 & $2,048$ & $\hspace{0.7em} 20:1$ \\
 \midrule
 \multirow{3}{*}{\ac{dct}} & $512$ & $6.67:1$ \\
 & $1,024$ & $\hspace{0.7em} 10:1$ \\
 & $2,048$ & $\hspace{0.7em} 20:1$ \\
 \midrule 
 \multirow{3}{*}{\ac{crt}} & $512$ & $6.67:1$ \\
 & $1,024$ & $\hspace{0.7em} 10:1$ \\
 & $2,048$ & $\hspace{0.7em} 10:1$ \\ 
\bottomrule
\end{tabular}
}

\begin{table}[htp]%
  \centering
  \subfloat[$\sparsity = O(\sqrt{\ddim})$.]{\label{subtable:sparsity_1_highest_compression_ratio} \tableone} %
  \hspace{7.5em}
  \subfloat[$\sparsity = O(1)$.]{\label{subtable:sparsity_2_highest_compression_ratio} \tabletwo}
  \caption{Highest compression ratio achieved with \ac{ictkm}.}
  \label{tbl:table_highest_compression_ratio}
\end{table}

\subsubsection{Recovery rates}
In \Cref{fig:recovery_rates} we evaluate the attained dictionary recovery rates for synthetic training signals of ambient dimension $\ddim = 1,024$ and intrinsic dimension $\tilde \ddim =\ddim$ with \ac{itkrm}, and \ac{ictkm} using the \ac{dft}, \ac{dct}, and \ac{crt} as \ac{jl} embedding and increasing compression ratios. The solid yellow line marks the highest recovery rate achieved with \ac{itkrm} in the experiment, i.e., $99 \%$ recovered atoms for sparsity levels $\sparsity = O(\sqrt{\ddim})$ and $94.7 \%$ for sparsity levels $\sparsity = O(1)$. We can see from the results that \ac{ictkm} usually requires far less iterations than \ac{itkrm} to reach the target recovery rate. In particular, for sparsity levels $\sparsity = O(\sqrt{\ddim})$ in \Cref{subfig:num_iteration_sparsity_1_d_1024} the \ac{dft} with compression ratios of $2.5 : 1$ and $3.3 : 1$, or the \ac{dct} with a compression ratio of $2.5 : 1$ can attain the $99 \%$ recovery rate with roughly $60$ iterations, or a $40\%$ reduction in the number of iterations compared to \ac{itkrm}. A similar trend can be seen in \Cref{subfig:num_iteration_sparsity_2_d_1024} for the sparsity levels $\sparsity = O(1)$, but here the increase in convergence speed is much more pronounced. In particular, the \ac{dft} with a compression ratio of $10 : 1$, or the \ac{dct} and \ac{crt} with a compression ratio of $5:1$ can attain the $94.7 \%$ recovery rate with no more than $40$ iterations, or a $60\%$ reduction in the number of iterations compared to \ac{itkrm}. Lastly, note that for the sparsity levels $\sparsity = O(\sqrt{\ddim})$, \ac{ictkm} always managed to achieve a higher recovery rate than \ac{itkrm}. For example, the \ac{dft} with compression ratios of $2.5 : 1$ and $5 : 1$ attained a $99.8 \%$ recovery rate compared to the $99\%$ rate achieved with \ac{itkrm}. For the sparsity levels $\sparsity = O(1)$, the \ac{dct} and \ac{crt} with a compression ratio of $5:1$, and the \ac{dft} with a compression ratio of $10:1$ also managed to attain a higher recovery rate of at least $96.5 \%$ compared to the $94.7 \%$ recovery rate attained with \ac{itkrm}, but for some compression ratios the improved recovery rates could not be attained. For example, the \ac{dct} with a compression ratio of $2.5 : 1$ attained a $93.18 \%$ recovery rate.

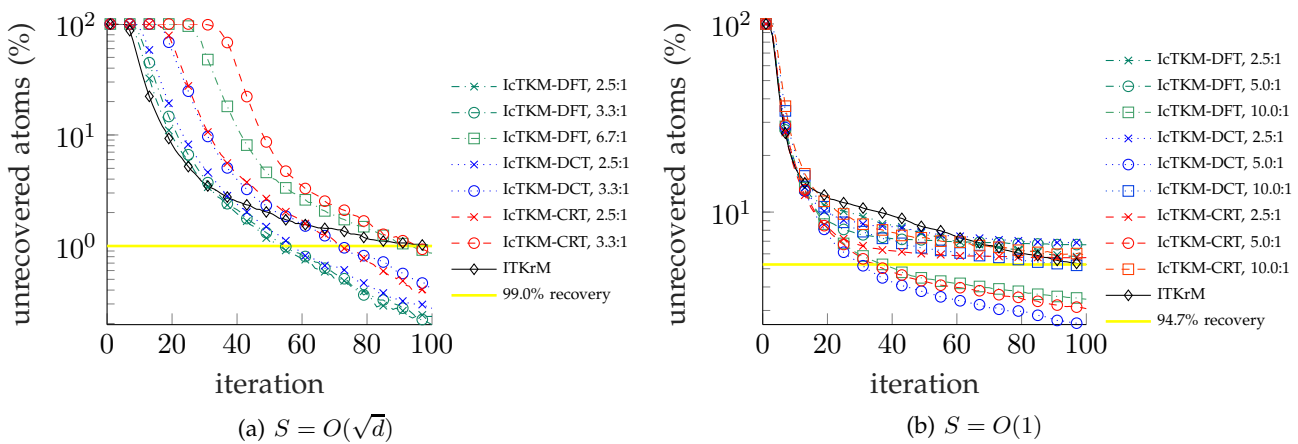
\begin{figure}[h]  
\setlength\figureheight{\defaultheight} 
\setlength\figurewidth{\defaultwidth}
\centering  

\captionsetup[sub]{singlelinecheck=off,margin=\defaultsubcaptionmargin}
\subfloat[$\sparsity = O(\sqrt{\ddim})$]
{
\label{subfig:num_iteration_sparsity_1_d_1024} 
\definecolor{mycolor1}{rgb}{0.00000,0.50000,0.40000}%
\definecolor{mycolor2}{rgb}{0.04762,0.52381,0.40000}%
\definecolor{mycolor3}{rgb}{0.23810,0.61905,0.40000}%
\definecolor{mycolor4}{rgb}{0.00000,0.04762,0.97619}%
\definecolor{mycolor5}{rgb}{1.00000,0.04762,0.00000}%
\definecolor{mycolor6}{rgb}{1.00000,1.00000,0.00000}%
\begin{tikzpicture}

\begin{axis}[%
width=0.72\figurewidth,
height=\figureheight,
at={(0\figurewidth,0\figureheight)},
scale only axis,
xmin=0,
xmax=100,
xlabel style={font=\color{white!15!black}},
xlabel={iteration},
ymode=log,
ymin=0.1953125,
ymax=100,
yminorticks=true,
ylabel style={font=\color{white!15!black}},
ylabel={unrecovered atoms (\%)},
axis background/.style={fill=white},
axis x line*=bottom,
axis y line*=left,
legend style={at={(1.03,0.5)}, anchor=west, legend cell align=left, align=left, fill=none, draw=none},
legend style={font=\tiny}
]
\addlegendimage{legend image code/.code=}  
\addlegendentry{\\}
\addplot [color=mycolor1, dashdotted, mark=x, mark repeat=6, mark options={solid, mycolor1}]
  table[row sep=crcr]{%
1	100\\
2	100\\
3	100\\
4	100\\
5	100\\
6	99.8987268518518\\
7	98.0251736111111\\
8	90.111400462963\\
9	77.1629050925926\\
10	62.9846643518519\\
11	50.4774305555556\\
12	40.3284143518519\\
13	32.2699652777778\\
14	26.1140046296296\\
15	21.1588541666667\\
16	17.896412037037\\
17	14.9233217592593\\
18	12.6446759259259\\
19	10.9953703703704\\
20	9.73668981481482\\
21	8.65162037037037\\
22	7.70399305555556\\
23	6.95891203703704\\
24	6.36574074074073\\
25	5.74363425925925\\
26	5.25173611111111\\
27	4.7019675925926\\
28	4.28964120370371\\
29	3.90624999999999\\
30	3.63859953703702\\
31	3.38541666666669\\
32	3.19010416666667\\
33	2.99479166666669\\
34	2.77777777777777\\
35	2.66203703703704\\
36	2.51012731481481\\
37	2.33651620370371\\
38	2.19907407407406\\
39	2.07609953703702\\
40	1.953125\\
41	1.85185185185186\\
42	1.7433449074074\\
43	1.67100694444444\\
44	1.56250000000001\\
45	1.50462962962965\\
46	1.43229166666669\\
47	1.33825231481481\\
48	1.25144675925927\\
49	1.17910879629629\\
50	1.15017361111113\\
51	1.09230324074073\\
52	1.04890046296298\\
53	0.991030092592595\\
54	0.954861111111128\\
55	0.911458333333329\\
56	0.882523148148138\\
57	0.846354166666671\\
58	0.824652777777771\\
59	0.795717592592581\\
60	0.759548611111114\\
61	0.752314814814824\\
62	0.716145833333343\\
63	0.687210648148138\\
64	0.629340277777771\\
65	0.61487268518519\\
66	0.600405092592581\\
67	0.5859375\\
68	0.564236111111114\\
69	0.549768518518519\\
70	0.506365740740733\\
71	0.506365740740733\\
72	0.499131944444443\\
73	0.484664351851862\\
74	0.462962962962962\\
75	0.434027777777771\\
76	0.412326388888872\\
77	0.397858796296291\\
78	0.383391203703709\\
79	0.390624999999986\\
80	0.383391203703695\\
81	0.347222222222229\\
82	0.311053240740733\\
83	0.282118055555543\\
84	0.289351851851848\\
85	0.289351851851848\\
86	0.282118055555557\\
87	0.282118055555557\\
88	0.267650462962962\\
89	0.274884259259267\\
90	0.260416666666671\\
91	0.260416666666671\\
92	0.267650462962962\\
93	0.260416666666657\\
94	0.245949074074062\\
95	0.231481481481467\\
96	0.238715277777771\\
97	0.238715277777771\\
98	0.231481481481481\\
99	0.231481481481481\\
100	0.231481481481481\\
};
\addlegendentry{IcTKM-DFT,  2.5:1}

\addplot [color=mycolor2, dashdotted, mark=o, mark repeat=6, mark options={solid, mycolor2}]
  table[row sep=crcr]{%
1	100\\
2	100\\
3	100\\
4	100\\
5	100\\
6	100\\
7	99.7395833333333\\
8	97.6056134259259\\
9	90.6467013888889\\
10	79.6440972222222\\
11	67.0428240740741\\
12	55.0708912037037\\
13	44.8350694444444\\
14	36.3208912037037\\
15	29.7381365740741\\
16	24.8480902777778\\
17	20.5439814814815\\
18	17.2670717592593\\
19	14.6773726851852\\
20	12.572337962963\\
21	10.8796296296296\\
22	9.3822337962963\\
23	8.36226851851852\\
24	7.45804398148148\\
25	6.6333912037037\\
26	6.00405092592592\\
27	5.42534722222221\\
28	4.8900462962963\\
29	4.43431712962965\\
30	4.05092592592592\\
31	3.71817129629629\\
32	3.47222222222223\\
33	3.25520833333333\\
34	3.01649305555556\\
35	2.77777777777777\\
36	2.58969907407408\\
37	2.40162037037038\\
38	2.26417824074072\\
39	2.19907407407406\\
40	2.11950231481482\\
41	2.01822916666669\\
42	1.85908564814815\\
43	1.76504629629629\\
44	1.68547453703704\\
45	1.60590277777777\\
46	1.54079861111113\\
47	1.43229166666667\\
48	1.3816550925926\\
49	1.31655092592592\\
50	1.25868055555556\\
51	1.17187499999999\\
52	1.12847222222223\\
53	1.07783564814815\\
54	1.02719907407408\\
55	0.983796296296276\\
56	0.933159722222229\\
57	0.896990740740733\\
58	0.875289351851862\\
59	0.860821759259267\\
60	0.824652777777771\\
61	0.802951388888886\\
62	0.795717592592581\\
63	0.759548611111114\\
64	0.716145833333329\\
65	0.708912037037038\\
66	0.658275462962962\\
67	0.629340277777771\\
68	0.607638888888872\\
69	0.593171296296291\\
70	0.564236111111114\\
71	0.542534722222229\\
72	0.506365740740733\\
73	0.506365740740719\\
74	0.484664351851862\\
75	0.448495370370381\\
76	0.426793981481481\\
77	0.412326388888886\\
78	0.383391203703709\\
79	0.361689814814824\\
80	0.376157407407419\\
81	0.361689814814824\\
82	0.347222222222229\\
83	0.332754629629619\\
84	0.325520833333329\\
85	0.325520833333329\\
86	0.311053240740733\\
87	0.325520833333329\\
88	0.303819444444443\\
89	0.296585648148152\\
90	0.289351851851862\\
91	0.274884259259267\\
92	0.253182870370381\\
93	0.231481481481495\\
94	0.22424768518519\\
95	0.231481481481481\\
96	0.217013888888886\\
97	0.217013888888886\\
98	0.209780092592595\\
99	0.202546296296305\\
100	0.1953125\\
};
\addlegendentry{IcTKM-DFT,  3.3:1}

\addplot [color=mycolor3, dashdotted, mark=square, mark repeat=6, mark options={solid, mycolor3}]
  table[row sep=crcr]{%
1	100\\
2	100\\
3	100\\
4	100\\
5	100\\
6	100\\
7	100\\
8	100\\
9	100\\
10	100\\
11	100\\
12	100\\
13	100\\
14	100\\
15	100\\
16	100\\
17	99.9927662037037\\
18	100\\
19	99.9927662037037\\
20	99.9638310185185\\
21	99.8336226851852\\
22	99.5515046296296\\
23	98.6255787037037\\
24	97.7719907407407\\
25	95.7248263888889\\
26	91.4713541666667\\
27	82.9137731481482\\
28	75.6221064814815\\
29	69.3214699074074\\
30	55.8015046296296\\
31	47.7864583333333\\
32	40.625\\
33	34.1869212962963\\
34	28.833912037037\\
35	24.5949074074074\\
36	20.8550347222222\\
37	18.1061921296296\\
38	15.5381944444444\\
39	13.3752893518519\\
40	11.7115162037037\\
41	10.2936921296296\\
42	9.15798611111111\\
43	8.10185185185185\\
44	7.34953703703702\\
45	6.62615740740742\\
46	6.02575231481481\\
47	5.49768518518519\\
48	4.98408564814814\\
49	4.60069444444446\\
50	4.32581018518518\\
51	4.05092592592592\\
52	3.93518518518519\\
53	3.68923611111111\\
54	3.47945601851852\\
55	3.3347800925926\\
56	3.19010416666667\\
57	3.02372685185185\\
58	2.93692129629629\\
59	2.87181712962962\\
60	2.69097222222221\\
61	2.60416666666667\\
62	2.46672453703701\\
63	2.45949074074075\\
64	2.32204861111111\\
65	2.22077546296296\\
66	2.11950231481481\\
67	2.06886574074073\\
68	1.98206018518519\\
69	1.91695601851853\\
70	1.88802083333331\\
71	1.83738425925924\\
72	1.76504629629628\\
73	1.72887731481482\\
74	1.67100694444444\\
75	1.62037037037038\\
76	1.59143518518519\\
77	1.52633101851853\\
78	1.51186342592592\\
79	1.49016203703702\\
80	1.43952546296296\\
81	1.39612268518518\\
82	1.3671875\\
83	1.32378472222221\\
84	1.28761574074073\\
85	1.24421296296296\\
86	1.17910879629629\\
87	1.15740740740742\\
88	1.14293981481482\\
89	1.08506944444444\\
90	1.07783564814815\\
91	1.04890046296296\\
92	1.03443287037037\\
93	1.00549768518518\\
94	0.991030092592581\\
95	0.954861111111114\\
96	0.954861111111114\\
97	0.918692129629633\\
98	0.889756944444429\\
99	0.868055555555543\\
100	0.860821759259267\\
};
\addlegendentry{IcTKM-DFT,  6.7:1}

\addplot [color=blue, dotted, mark=x, mark repeat=6, mark options={solid, blue}]
  table[row sep=crcr]{%
1	100\\
2	100\\
3	100\\
4	100\\
5	100\\
6	100\\
7	99.9927662037037\\
8	99.8336226851852\\
9	98.1047453703704\\
10	91.5147569444444\\
11	80.9461805555556\\
12	69.3359375\\
13	58.2899305555556\\
14	48.2421875\\
15	40.4441550925926\\
16	33.4997106481482\\
17	27.6548032407407\\
18	22.9890046296296\\
19	19.3214699074074\\
20	16.3628472222222\\
21	13.9684606481482\\
22	11.9791666666667\\
23	10.4890046296296\\
24	9.17245370370372\\
25	8.2103587962963\\
26	7.26996527777777\\
27	6.72019675925925\\
28	5.98958333333333\\
29	5.45428240740742\\
30	5.01302083333331\\
31	4.60792824074073\\
32	4.2896412037037\\
33	3.94241898148148\\
34	3.57349537037037\\
35	3.36371527777777\\
36	3.07436342592592\\
37	2.82841435185185\\
38	2.65480324074072\\
39	2.51012731481482\\
40	2.37991898148147\\
41	2.22800925925927\\
42	2.1556712962963\\
43	2.01822916666667\\
44	1.90972222222223\\
45	1.81568287037038\\
46	1.70717592592592\\
47	1.64930555555557\\
48	1.5625\\
49	1.49016203703702\\
50	1.41782407407408\\
51	1.33825231481481\\
52	1.28761574074075\\
53	1.24421296296298\\
54	1.1646412037037\\
55	1.11400462962963\\
56	1.07060185185185\\
57	1.01273148148147\\
58	0.954861111111128\\
59	0.904224537037038\\
60	0.860821759259281\\
61	0.817418981481481\\
62	0.817418981481481\\
63	0.766782407407405\\
64	0.7595486111111\\
65	0.737847222222214\\
66	0.694444444444429\\
67	0.665509259259252\\
68	0.651041666666657\\
69	0.658275462962962\\
70	0.643807870370352\\
71	0.643807870370352\\
72	0.622106481481467\\
73	0.607638888888872\\
74	0.571469907407405\\
75	0.535300925925924\\
76	0.513599537037038\\
77	0.484664351851848\\
78	0.477430555555557\\
79	0.462962962962976\\
80	0.434027777777786\\
81	0.434027777777786\\
82	0.41956018518519\\
83	0.383391203703709\\
84	0.376157407407405\\
85	0.376157407407405\\
86	0.361689814814824\\
87	0.347222222222229\\
88	0.332754629629633\\
89	0.332754629629633\\
90	0.325520833333329\\
91	0.318287037037038\\
92	0.311053240740748\\
93	0.303819444444457\\
94	0.303819444444443\\
95	0.296585648148138\\
96	0.303819444444429\\
97	0.296585648148138\\
98	0.289351851851848\\
99	0.274884259259252\\
100	0.274884259259252\\
};
\addlegendentry{IcTKM-DCT,  2.5:1}

\addplot [color=mycolor4, dotted, mark=o, mark repeat=6, mark options={solid, mycolor4}]
  table[row sep=crcr]{%
1	100\\
2	100\\
3	100\\
4	100\\
5	100\\
6	100\\
7	100\\
8	100\\
9	100\\
10	100\\
11	100\\
12	100\\
13	99.9710648148148\\
14	99.8625578703704\\
15	99.291087962963\\
16	96.8967013888889\\
17	91.015625\\
18	80.7653356481482\\
19	68.5329861111111\\
20	56.958912037037\\
21	48.0541087962963\\
22	40.6828703703704\\
23	34.6281828703704\\
24	29.2679398148148\\
25	24.7974537037037\\
26	20.9635416666667\\
27	17.6866319444445\\
28	14.9377893518518\\
29	12.9557291666667\\
30	11.1183449074074\\
31	9.69328703703704\\
32	8.53587962962963\\
33	7.65335648148147\\
34	6.79253472222223\\
35	6.09809027777777\\
36	5.59172453703704\\
37	5.04918981481482\\
38	4.64409722222223\\
39	4.34027777777777\\
40	3.96412037037038\\
41	3.68923611111111\\
42	3.45775462962962\\
43	3.23350694444446\\
44	2.97309027777777\\
45	2.80671296296296\\
46	2.64033564814814\\
47	2.51736111111111\\
48	2.38715277777777\\
49	2.31481481481482\\
50	2.20630787037035\\
51	2.11950231481481\\
52	2.01822916666667\\
53	1.9675925925926\\
54	1.90248842592592\\
55	1.85185185185186\\
56	1.81568287037038\\
57	1.7722800925926\\
58	1.69270833333334\\
59	1.62037037037037\\
60	1.53356481481482\\
61	1.51186342592592\\
62	1.45399305555554\\
63	1.42505787037038\\
64	1.38888888888887\\
65	1.30208333333333\\
66	1.28038194444446\\
67	1.23697916666667\\
68	1.16464120370371\\
69	1.08506944444444\\
70	1.07783564814814\\
71	1.04890046296296\\
72	0.998263888888886\\
73	0.962094907407405\\
74	0.896990740740733\\
75	0.868055555555557\\
76	0.846354166666671\\
77	0.824652777777771\\
78	0.824652777777771\\
79	0.817418981481467\\
80	0.824652777777771\\
81	0.752314814814795\\
82	0.774016203703695\\
83	0.752314814814838\\
84	0.737847222222229\\
85	0.716145833333329\\
86	0.694444444444457\\
87	0.65827546296299\\
88	0.629340277777771\\
89	0.593171296296291\\
90	0.593171296296305\\
91	0.564236111111114\\
92	0.528067129629633\\
93	0.528067129629633\\
94	0.506365740740748\\
95	0.499131944444443\\
96	0.484664351851848\\
97	0.462962962962962\\
98	0.462962962962962\\
99	0.426793981481481\\
100	0.390625\\
};
\addlegendentry{IcTKM-DCT,  3.3:1}

\addplot [color=red, dashed, mark=x, mark repeat=6, mark options={solid, red}]
  table[row sep=crcr]{%
1	100\\
2	100\\
3	100\\
4	100\\
5	100\\
6	100\\
7	100\\
8	100\\
9	100\\
10	99.9927662037037\\
11	99.9782986111111\\
12	99.9565972222222\\
13	99.8336226851852\\
14	99.5587384259259\\
15	98.5966435185185\\
16	96.4554398148148\\
17	92.650462962963\\
18	86.451099537037\\
19	78.7760416666667\\
20	66.9704861111111\\
21	56.958912037037\\
22	47.4537037037037\\
23	39.966724537037\\
24	33.1741898148148\\
25	27.6403356481481\\
26	22.9166666666667\\
27	19.3648726851852\\
28	16.4713541666667\\
29	14.0480324074074\\
30	12.1238425925926\\
31	10.619212962963\\
32	9.375\\
33	8.26099537037037\\
34	7.51591435185185\\
35	6.6767939814815\\
36	6.22106481481481\\
37	5.57725694444443\\
38	5.18663194444443\\
39	4.75260416666667\\
40	4.45601851851852\\
41	4.21006944444441\\
42	3.95688657407408\\
43	3.74710648148147\\
44	3.53732638888889\\
45	3.38541666666667\\
46	3.16116898148147\\
47	3.00202546296296\\
48	2.83564814814814\\
49	2.66927083333334\\
50	2.5390625\\
51	2.43055555555557\\
52	2.32204861111113\\
53	2.19907407407408\\
54	2.08333333333333\\
55	2.03269675925925\\
56	1.93142361111111\\
57	1.87355324074075\\
58	1.81568287037037\\
59	1.73611111111111\\
60	1.64930555555554\\
61	1.59143518518518\\
62	1.51909722222223\\
63	1.44675925925927\\
64	1.38888888888887\\
65	1.33101851851853\\
66	1.26591435185185\\
67	1.27314814814815\\
68	1.23697916666666\\
69	1.17910879629628\\
70	1.11400462962962\\
71	1.04166666666666\\
72	0.983796296296291\\
73	0.947627314814824\\
74	0.933159722222229\\
75	0.896990740740748\\
76	0.904224537037038\\
77	0.875289351851862\\
78	0.831886574074076\\
79	0.774016203703709\\
80	0.737847222222214\\
81	0.701678240740733\\
82	0.679976851851848\\
83	0.636574074074076\\
84	0.593171296296291\\
85	0.593171296296291\\
86	0.564236111111114\\
87	0.535300925925924\\
88	0.542534722222229\\
89	0.506365740740748\\
90	0.528067129629619\\
91	0.484664351851848\\
92	0.441261574074076\\
93	0.434027777777771\\
94	0.397858796296305\\
95	0.412326388888886\\
96	0.405092592592595\\
97	0.405092592592595\\
98	0.368923611111128\\
99	0.368923611111128\\
100	0.361689814814824\\
};
\addlegendentry{IcTKM-CRT,  2.5:1}

\addplot [color=mycolor5, dashed, mark=o, mark repeat=6, mark options={solid, mycolor5}]
  table[row sep=crcr]{%
1	100\\
2	100\\
3	100\\
4	100\\
5	100\\
6	100\\
7	100\\
8	100\\
9	100\\
10	100\\
11	100\\
12	100\\
13	100\\
14	100\\
15	100\\
16	100\\
17	100\\
18	100\\
19	100\\
20	100\\
21	100\\
22	100\\
23	99.9927662037037\\
24	99.9927662037037\\
25	99.9855324074074\\
26	99.9710648148148\\
27	99.8987268518518\\
28	99.7829861111111\\
29	99.5659722222222\\
30	99.2404513888889\\
31	98.4302662037037\\
32	96.6941550925926\\
33	94.4372106481482\\
34	90.4730902777778\\
35	84.1435185185185\\
36	76.3237847222222\\
37	68.091724537037\\
38	58.8975694444444\\
39	48.828125\\
40	39.8365162037037\\
41	33.2103587962963\\
42	27.0978009259259\\
43	22.1643518518519\\
44	18.7644675925926\\
45	15.7696759259259\\
46	13.1365740740741\\
47	11.2847222222222\\
48	9.9537037037037\\
49	8.65162037037038\\
50	7.76909722222223\\
51	6.87210648148147\\
52	6.13425925925925\\
53	5.57002314814814\\
54	5.12152777777777\\
55	4.71643518518518\\
56	4.36197916666667\\
57	4.15219907407408\\
58	3.86284722222223\\
59	3.64583333333333\\
60	3.49392361111111\\
61	3.29137731481481\\
62	3.05989583333333\\
63	3.02372685185186\\
64	2.92245370370371\\
65	2.81394675925925\\
66	2.67650462962962\\
67	2.53182870370371\\
68	2.40885416666669\\
69	2.36545138888889\\
70	2.25694444444444\\
71	2.24971064814815\\
72	2.14120370370372\\
73	2.09780092592592\\
74	2.03269675925928\\
75	1.93865740740739\\
76	1.88802083333334\\
77	1.85185185185185\\
78	1.7650462962963\\
79	1.67824074074072\\
80	1.63483796296296\\
81	1.5697337962963\\
82	1.46846064814815\\
83	1.41782407407408\\
84	1.38888888888889\\
85	1.27314814814814\\
86	1.23697916666669\\
87	1.20804398148147\\
88	1.19357638888889\\
89	1.21527777777779\\
90	1.1501736111111\\
91	1.11400462962963\\
92	1.09953703703702\\
93	1.08506944444444\\
94	1.04890046296298\\
95	0.998263888888872\\
96	0.9765625\\
97	0.933159722222229\\
98	0.911458333333329\\
99	0.875289351851848\\
100	0.868055555555543\\
};
\addlegendentry{IcTKM-CRT,  3.3:1}

\addplot [color=black, mark=diamond, mark repeat=6, mark options={solid, black}]
  table[row sep=crcr]{%
1	100\\
2	100\\
3	100\\
4	100\\
5	99.9855324074074\\
6	98.3434606481482\\
7	87.1021412037037\\
8	70.580150462963\\
9	55.2155671296296\\
10	43.0483217592593\\
11	33.9337384259259\\
12	27.1990740740741\\
13	22.2873263888889\\
14	18.7138310185185\\
15	16.0011574074074\\
16	13.7586805555556\\
17	12.0081018518519\\
18	10.8362268518519\\
19	9.29542824074075\\
20	8.1958912037037\\
21	7.32060185185185\\
22	6.68402777777777\\
23	6.19936342592594\\
24	5.63512731481481\\
25	5.18663194444444\\
26	4.78153935185185\\
27	4.49942129629629\\
28	4.17390046296296\\
29	3.84837962962962\\
30	3.65306712962963\\
31	3.47222222222221\\
32	3.34201388888889\\
33	3.24797453703704\\
34	3.1177662037037\\
35	2.97309027777776\\
36	2.85734953703702\\
37	2.75607638888889\\
38	2.65480324074073\\
39	2.59693287037037\\
40	2.51012731481482\\
41	2.46672453703704\\
42	2.39438657407408\\
43	2.34374999999999\\
44	2.31481481481481\\
45	2.21354166666666\\
46	2.14120370370371\\
47	2.13396990740742\\
48	2.06163194444444\\
49	2.02546296296298\\
50	1.98929398148148\\
51	1.93142361111111\\
52	1.85185185185188\\
53	1.7722800925926\\
54	1.7722800925926\\
55	1.72164351851853\\
56	1.69270833333334\\
57	1.65653935185186\\
58	1.59143518518518\\
59	1.59143518518518\\
60	1.56973379629629\\
61	1.5552662037037\\
62	1.5407986111111\\
63	1.51186342592592\\
64	1.48292824074073\\
65	1.45399305555554\\
66	1.45399305555554\\
67	1.43952546296295\\
68	1.43952546296295\\
69	1.43952546296295\\
70	1.41782407407406\\
71	1.39612268518516\\
72	1.39612268518516\\
73	1.3527199074074\\
74	1.32378472222223\\
75	1.28038194444443\\
76	1.28038194444443\\
77	1.26591435185183\\
78	1.25144675925925\\
79	1.19357638888889\\
80	1.16464120370371\\
81	1.14293981481482\\
82	1.13570601851853\\
83	1.12123842592592\\
84	1.10677083333333\\
85	1.10677083333333\\
86	1.09230324074073\\
87	1.09230324074073\\
88	1.07783564814814\\
89	1.06336805555556\\
90	1.06336805555556\\
91	1.06336805555556\\
92	1.06336805555556\\
93	1.04890046296296\\
94	1.03443287037038\\
95	1.03443287037038\\
96	1.03443287037038\\
97	1.01996527777777\\
98	1.00549768518518\\
99	1.00549768518518\\
100	0.998263888888886\\
};
\addlegendentry{ITKrM}

\addplot [color=mycolor6, line width=1.0pt]
  table[row sep=crcr]{%
0	0.998263888888886\\
100	0.998263888888886\\
};
\addlegendentry{99.0\% recovery}

\end{axis}
\end{tikzpicture}%
} 
\subfloat[$\sparsity = O(1)$]
{
\label{subfig:num_iteration_sparsity_2_d_1024} 
\definecolor{mycolor1}{rgb}{0.00000,0.50000,0.40000}%
\definecolor{mycolor2}{rgb}{0.04762,0.52381,0.40000}%
\definecolor{mycolor3}{rgb}{0.23810,0.61905,0.40000}%
\definecolor{mycolor4}{rgb}{0.00000,0.04762,0.97619}%
\definecolor{mycolor5}{rgb}{0.00000,0.23810,0.88095}%
\definecolor{mycolor6}{rgb}{1.00000,0.04762,0.00000}%
\definecolor{mycolor7}{rgb}{1.00000,0.23810,0.00000}%
\definecolor{mycolor8}{rgb}{1.00000,1.00000,0.00000}%
%
}
\caption{Dictionary recovery rates with increasing compression ratios.}
\label{fig:recovery_rates}
\end{figure}

\subsubsection{Recovery time}
In \Cref{fig:recovery_time} we evaluate the dictionary recovery time attained for synthetic training signals of ambient dimension $\ddim = 1,024$ and intrinsic dimension $\tilde \ddim =\ddim$ with \ac{itkrm}, and \ac{ictkm} using the \ac{dft}, \ac{dct}, and \ac{crt} as \ac{jl} embedding and increasing compression ratios. The solid yellow line again marks the highest recovery rate achieved with \ac{itkrm} in the experiment. As predicted by our theoretical results, we can attain much better improvement in the computational complexity of \ac{ictkm} when using reduced sparsity levels, compare the results with $\sparsity = O(\sqrt{\ddim})$ in \Cref{subfig:cpu_time_sparsity_1_d_1024} and $\sparsity = O(1)$ in \Cref{subfig:cpu_time_sparsity_2_d_1024}. In particular, for the higher sparsity levels, \ac{ictkm} with the \ac{dct} and a compression ratio of $2.5:1$ requires $9.02$ hours to attain the $99\%$ recovery rate compared to $18.44$ hours required by \ac{itkrm}, or a $2.04$ speed-up in dictionary recovery time. For the lower sparsity levels, on the other hand, the \ac{dct} with a compression ratio of $5:1$ required $2.98$ hours to attain the $94.7\%$ recovery rate compared to $16.3$ hours required by \ac{itkrm}, or a $5.47$ speed-up ratio. Lastly, note that although the \ac{dft} at a given compression ratio usually requires less iterations to recover the dictionary than the \ac{dct} and \ac{crt} at same compression ratio, compare the results in \Cref{fig:recovery_rates}, the recovery time results do not reflect this faster convergence rate. We can see from \Cref{fig:recovery_time} that the \ac{dft} is usually slower than the other transforms to recover the dictionary. The reason for the worse performance is that the matrix product $\cdico^{*}\ctrainmatrix$ in the \ac{dft} has to be computed with complex numbers, thus requiring twice the amount of arithmetic operations than the \ac{dct} and \ac{crt}.

\begin{figure}[htp]
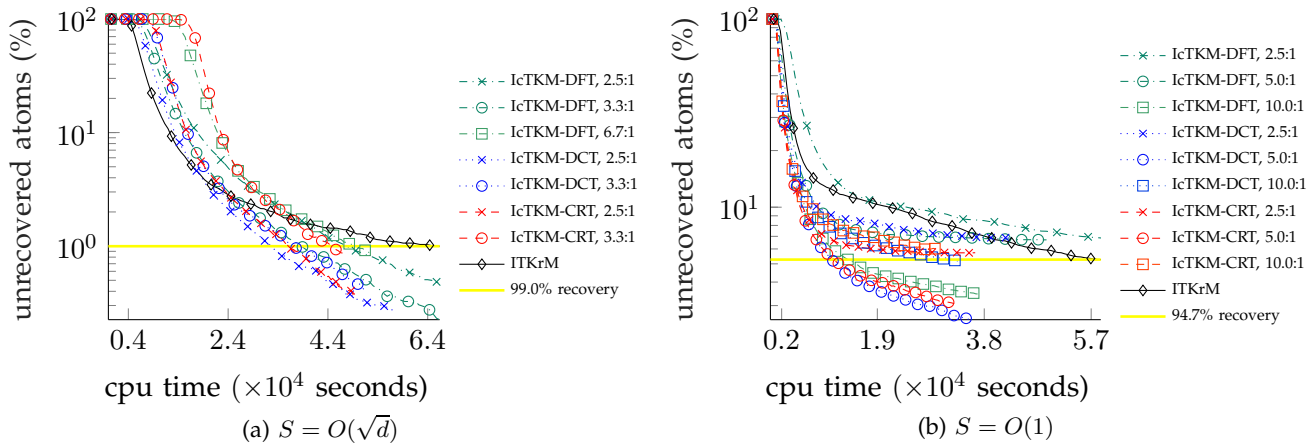
  
\setlength\figureheight{\defaultheight} 
\setlength\figurewidth{\defaultwidth}
\centering  

\captionsetup[sub]{singlelinecheck=off,margin=\defaultsubcaptionmargin}
\subfloat[$\sparsity = O(\sqrt{\ddim})$]
{\label{subfig:cpu_time_sparsity_1_d_1024} 
\definecolor{mycolor1}{rgb}{0.00000,0.50000,0.40000}%
\definecolor{mycolor2}{rgb}{0.04762,0.52381,0.40000}%
\definecolor{mycolor3}{rgb}{0.23810,0.61905,0.40000}%
\definecolor{mycolor4}{rgb}{0.00000,0.04762,0.97619}%
\definecolor{mycolor5}{rgb}{1.00000,0.04762,0.00000}%
\definecolor{mycolor6}{rgb}{1.00000,1.00000,0.00000}%
%
} 
\subfloat[$\sparsity = O(1)$]
{
\label{subfig:cpu_time_sparsity_2_d_1024} 
\definecolor{mycolor1}{rgb}{0.00000,0.50000,0.40000}%
\definecolor{mycolor2}{rgb}{0.04762,0.52381,0.40000}%
\definecolor{mycolor3}{rgb}{0.23810,0.61905,0.40000}%
\definecolor{mycolor4}{rgb}{0.00000,0.04762,0.97619}%
\definecolor{mycolor5}{rgb}{0.00000,0.23810,0.88095}%
\definecolor{mycolor6}{rgb}{1.00000,0.04762,0.00000}%
\definecolor{mycolor7}{rgb}{1.00000,0.23810,0.00000}%
\definecolor{mycolor8}{rgb}{1.00000,1.00000,0.00000}%
%
} 
\caption{Dictionary recovery time with increasing compression ratios.}
\label{fig:recovery_time}
\end{figure}

\subsubsection{Scalability}

In \Cref{fig:scalability} we evaluate the scalability of dictionary recovery with ambient dimension for \ac{itkrm}, and \ac{ictkm} using the \ac{dft}, \ac{dct}, and \ac{crt} as \ac{jl} embedding. Synthetic signals with ambient dimension ranging from $\ddim = 2,048$ up to $\ddim = 131,072$ have been used in this experiment. To carry out the learning process with these high-dimensional signals, we have fixed $\natoms = (3/2)\tilde{\ddim}$ with $\tilde{\ddim} = 1,024$ across all the ambient dimensions tested, so that we could avoid the large dictionary memory overhead, e.g., in the highest dimension setting $\natoms = (3/2) \ddim$ would have required more than $200$ gigabytes of volatile memory to manipulate the dictionary matrix in double-precision floating-point representation. Similar to before, we compare the time needed to recover $99\%$ of all atoms for $\sparsity = O({\tilde \ddim}^{1/2})$ and $94.7\%$ of all atoms for $\sparsity = O(1)$. We can see from the results that \ac{ictkm} performs particularly well on high-dimensional signals. For the higher sparsity levels $\sparsity = O({\tilde \ddim}^{1/2})$ in \Cref{subfig:scalability_sparsity_1}, \ac{ictkm} with the \ac{dct} and a compression ratio of $2.5:1$ is $2.28 \times$ faster than \ac{itkrm} to recover the dictionary at the highest dimension tested, while the \ac{dft} and \ac{crt} are roughly $1.56 \times$ faster. For the lower sparsity levels $\sparsity = O(1)$ in \Cref{subfig:scalability_sparsity_2}, on the other hand, \ac{ictkm} performs significantly faster than \ac{itkrm}. In particular, \ac{ictkm} with the \ac{dct} and a compression ratio of $5:1$ is almost $10 \times$ faster than \ac{itkrm} to recover the dictionary for signals with $\ddim = 131,072$.

\begin{figure}[htp]  
\setlength\figureheight{\defaultheight} 
\setlength\figurewidth{\defaultwidth}
\centering  

\captionsetup[sub]{singlelinecheck=off,margin=\defaultsubcaptionmargin}
\subfloat[$\sparsity = O(\tilde{\ddim}^{1/2})$]
{
\label{subfig:scalability_sparsity_1} 
\definecolor{mycolor1}{rgb}{0.04762,0.52381,0.40000}%
\begin{tikzpicture}

\begin{axis}[%
tick scale binop=\times,
width=0.736\figurewidth,
height=\figureheight,
at={(0\figurewidth,0\figureheight)},
scale only axis,
xmode=log,
xmin=1024,
xmax=131072,
xtick={1024,10000,100000},
xticklabels={{$\text{10}^{\text{3}}$},{$\text{10}^{\text{4}}$},{$\text{10}^{\text{5}}$}},
xminorticks=true,
minor xtick={1500,3000,4500,6000,7500,10000,15000,20000,30000,40000,50000,100000},
xlabel style={font=\color{white!15!black}},
xlabel={ambient dimension},
ymin=-50000,
ymax=6000000,
ylabel style={font=\color{white!15!black}},
ylabel={cpu time (seconds)},
axis background/.style={fill=white},
axis x line*=bottom,
axis y line*=left,
legend style={at={(1.03,0.5)}, anchor=west, legend cell align=left, align=left, fill=none, draw=none},
legend style={font=\tiny}
]
\addlegendimage{legend image code/.code=}  
\addlegendentry{\\}
\addplot [color=mycolor1, dashdotted, mark=o, mark options={solid, mycolor1}]
  table[row sep=crcr]{%
2048	71000.8291666667\\
4096	111398.241166667\\
8192	188689.6825\\
16384	366281.617166667\\
24576	546472.857333333\\
32768	768980.428333333\\
49152	1164324.7165\\
65536	1684122.83416667\\
98304	2647040.18333333\\
131072	3281723.95833333\\
};
\addlegendentry{IcTKM-DFT,  3.3:1}

\addplot [color=blue, dotted, mark=x, mark options={solid, blue}]
  table[row sep=crcr]{%
2048	43782.7480666667\\
4096	70516.5991333333\\
8192	123777.082733333\\
16384	231728.8036\\
24576	356041.1954\\
32768	493274.167933333\\
49152	789722.265733333\\
65536	1127092.163\\
98304	1715818.85846667\\
131072	2252954.83426667\\
};
\addlegendentry{IcTKM-DCT,  2.5:1}

\addplot [color=red, dashed, mark=x, mark options={solid, red}]
  table[row sep=crcr]{%
2048	51892.608\\
4096	86120.2128\\
8192	153693.1152\\
16384	288265.7592\\
24576	441104.4288\\
32768	612740.0736\\
49152	1098538.8672\\
65536	1365274.6512\\
98304	2105582.3016\\
131072	3078257.9352\\
};
\addlegendentry{IcTKM-CRT,  2.5:1}

\addplot [color=black, mark=diamond, mark options={solid, black}]
  table[row sep=crcr]{%
2048	104682.806666667\\
4096	173126.32\\
8192	312605.406666667\\
16384	588158.81\\
24576	885509.116666667\\
32768	1202559.22\\
49152	1857778.36666667\\
65536	2494925.13\\
98304	3772313.88333333\\
131072	5137361.53666667\\
};
\addlegendentry{ITKrM}

\end{axis}
\end{tikzpicture}%
} 
\subfloat[$\sparsity = O(1)$]
{
\label{subfig:scalability_sparsity_2} 
\definecolor{mycolor1}{rgb}{0.23810,0.61905,0.40000}%
\definecolor{mycolor2}{rgb}{0.00000,0.04762,0.97619}%
\definecolor{mycolor3}{rgb}{1.00000,0.04762,0.00000}%
\begin{tikzpicture}

\begin{axis}[%
tick scale binop=\times,
width=0.734\figurewidth,
height=\figureheight,
at={(0\figurewidth,0\figureheight)},
scale only axis,
xmode=log,
xmin=1024,
xmax=131072,
xtick={1024,10000,100000},
xticklabels={{$\text{10}^{\text{3}}$},{$\text{10}^{\text{4}}$},{$\text{10}^{\text{5}}$}},
xminorticks=true,
minor xtick={1500,3000,4500,6000,7500,10000,15000,20000,30000,40000,50000,100000},
xlabel style={font=\color{white!15!black}},
xlabel={ambient dimension},
ymin=-50000,
ymax=4000000,
ylabel style={font=\color{white!15!black}},
ylabel={cpu time (seconds)},
axis background/.style={fill=white},
axis x line*=bottom,
axis y line*=left,
legend style={at={(1.03,0.5)}, anchor=west, legend cell align=left, align=left, fill=none, draw=none},
legend style={font=\tiny}
]
\addlegendimage{legend image code/.code=}  
\addlegendentry{\\}
\addplot [color=mycolor1, dashdotted, mark=square, mark options={solid, mycolor1}]
  table[row sep=crcr]{%
2048	16822.543\\
4096	36610.2906\\
8192	58979.7341333333\\
16384	98633.465\\
24576	137846.6454\\
32768	176939.436733333\\
49152	256448.108466667\\
65536	345284.755466667\\
98304	521310.5544\\
131072	726566.815933333\\
};
\addlegendentry{IcTKM-DFT, 10.0:1}

\addplot [color=mycolor2, dotted, mark=o, mark options={solid, mycolor2}]
  table[row sep=crcr]{%
2048	12637.7813666667\\
4096	21926.0220333333\\
8192	35601.4178333333\\
16384	61253.2182666667\\
24576	86588.456\\
32768	109809.211633333\\
49152	160361.625666667\\
65536	216483.610266667\\
98304	328181.8007\\
131072	438990.9938\\
};
\addlegendentry{IcTKM-DCT,  5.0:1}

\addplot [color=mycolor3, dashed, mark=o, mark options={solid, mycolor3}]
  table[row sep=crcr]{%
2048	13688.382\\
4096	24187.0764\\
8192	40864.4724\\
16384	71176.9884\\
24576	98893.8396\\
32768	127778.4312\\
49152	187145.8008\\
65536	254763.1056\\
98304	386396.1108\\
131072	518110.4064\\
};
\addlegendentry{IcTKM-CRT,  5.0:1}

\addplot [color=black, mark=diamond, mark options={solid, black}]
  table[row sep=crcr]{%
2048	85832.45\\
4096	141957.603333333\\
8192	254726.263333333\\
16384	467920.146666667\\
24576	671109.716666667\\
32768	883394.776666666\\
49152	1298444.76\\
65536	1731548.78333333\\
98304	2578560.36\\
131072	3422511.24333333\\
};
\addlegendentry{ITKrM}

\end{axis}
\end{tikzpicture}%
} 
\caption{Scalability of dictionary recovery time with ambient dimension and compression ratio.}
\label{fig:scalability}
\end{figure}
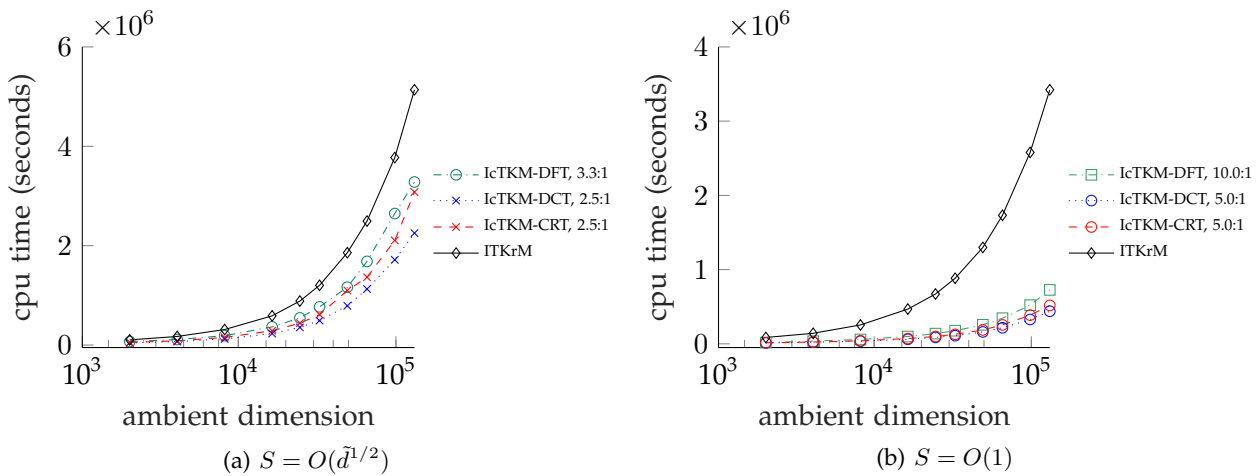

\subsection{Audio data}

For the real data we have selected several recordings comprised of stereo audio signals sampled at $44.1$ KHz. Details for the audio training data and the recovery simulations are described next. 

\textbf{Audio recordings.} We have used three audio recordings of roughly 10 minutes each for carrying out the simulations. The recordings represent distinct musical genres: classical, folk Japanese and Flamenco musical pieces, which have been obtained from RWC's classical and music genre databases. The classical piece is the first movement of a piano sonata by Mozart. The folk Japanese piece are \emph{min'y{\=o}} traditional songs comprised of female vocal, \emph{shamisen} (three-stringed instrument), \emph{shakuhachi} (an end-blown long flute), \emph{shinobue} (a high-pitched short flute), and traditional percussion instruments. The Flamenco piece is solely comprised of male vocal, and guitar which also acts as a percussive instrument. \\
\indent \textbf{Block size/overlap.} We have first summed the audio signals to mono, and then partitioned the resulting signals into smaller blocks. Short duration blocks of $0.25$ seconds and long ones of $1$ second have been used. The blocks were allowed to overlap such that the maximally allowed amount of overlap of one block with a shifted version of itself varied from $95 \%$ for the short block, up to $98.75 \%$ for the long block. \\
\indent \textbf{Training signals.} The dictionaries have been learned directly from the time-domain samples of our musical recordings, with each audio block assigned to one training signal. The short and long blocks amount to training signals with ambient dimension of $\ddim = 11,025$ and $\ddim = 44,100$, respectively. The number of training signals for the three audio recordings were approximately $\nsig=48,000$ for the classical piece, $\nsig=42,000$ for the folk Japanese, and $\nsig=59,000$ for the Flamenco piece. \\
\indent \textbf{Learning parameters.} We have carried out the learning simulations with two dictionary sizes of $\natoms = 64$ and $\natoms = 256$ atoms, and the sparsity level was fixed at $\sparsity = 4$. To learn dictionaries on the audio data, we ran $200$ iterations of \ac{ictkm} with a \ac{dct} based \ac{jl} embedding and a compression ratio of $5:1$. 

Next, we will explore the ability of \ac{ictkm} to learn audio dictionaries for extracting notes of the musical recordings. We will also take a look at how increased ambient dimensions can be used to improve the tone quality of the audio representations.

\subsubsection{Extracting musical notes}

In \Cref{fig:dictionary_spectra_classical_piano,fig:dictionary_spectra_folk_japanese,fig:dictionary_spectra_flamenco} we evaluate the magnitude spectra of the recovered atoms for the classical piano, folk Japanese, and Flamenco recordings. The learning simulations have been carried out with short duration blocks of $0.25$ seconds, and the atoms in the learned dictionaries have been sorted by their fundamental frequency. We can see from the results that the larger dictionary is able to capture more musical notes than the smaller one. In particular, for the smaller dictionary we have identified $26$ unique fundamental frequencies in the range $[108, 992]$ Hz for the classical piano recording, and $24$ frequencies in $[92, 440]$ Hz for the Flamenco recording. For the larger dictionary, on the other hand, we have $55$ unique fundamental frequencies in the range $[108, 1408]$ Hz for the classical piano, and $56$ frequencies in $[88, 524]$ Hz for the Flamenco recording. These unique fundamental frequencies correspond to notes of the Western equally tempered 12 tone scale found in the recordings. For the folk Japanese recording we have found more notes in a larger frequency range; $31$ unique fundamental frequencies in the range $[132, 1584]$ Hz for the smaller dictionary, and $98$ frequencies in $[128, 1784]$ Hz for the larger dictionary. We can further see from the results that the learned dictionaries sometimes have multiple atoms with same fundamental frequency, but these equally pitched atoms usually differ in their harmonic structure. 

\begin{figure}[htp]  
\setlength\figureheight{\defaultheight} 
\setlength\figurewidth{\defaultwidth}
\centering  
\captionsetup[sub]{singlelinecheck=off,margin=\defaultsubcaptionmargin}
\subfloat[$\natoms = 64$]
{
\label{subfig:1} 
\begin{tikzpicture}

\begin{axis}[%
width=0.866\figurewidth,
height=\figureheight,
at={(0\figurewidth,0\figureheight)},
scale only axis,
point meta min=-67.4154038692758,
point meta max=37.2758523669715,
axis on top,
xmin=0.5,
xmax=64.5,
xlabel style={font=\color{white!15!black}},
xlabel={Atom number},
ymin=0.5,
ymax=1002.5,
xtick={16,32,48,64},
ytick={1,221.48,441.96,662.44,882.92,1103.4,1213.64,1323.88,1434.12,1544.36,1654.6,1764.84,1875.08,1985.32,2095.56,2205.8,2316.04,2426.28,2536.52,2646.76,2757,2867.24,2977.48,3087.72,3197.96,3308.2,3418.44,3528.68,3638.92,3749.16,3859.4,3969.64,4079.88,4190.12,4300.36,4410.6,4520.84,4631.08,4741.32,4851.56,4961.8,5072.04,5182.28,5292.52,5402.76,5513},
yticklabels={{0},{0.8},{1.7},{2.6},{3.5},{4410},{4851},{5292},{5733},{6174},{6615},{7056},{7497},{7938},{8379},{8820},{9261},{9702},{10143},{10584},{11025},{11466},{11907},{12348},{12789},{13230},{13671},{14112},{14553},{14994},{15435},{15876},{16317},{16758},{17199},{17640},{18081},{18522},{18963},{19404},{19845},{20286},{20727},{21168},{21609},{22050}},
ylabel style={font=\color{white!15!black}},
ylabel={Frequency (kHz)},
axis background/.style={fill=white},
title style={font=\bfseries},
legend style={font=\tiny},
colormap={mymap}{[1pt] rgb(0pt)=(1,1,1); rgb(63pt)=(0,0,0)},
colorbar,
colorbar style={
			ylabel=(dB),
			ytick={35,10,-15,-40,-65},
		}
]
\addplot [forget plot] graphics [xmin=0.5, xmax=64.5, ymin=0.5, ymax=1002.5] {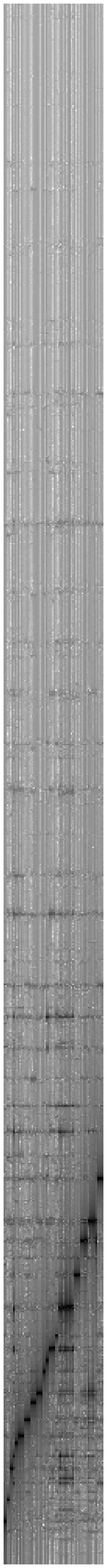};
\end{axis}
\end{tikzpicture}%
} 
\subfloat[$\natoms = 256$]
{
\label{subfig:_2} 
\begin{tikzpicture}

\begin{axis}[%
width=0.866\figurewidth,
height=\figureheight,
at={(0\figurewidth,0\figureheight)},
scale only axis,
point meta min=-82.2189883914884,
point meta max=37.2871689764193,
axis on top,
xmin=0.5,
xmax=256.5,
xlabel style={font=\color{white!15!black}},
xlabel={Atom number},
ymin=0.5,
ymax=1002.5,
xtick={64,128,192,256},
ytick={1,221.48,441.96,662.44,882.92,1103.4,1213.64,1323.88,1434.12,1544.36,1654.6,1764.84,1875.08,1985.32,2095.56,2205.8,2316.04,2426.28,2536.52,2646.76,2757,2867.24,2977.48,3087.72,3197.96,3308.2,3418.44,3528.68,3638.92,3749.16,3859.4,3969.64,4079.88,4190.12,4300.36,4410.6,4520.84,4631.08,4741.32,4851.56,4961.8,5072.04,5182.28,5292.52,5402.76,5513},
yticklabels={{0},{0.8},{1.7},{2.6},{3.5},{4410},{4851},{5292},{5733},{6174},{6615},{7056},{7497},{7938},{8379},{8820},{9261},{9702},{10143},{10584},{11025},{11466},{11907},{12348},{12789},{13230},{13671},{14112},{14553},{14994},{15435},{15876},{16317},{16758},{17199},{17640},{18081},{18522},{18963},{19404},{19845},{20286},{20727},{21168},{21609},{22050}},
ylabel style={font=\color{white!15!black}},
ylabel={Frequency (kHz)},
axis background/.style={fill=white},
title style={font=\bfseries},
legend style={font=\tiny},
colormap={mymap}{[1pt] rgb(0pt)=(1,1,1); rgb(63pt)=(0,0,0)},
colorbar,
colorbar style={
			ylabel=(dB),
			ytick={35,5,-25,-55,-80},
		}
]
\addplot [forget plot] graphics [xmin=0.5, xmax=256.5, ymin=0.5, ymax=1002.5] {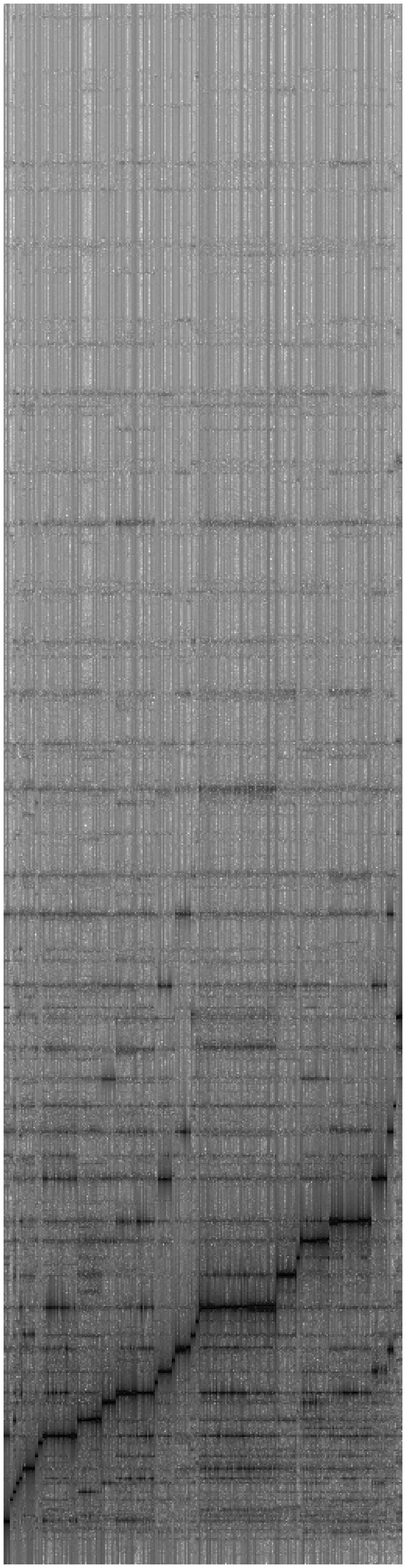};
\end{axis}
\end{tikzpicture}%
} 
\caption{Spectra of recovered dictionaries for the classical piano recording using increasing dictionary sizes.}
\label{fig:dictionary_spectra_classical_piano}
\end{figure}

\begin{figure}[htp]  
\setlength\figureheight{\defaultheight} 
\setlength\figurewidth{\defaultwidth}
\centering  
\captionsetup[sub]{singlelinecheck=off,margin=\defaultsubcaptionmargin}
\subfloat[$\natoms = 64$]
{
\label{subfig:3} 
\begin{tikzpicture}

\begin{axis}[%
width=0.866\figurewidth,
height=\figureheight,
at={(0\figurewidth,0\figureheight)},
scale only axis,
point meta min=-69.3723567795034,
point meta max=37.211373231432,
axis on top,
xmin=0.5,
xmax=64.5,
xlabel style={font=\color{white!15!black}},
xlabel={Atom number},
ymin=0.5,
ymax=1002.5,
xtick={16,32,48,64},
ytick={1,221.48,441.96,662.44,882.92,1103.4,1213.64,1323.88,1434.12,1544.36,1654.6,1764.84,1875.08,1985.32,2095.56,2205.8,2316.04,2426.28,2536.52,2646.76,2757,2867.24,2977.48,3087.72,3197.96,3308.2,3418.44,3528.68,3638.92,3749.16,3859.4,3969.64,4079.88,4190.12,4300.36,4410.6,4520.84,4631.08,4741.32,4851.56,4961.8,5072.04,5182.28,5292.52,5402.76,5513},
yticklabels={{0},{0.8},{1.7},{2.6},{3.5},{4410},{4851},{5292},{5733},{6174},{6615},{7056},{7497},{7938},{8379},{8820},{9261},{9702},{10143},{10584},{11025},{11466},{11907},{12348},{12789},{13230},{13671},{14112},{14553},{14994},{15435},{15876},{16317},{16758},{17199},{17640},{18081},{18522},{18963},{19404},{19845},{20286},{20727},{21168},{21609},{22050}},
ylabel style={font=\color{white!15!black}},
ylabel={Frequency (kHz)},
axis background/.style={fill=white},
title style={font=\bfseries},
legend style={font=\tiny},
colormap={mymap}{[1pt] rgb(0pt)=(1,1,1); rgb(63pt)=(0,0,0)},
colorbar,
colorbar style={
			ylabel=(dB),
			ytick={35,10,-15,-40,-65},
		}
]
\addplot [forget plot] graphics [xmin=0.5, xmax=64.5, ymin=0.5, ymax=1002.5] {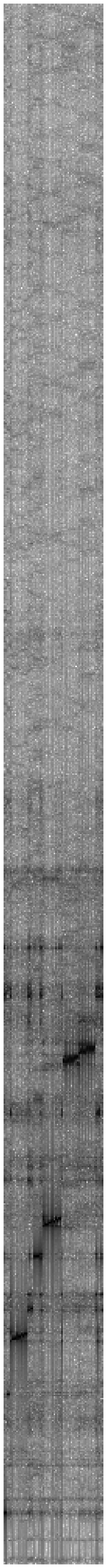};
\end{axis}
\end{tikzpicture}%
} 
\subfloat[$\natoms = 256$]
{
\label{subfig:_4} 
\begin{tikzpicture}

\begin{axis}[%
width=0.866\figurewidth,
height=\figureheight,
at={(0\figurewidth,0\figureheight)},
scale only axis,
point meta min=-89.7453866946373,
point meta max=37.1694938132869,
axis on top,
xmin=0.5,
xmax=256.5,
xlabel style={font=\color{white!15!black}},
xlabel={Atom number},
ymin=0.5,
ymax=1002.5,
xtick={64,128,192,256},
ytick={1,221.48,441.96,662.44,882.92,1103.4,1213.64,1323.88,1434.12,1544.36,1654.6,1764.84,1875.08,1985.32,2095.56,2205.8,2316.04,2426.28,2536.52,2646.76,2757,2867.24,2977.48,3087.72,3197.96,3308.2,3418.44,3528.68,3638.92,3749.16,3859.4,3969.64,4079.88,4190.12,4300.36,4410.6,4520.84,4631.08,4741.32,4851.56,4961.8,5072.04,5182.28,5292.52,5402.76,5513},
yticklabels={{0},{0.8},{1.7},{2.6},{3.5},{4410},{4851},{5292},{5733},{6174},{6615},{7056},{7497},{7938},{8379},{8820},{9261},{9702},{10143},{10584},{11025},{11466},{11907},{12348},{12789},{13230},{13671},{14112},{14553},{14994},{15435},{15876},{16317},{16758},{17199},{17640},{18081},{18522},{18963},{19404},{19845},{20286},{20727},{21168},{21609},{22050}},
ylabel style={font=\color{white!15!black}},
ylabel={Frequency (kHz)},
axis background/.style={fill=white},
title style={font=\bfseries},
legend style={font=\tiny},
colormap={mymap}{[1pt] rgb(0pt)=(1,1,1); rgb(63pt)=(0,0,0)},
colorbar,
colorbar style={
			ylabel=(dB),
			ytick={35,5,-25,-55,-85},
		}
]
\addplot [forget plot] graphics [xmin=0.5, xmax=256.5, ymin=0.5, ymax=1002.5] {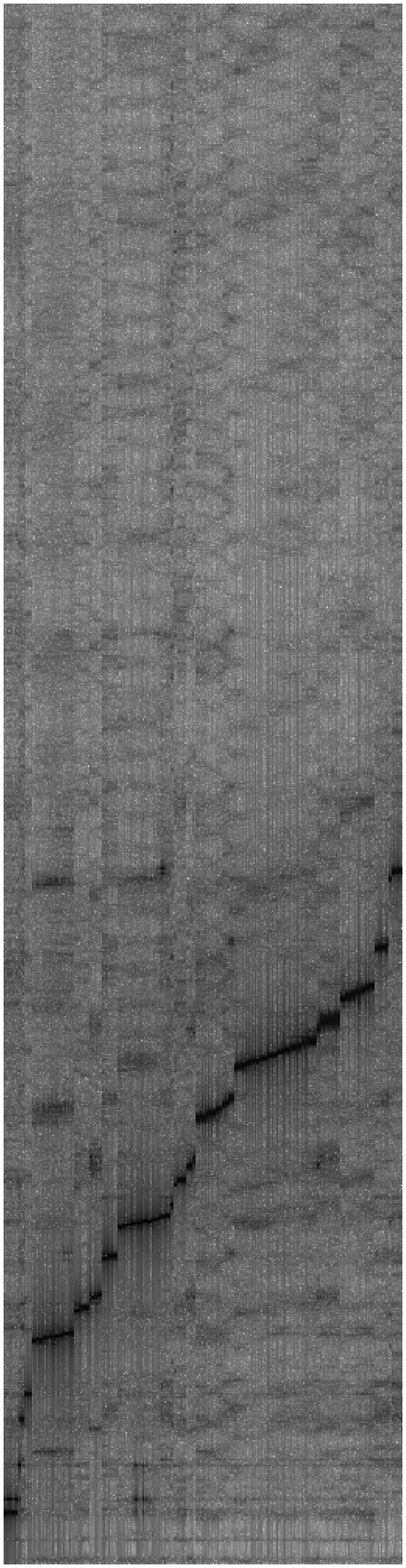};
\end{axis}
\end{tikzpicture}%
} 
\caption{Spectra of recovered dictionaries for the folk Japanese recording using increasing dictionary sizes.}
\label{fig:dictionary_spectra_folk_japanese}
\end{figure}

\begin{figure}[htp]  
\setlength\figureheight{\defaultheight} 
\setlength\figurewidth{\defaultwidth}
\centering  
\captionsetup[sub]{singlelinecheck=off,margin=\defaultsubcaptionmargin}
\subfloat[$\natoms = 64$]
{
\label{subfig:5} 
\begin{tikzpicture}

\begin{axis}[%
width=0.866\figurewidth,
height=\figureheight,
at={(0\figurewidth,0\figureheight)},
scale only axis,
point meta min=-87.4715073545589,
point meta max=36.8385320128563,
axis on top,
xmin=0.5,
xmax=64.5,
xlabel style={font=\color{white!15!black}},
xlabel={Atom number},
ymin=0.5,
ymax=1002.5,
xtick={16,32,48,64},
ytick={1,221.48,441.96,662.44,882.92,1103.4,1213.64,1323.88,1434.12,1544.36,1654.6,1764.84,1875.08,1985.32,2095.56,2205.8,2316.04,2426.28,2536.52,2646.76,2757,2867.24,2977.48,3087.72,3197.96,3308.2,3418.44,3528.68,3638.92,3749.16,3859.4,3969.64,4079.88,4190.12,4300.36,4410.6,4520.84,4631.08,4741.32,4851.56,4961.8,5072.04,5182.28,5292.52,5402.76,5513},
yticklabels={{0},{0.8},{1.7},{2.6},{3.5},{4410},{4851},{5292},{5733},{6174},{6615},{7056},{7497},{7938},{8379},{8820},{9261},{9702},{10143},{10584},{11025},{11466},{11907},{12348},{12789},{13230},{13671},{14112},{14553},{14994},{15435},{15876},{16317},{16758},{17199},{17640},{18081},{18522},{18963},{19404},{19845},{20286},{20727},{21168},{21609},{22050}},
ylabel style={font=\color{white!15!black}},
ylabel={Frequency (kHz)},
axis background/.style={fill=white},
title style={font=\bfseries},
legend style={font=\tiny},
colormap={mymap}{[1pt] rgb(0pt)=(1,1,1); rgb(63pt)=(0,0,0)},
colorbar,
colorbar style={
			ylabel=(dB),
			ytick={35,5,-25,-55,-85},
		}
]
\addplot [forget plot] graphics [xmin=0.5, xmax=64.5, ymin=0.5, ymax=1002.5] {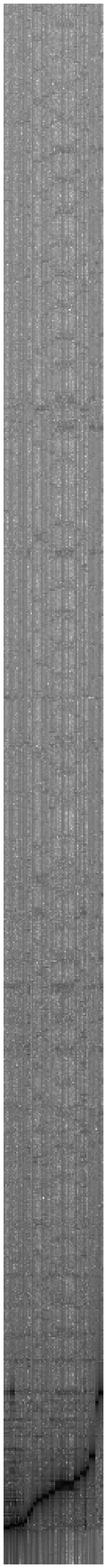};
\end{axis}
\end{tikzpicture}%
} 
\subfloat[$\natoms = 256$]
{
\label{subfig:6} 
\begin{tikzpicture}

\begin{axis}[%
width=0.866\figurewidth,
height=\figureheight,
at={(0\figurewidth,0\figureheight)},
scale only axis,
point meta min=-82.5021253427667,
point meta max=37.2094852799025,
axis on top,
xmin=0.5,
xmax=256.5,
xlabel style={font=\color{white!15!black}},
xlabel={Atom number},
ymin=0.5,
ymax=1002.5,
ytick={1,221.48,441.96,662.44,882.92,1103.4,1213.64,1323.88,1434.12,1544.36,1654.6,1764.84,1875.08,1985.32,2095.56,2205.8,2316.04,2426.28,2536.52,2646.76,2757,2867.24,2977.48,3087.72,3197.96,3308.2,3418.44,3528.68,3638.92,3749.16,3859.4,3969.64,4079.88,4190.12,4300.36,4410.6,4520.84,4631.08,4741.32,4851.56,4961.8,5072.04,5182.28,5292.52,5402.76,5513},
yticklabels={{0},{0.8},{1.7},{2.6},{3.5},{4410},{4851},{5292},{5733},{6174},{6615},{7056},{7497},{7938},{8379},{8820},{9261},{9702},{10143},{10584},{11025},{11466},{11907},{12348},{12789},{13230},{13671},{14112},{14553},{14994},{15435},{15876},{16317},{16758},{17199},{17640},{18081},{18522},{18963},{19404},{19845},{20286},{20727},{21168},{21609},{22050}},
ylabel style={font=\color{white!15!black}},
ylabel={Frequency (kHz)},
axis background/.style={fill=white},
legend style={font=\tiny},
colormap={mymap}{[1pt] rgb(0pt)=(1,1,1); rgb(63pt)=(0,0,0)},
colorbar,
colorbar style={
			ylabel=(dB),
			ytick={35,5,-25,-55,-80},
		}
]
\addplot [forget plot] graphics [xmin=0.5, xmax=256.5, ymin=0.5, ymax=1002.5] {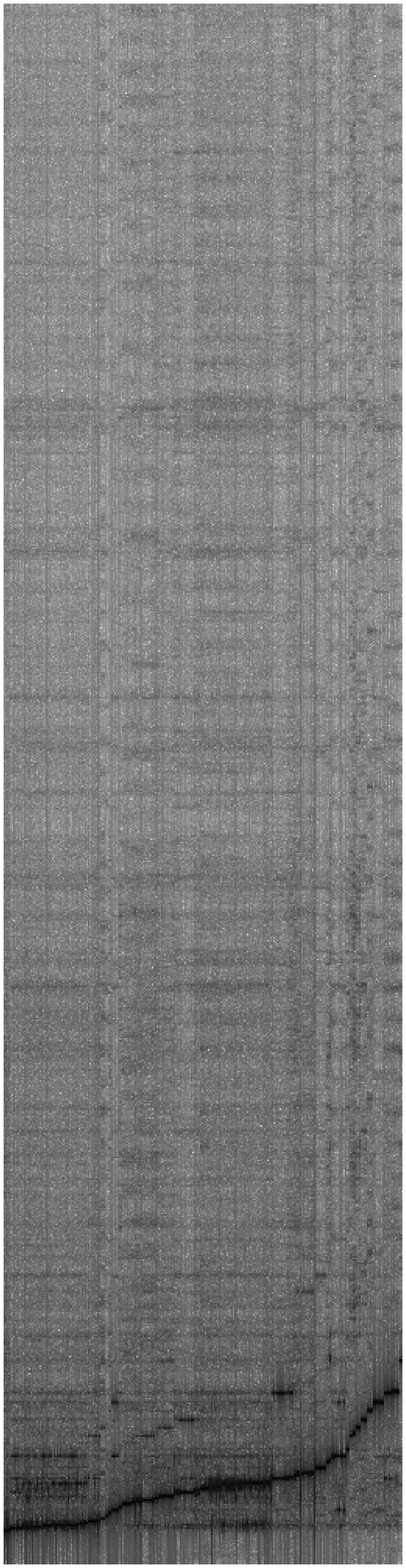};
\end{axis}
\end{tikzpicture}%
}
\caption{Spectra of recovered dictionaries for the Flamenco recording using increasing dictionary sizes.}
\label{fig:dictionary_spectra_flamenco}
\end{figure}

\subsubsection{Tone quality}

The learned dictionaries have been found to possess musical notes with distinct tone quality, and this is a direct consequence of the different musical genres and instruments in the audio training data\footnote{We encourage the reader to listen to the sonification of the learned dictionaries. The dictionary audio files can be found at the provided MATLAB toolbox package.}. In particular, the musical notes found in the dictionary of the classical piece have a distinct piano tone quality, while in the Flamenco piece the notes usually have a mixed tone quality reflecting pitched notes from the guitar/vocals and percussive, un-pitched sounds from tapping the guitar plate and plucking its strings. In the dictionary for the folk Japanese piece the lower-pitched atoms have a distinct drum tone quality while the mid- and high-pitched ones resemble the tone quality of the traditional Japanese flutes. The harmonic content of the female vocal can be found in many atoms of the learned dictionary, which gives them a distinct chorus-like sound quality. 

In \Cref{fig:2} we evaluate the spectrograms for atoms of the folk Japanese dictionary. Dictionaries have been learned with the short and long audio blocks, and their atoms have been similarly sorted by fundamental frequency. \Cref{subfig:5} shows the spectrograms of the atoms number $15$ of the dictionaries learned with short blocks (on the left) and long blocks (on the right). Similarly, \Cref{subfig:6} shows the spectrograms of the atoms number $183$. As can be seen from these figures, the learned dictionaries can extract similar musical notes, but the higher-dimensional training signals promote notes with a much richer harmonic structure. This intricate harmonic structure translates to dictionaries where the individual instruments and vocals in the musical piece can be more easily identified.

\begin{figure}[htp]  
\setlength\figureheight{\defaultheight} 
\setlength\figurewidth{\defaultwidth}
\centering  

\captionsetup[sub]{singlelinecheck=off,margin=\defaultsubcaptionmargin}
\subfloat[$\ppatom_{15}$]
{
\label{subfig:5} 
\begin{tikzpicture}

\begin{axis}[%
width=0.188\figurewidth,
height=\figureheight,
at={(0\figurewidth,0\figureheight)},
scale only axis,
point meta min=-156.082029428581,
point meta max=-37.6383848559123,
axis on top,
xmin=5,
xmax=250,
xtick={250},
ymin=0,
ymax=6,
ylabel style={font=\color{white!15!black}},
ylabel={Frequency (kHz)},
axis background/.style={fill=white},
legend style={font=\tiny}
]
\addplot [forget plot] graphics [xmin=5, xmax=245, ymin=-0.002, ymax=22.05] {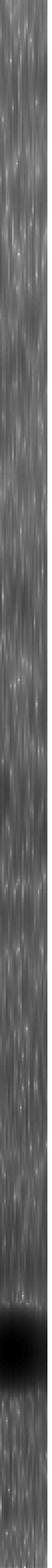};
\end{axis}

\begin{axis}[%
width=0.54\figurewidth,
height=\figureheight,
at={(0.282\figurewidth,0\figureheight)},
scale only axis,
point meta min=-156.512717222255,
point meta max=-40.9868294740555,
axis on top,
xmin=5,
xmax=1000,
xlabel style={font=\color{white!15!black}},
xlabel={Time (ms)},
xtick={250, 500,750,1000},
ymin=0,
ymax=6,
ytick={\empty},
axis background/.style={fill=white},
legend style={font=\tiny},
colormap={mymap}{[1pt] rgb(0pt)=(1,1,1); rgb(63pt)=(0,0,0)},
colorbar,
colorbar style={ylabel=(dB), ytick={-45,-75,-100,-125,-155}}
]
\addplot [forget plot] graphics [xmin=5, xmax=995, ymin=-0.002, ymax=22.05] {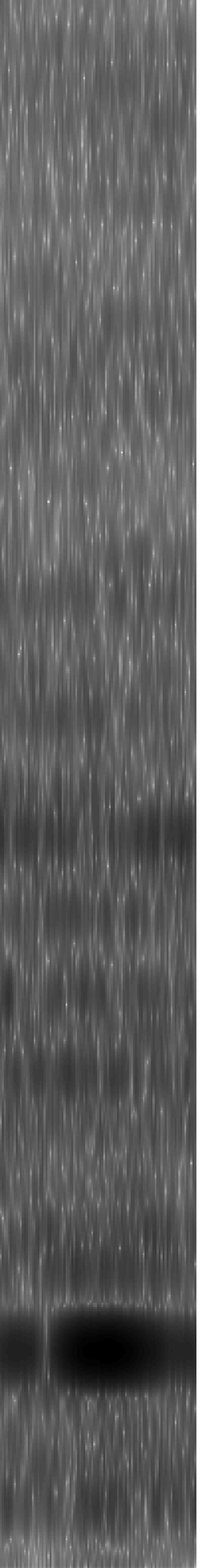};
\end{axis}
\end{tikzpicture}%
} 
\subfloat[$\ppatom_{183}$]
{
\label{subfig:6} 
\begin{tikzpicture}

\begin{axis}[%
width=0.188\figurewidth,
height=\figureheight,
at={(0\figurewidth,0\figureheight)},
scale only axis,
point meta min=-156.157261756311,
point meta max=-33.844793620919,
axis on top,
xmin=5,
xmax=250,
xtick={250},
ymin=0,
ymax=6,
ylabel style={font=\color{white!15!black}},
ylabel={Frequency (kHz)},
axis background/.style={fill=white},
legend style={font=\tiny}
]
\addplot [forget plot] graphics [xmin=5, xmax=245, ymin=-0.002, ymax=22.05] {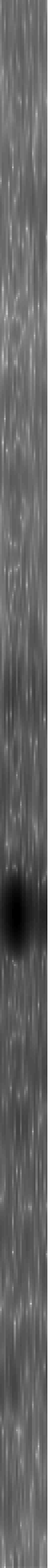};
\end{axis}

\begin{axis}[%
width=0.54\figurewidth,
height=\figureheight,
at={(0.282\figurewidth,0\figureheight)},
scale only axis,
point meta min=-156.527378684222,
point meta max=-36.4402310501979,
axis on top,
xmin=5,
xmax=1000,
xlabel style={font=\color{white!15!black}},
xlabel={Time (ms)},
xtick={250, 500,750,1000},
ymin=0,
ymax=6,
ytick={\empty},
axis background/.style={fill=white},
legend style={font=\tiny},
colormap={mymap}{[1pt] rgb(0pt)=(1,1,1); rgb(63pt)=(0,0,0)},
colorbar,
colorbar style={ylabel=(dB), ytick={-45,-75,-100,-125,-155}}
]
\addplot [forget plot] graphics [xmin=5, xmax=995, ymin=-0.002, ymax=22.05] {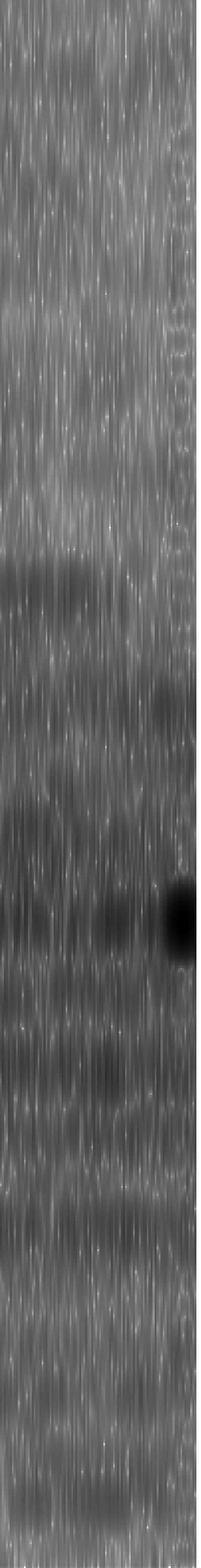};
\end{axis}
\end{tikzpicture}%
}  
\caption{Spectrograms of atoms recovered with increasing ambient dimensions.}
\label{fig:2}
\end{figure}

%%%%%%%%%%%%%%%%%%%%%%%%%%%%%%%%%%%%%%%%%%%%%%%%%%%%
%%%%%%%%%%%%%%%%%%%%%CONCLUSIONS%%%%%%%%%%%%%%%%%%%%
%%%%%%%%%%%%%%%%%%%%%%%%%%%%%%%%%%%%%%%%%%%%%%%%%%%%
\section{Conclusions}\label{sec:conclusions}

We have shown in this work that \ac{ictkm} is a powerful, low-computational cost algorithm for learning dictionaries. Given a recovery error $\targeterror$ and failure probability $O(\jlprobterm \log(1/\targeterror))$, \ac{ictkm} will recover an incoherent generating dictionary from noisy signals with sparsity levels $\sparsity \leq O(\ddim/(\log (\natoms^2/\targeterror) )$ and \ac{snr} of order $O(1)$ in $O(\log(1/\targeterror))$ iterations, as long as the initial dictionary is within a radius $O(1/\sqrt{\max\{S, \log K\}}$ to the generating dictionary, and in each iteration a new batch of $\nsig = O(\targeterror^{-2}\natoms \log(\natoms/\jlprobterm))$ training signals as well as a new random embedding with embedding dimension $\jlembeddingdim$ satisfying $\jlembeddingdim\geq O(\sparsity \log^4(\sparsity) \log^3(\natoms/\jlprobterm))$ is used.
This means that the computational cost of dictionary recovery with \ac{ictkm} can be reduced by a factor $O(\sparsity/\ddim \cdot \log^4(\sparsity) \log^3(\natoms/\jlprobterm))$, and lets us conclude that \ac{ictkm} is an appealing algorithm for learning dictionaries from high-dimensional signals, particularly in learning tasks with heavily-sparse data under controlled noise levels.

We have further demonstrated with numerical experiments that \ac{ictkm} can stably recover dictionaries with low computational cost in a practical setting. For synthetic signals, we had successfully carried out the learning process with high compression ratios, even when low-dimensional data were used. We have also seen that \ac{ictkm} scales quite well with increasing ambient dimensions. For high-dimensional signals with roughly a tenth-of-a-million dimensions, we were able to speed-up the learning process by up to an order of magnitude. Further, \ac{ictkm} has been shown to be a powerful algorithm for learning dictionaries from high-dimensional audio data. The learned dictionaries worked particularly well for extracting notes from large musical pieces. We have further seen that the ability to learn dictionaries from high-dimensional audio signals allows us to more easily identify individual instruments from musical pieces. The learned dictionaries have been found to contain richer harmonic structure directly corresponding to the musical instruments, particularly for the longer-duration training data.

There are a few research directions we would like to pursue for future work. In \ac{nmf}, the principle of learning representations by the (additive) combination of multiple bases is achieved via the matrix factorization $\trainmatrix = \dico \coeffmatrix$, where $\dico$ and $\coeffmatrix$ are only allowed to have non-negative valued entries, see \cite{leeung99}. Sparse \ac{nmf}, where $\coeffmatrix$ is required to be sparse in addition to non-negative, has been shown to work quite well in audio processing tasks such as pitch detection, automatic music transcription, and source separation \cite{leylahmerath06,nen07,leyathdetvalies10,dis04}. In these applications the learning process is typically carried out in the frequency domain, and thus the data matrix $\trainmatrix$ is usually given by the power spectrum of the audio training data. Addressing sparse \ac{nmf} problems with \ac{itkm} based algorithms is a line of inquiry we would like to pursue. Non-negative \ac{itkm} requires straightforward adjustments to the sparse coefficient update formula to ensure that the updated dictionary is non-negative. 

Inspired by masked dictionary learning, see \cite{nasc18}, another line of inquiry we want to pursue is \emph{blind compressed sensing}. In the masked setting we are asked to recover the generating dictionary $\dico$ from training data which has been corrupted or lost via a binary erasure channel. Data corruption is typically modeled with the concept of a binary mask $\mask$, a $\ddim \times \ddim$ diagonal matrix with $0$ or $1$ entries in the diagonal, where the corrupted data is thus given by $\mask \trainmatrix$, see \cite{nasc18} for details. A practical application of masked dictionary learning is \emph{image inpainting}, which is based on the observation that if the training data is $\sparsity$-sparse in the generating dictionary, then the corrupted data must also be sparse in the corrupted dictionary. In other words, if $\trainmatrix = \dico \trainmatrix$ then $\mask \trainmatrix = \mask \dico \coeffmatrix$, where $\coeffmatrix$ is sparse. We can see that masked dictionary learning is closely related to compressed dictionary learning, in the sense that the mask $\mask$ has a similar role to $\jlmatrix$. By erasing the data with zeros in the diagonal of $\mask$ we are effectively reducing the dimension of our learning problem. However, since the erasures occur always in different coordinates we are able to observe the signals on different coordinates for each mask and combining these observations allows us to recover the full signal.
To employ these concepts for compressed dictionary learning, we simply need to choose the masks such that they behave as a low-distortion embedding similar to the \ac{jl} lemma. Conversely, we can use our dimensionality reduction tools to study the theoretical properties of masked dictionary learning, and thus to prove local convergence of the ITKrMM algorithm presented in \cite{nasc18}. Finally, note that such a combination of compressed and masked dictionary learning supported with theoretical guarantees would be a big step towards blind compressed sensing \cite{mandar10}.

%%%%%%%%%%%%%%%%%%%%%%%%%%%%%%%%%%%%%%%%%%%%%%%%%%%%
%%%%%%%%%%%%%%%%%%%%%%%%%%%%%%%%%%%%%%%%%%%%%%%%%%%%
\subsection*{Acknowledgment}
This work was supported by the Austrian Science Fund (FWF) under Grant no.~Y760. The computational results presented have been achieved (in part) using the HPC infrastructure LEO of the University of Innsbruck. Further, we gratefully acknowledge the support of NVIDIA Corporation with the donation of a GPU used for the synthetic simulations.\\
Part of the work on the audio simulations was done while F. Teixeira was affiliated to the University of Innsbruck and visited the Parsimony and New Algorithms for Audio and Signal Modeling (PANAMA) team at Inria Rennes, Bretagne Atlantique. He wishes to thank the PANAMA team for its hospitality. \\
Finally, we would like to thank the anonymous reviewer for references and suggestions to improve the theoretical results and their presentation as well as Marie Pali and Simon Ruetz for proofreading the manuscript at various stages.
 %%%%%%%%%%%%%%%%%%%%%%%%%%%%%%%%%%%%%%%%%%%%%%%%%%%%
%%%%%%%%%%%%%%%%%%%%%%%%%%%%%%%%%%%%%%%%%%%%%%%%%%%% 
\clearpage
\appendices
%%%%%%%%%%%%%%%%%%%%%%%%%%%%%%%%%%%%%%%%%%%%%%%%%%%%
%%%%%%%%%%%%%%%%%%%%%APPENDIX: PROOF%%%%%%%%%%%%%%%%%
%%%%%%%%%%%%%%%%%%%%%%%%%%%%%%%%%%%%%%%%%%%%%%%%%%%%

\section{Exact Statement and Proof of \texorpdfstring{\Cref{th:ictkm_convergence_o_notation}}{Theorem \ref{th:ictkm_convergence_o_notation}}}

\label[secinapp]{sec:proof_of_ictkm_convergence_theorem}

Before we can state and prove the exact version of \Cref{th:ictkm_convergence_o_notation} we have to introduce additional notation and a few statistics for our signal model in \labelcref{eq:training_signal_with_model}. First, we refer to the position of the largest $\sparsity$ terms of $\coeffvector$ (in magnitude) as the oracle support $\indexset^{o} = p^{-1}(\{1\ldots S\})$.
On top of the already defined absolute $\csum_{\sparsity}$ and relative $\Pert_\sparsity > \csum_{\sparsity}$ gaps, we also set
for a $\nu_c$-random draw from our coefficient set $\mathcal C$ the following statistics
\begin{align} 
&\statinputone : = \E_{\coeffsequence} \left( \coeffsequence(1) + \cdots + \coeffsequence(\sparsity) \right), \hspace{1em} \statinputotwo := \E_{\coeffsequence} \left( \coeffsequence^2(1) + \cdots + \coeffsequence^2(\sparsity) \right), \quad \text{and} \quad \snrinputsig := \E_{\noise} \left( \frac{1}{\sqrt{1 + \left\| \noise \right\|_2^2}} \right).
\label{eq:signal_statistics}	
\end{align}
We can think of the constant $\statinputotwo$ as the expected energy of the $S$-sparse approximations of our signals, meaning $\statinputotwo = \E(\| \dico_{\indexset^{o}} \coeffvector_{\indexset^{o}} \|_2^2)$. The constant $\statinputone$ is related to the expected size of a coefficient in the sparse approximation via $\statinputone = \E ( | \coeffvector(i) | : i \in \indexset^{o} ) \cdot \sparsity $. In particular, we have the following bounds $\statinputotwo \leq 1$ and $\sparsity \csum_{\sparsity} \leq \statinputone \leq \sqrt{\sparsity}$. For the simple distribution based on $\coeffsequence(k) = 1/\sqrt{S}$ for $k\leq S$ and $0$ otherwise, we have the equality $\statinputone = \sqrt{\sparsity}$ since $\csum_{\sparsity} = \coeffsequence(\sparsity) = 1/\sqrt{S}$ in this simplification.
Finally, the constant $\snrinputsig$ can be lower bounded by $\snrinputsig \geq \frac{1-e^{-\ddim}}{\sqrt{1 + 5\ddim \nsigma^2}}$, see \cite{hskazh11, sc14b} and thus for large $\nsigma$ we have $\snrinputsig^2 \approx 1/\ddim \nsigma^2 \approx \E (\| \dico \coeffvector \|_2^2)/\E ( \| \noise \|_2^2 )$, and so we can think of $\snrinputsig$ as the signal-to-noise ratio.\\
 We are now ready to state the general version of \Cref{th:ictkm_convergence_o_notation}.
%%%%%%%%%%% Theorem %%%%%%%%%%%%
\begin{theorem}\label{th:ictkm_convergence}
Let the generating dictionary $\dico$ have coherence $\coherence$ and operator norm $\|\dico\|^2_{2,2} \leq \frameupbound$. Assume that the training signals $\trainsignal_n$ follow the signal model in \labelcref{eq:training_signal_with_model} with coefficients that have an absolute gap $\csum_{\sparsity} > 0$ and a relative gap $\Pert_{\sparsity} > 0$. Further assume that $\sparsity \leq \frac{1}{98} \min\left\{ \frac{\natoms}{\frameupbound}, \frac{1}{ \nsigma^2} \right\}$ and $\eps_\isometryconst := \natoms \exp ( - \frac{1}{4,741 \coherence^2 \sparsity} )  \leq \frac{1}{48 (\frameupbound + 1)}$. Take a \ac{jl} embedding based on an orthogonal transform in \Cref{def:fast_jl_embedding}, and choose an embedding distortion $\jldistorterm < \csum_{\sparsity}\sqrt{\sparsity}/4$ and a target error $\targeterror \geq 8 \epsopt$, where
\begin{equation}
\epsopt := \frac{13\natoms^2 \sqrt{\frameupbound + 1}}{\snrinputsig \statinputone}  \exp \left(\frac{-\big(\beta_S - \frac{2\jldistorterm}{\sqrt{S}}\big)^2} { 72 \max\{ (\mu +\jldistorterm)^2, \nsigma^2 +\jldistorterm^2\ddim \nsigma^2 \}}\right).
\label{eq:besterror_ictkm}	
\end{equation}
%Additionally, assume that $\targeterror \leq 1 - \statinputotwo + \ddim \nsigma^2$. 
If the initial dictionary $\pdico$ satisfies
\begin{equation}
d(\pdico, \dico) \leq\min \left\{ \frac{1}{32\sqrt{S}}, \frac{\Delta_S - 2\jldistorterm}{9\sqrt{B}\left(\frac{1}{4} + \sqrt{\log\left(\frac{1392 K^2 (B+1)}{\snrinputsig \statinputone(\Delta_S - 2\jldistorterm)}\right)}\right)} \right\},
\label{eq:converg_radius}	
\end{equation} 
%with 
%$$
%\jltermradius := \Pert_{\sparsity} \sqrt{ \frac{ 1 - \frac{4 \jldistorterm}{\csum_{\sparsity} \sqrt{\sparsity}} }{ \frameupbound + \jldistorterm (2 +\jldistorterm)(\csum_{\sparsity}/\Pert_{\sparsity})^2 } },
%$$ 
and the embedding dimension $\jlembeddingdim$ is at least of order 
\begin{equation}
\jlembeddingdim \geq O\left( \jldistorterm^{-2} \cdot \log^2(\jldistorterm^{-1}) \cdot \log(\natoms/\jlprobterm) \cdot \log^2 ( \jldistorterm^{-1} \log(\natoms/\jlprobterm) ) \cdot \log\ddim \right)
\label{eq:embedding_dimension}
\end{equation}
then after $L = 5 \ceil{\log(\targeterror^{-1})}$ iterations of \ac{ictkm}, each using a new \ac{jl} embedding $\jlmatrix$ and a new training data set $\trainmatrix$, the output dictionary $\ppdico$ satisfies $d(\ppdico,\dico) \leq \targeterror$ except with probability 
\begin{equation}
 \jlprobterm L+ 6L\natoms \exp \left( \frac{- \snrinputsig^2 \statinputone^2 \nsig \targeterror^2}{576 \natoms \max \{ \sparsity, \frameupbound + 1 \} ( \targeterror + 1 - \statinputotwo + \ddim \nsigma^2) } \right).
\label{eq:ictkm_failure_prob_final_estimate}
\end{equation}
\end{theorem}
%%%%%%%
%%%%%%%Proof%%%%%%%%
\begin{proof} \normalfont Rewriting the atom update rule in \Cref{alg:ictkm}, we have 
\begin{spreadlines}{-.3\baselineskip}
\begin{align}
\ppatom_k = & \frac{1}{\nsig} \sum_n{\left[ R^{ct} (\pdico, \trainsignal_n, k) - R^{o}(\pdico, \trainsignal_n, k) \right]} + \frac{1}{\nsig} \sum_n{\left[ R^{o}(\pdico, \trainsignal_n, k) - R^{o}(\dico, \trainsignal_n, k) \right]} \nonumber \\
& \hspace{5em} + \frac{1}{\nsig} \sum_n{ \signsequence_n(k) \cdot \indicator_{\oraclesuppn}(k) \left[ \trainsignal_n - \projop (\dico_{\oraclesuppn}) \trainsignal_n \right]} + \frac{1}{\nsig} \sum_n{\langle \trainsignal_n, \atom_k \rangle \cdot \signsequence_n(k) \cdot \indicator_{\oraclesuppn}(k)} \atom_k, \label{eq:dictionary_update_ictkm_rewritten}	
\end{align}
\end{spreadlines}
where $R^{ct} (\pdico, \trainsignal_n, k)$ is the compressed-thresholding residual based on $\pdico$ defined by
$$
R^{ct} (\pdico, \trainsignal_n, k) := \signop \left( \langle \patom_k, \trainsignal_n \rangle \right) \cdot \indicator_{\compthreshpdicosuppn}(k) \left[ \trainsignal_n - \projop ( \pdico_{\compthreshpdicosuppn} ) \trainsignal_n + \projop ( \patom_k) \trainsignal_n \right],
%\label{eq:compressed_thresholding_residual}	
$$
and $R^{o}(\cdot, \trainsignal_n, k)$ is the oracle residual based on $\pdico$ (or $\dico$) 
$$
R^{o}(\pdico, \trainsignal_n, k) := \signsequence_n(k) \cdot \indicator_{\oraclesuppn}(k) \left[ \trainsignal_n - \projop ( \pdico_{\oraclesuppn}) \trainsignal_n + \projop (\patom_k) \trainsignal_n \right].
%\label{eq:oracle_residual_pdico}	
$$
Applying the triangle inequality to \labelcref{eq:dictionary_update_ictkm_rewritten}, we have $\left\| \ppatom_k - s_k \atom_k \right\|_2 \leq t$, where 
$$
s_k := \frac{1}{\nsig} \sum_n{\langle \trainsignal_n, \atom_k \rangle \cdot \signsequence_n(k) \cdot \indicator_{\oraclesuppn}(k)},
$$ 
and $t := t_1 + t_2 + t_3$ with 
\begin{spreadlines}{-.1\baselineskip}
\begin{align*}
t_1 &: = \max_k \frac{1}{\nsig} \big\| \sum_n \left[ R^{ct}(\pdico, \trainsignal_n,k) - R^{o}(\pdico,\trainsignal_n,k) \right] \big\|_2, \\
t_2 &= \max_k \frac{1}{\nsig} \big\| \sum_n \left[ R^{o}(\pdico,\trainsignal_n,k) - R^{o}(\dico,\trainsignal_n,k) \right] \big\|_2, \\
t_3 &= \max_k \frac{1}{\nsig} \big\| \sum_n{\signsequence_n(k) \cdot \indicator_{\oraclesuppn}(k) \cdot \left[ \trainsignal_n - \projop(\dico_{\oraclesuppn}) \trainsignal_n \right]}\big\|_2.
\end{align*}
\end{spreadlines}
Now from Lemma B.10 in \cite{sc15}, we know that the inequality $\left\| \ppatom_k - s_k \atom_k \right\|_2 \leq t$ implies that 
\begin{equation}
 \left\| \frac{\patomn_k}{\left\| \patomn_k \right\|_2} - \atom_k \right\|^2_2 \leq 2 -2 \sqrt{1 - t^2/s_k^2}.
 \end{equation}
As $ 2 -2 \sqrt{1 - x}\leq (1+x)x$ for $0< x <1$ we can define $s: = \min_k s_k$ and assuming that $0<t <s$ get
 \begin{equation}
d(\pdico, \dico) = \max_k \left\| \frac{\patomn_k}{\left\| \patomn_k \right\|_2} - \atom_k \right\|_2 \leq \frac{t}{s} \sqrt{1+ \frac{t^2}{s^2}}.
\label{eq:lemma_B10}
 \end{equation}
This means that to ensure the error is not increased in one iteration, a tight control of $t/s$ needs to be established with high probability. \\
We proceed by controlling $t$ and $s$ using concentration of measure results. Starting with the first term $t_1$, from \Cref{lem:compressed_thresholding_residual_estimate} to be found in Appendix~\ref{appendix_technicalities} we have for $K\geq 55$ and $v_1,\paramprobestimate, \jlprobterm > 0$ the following estimate
\begin{spreadlines}{-.3\baselineskip}
\begin{align}
& \P \left( t_1 \geq \frac{\snrinputsig \statinputone}{\natoms} \left( \epsopt +  \paramprobestimate \eps + v_1 \right) \right) \leq \jlprobterm + \exp \left( \frac{-v_1^2 \snrinputsig \statinputone \nsig}{4 \natoms \sqrt{\frameupbound + 1} \left( \epsopt + \paramprobestimate \eps + v_1/2 \right)} \right),
\label{eq:prob_estimate_tk1}
\end{align}
\end{spreadlines}
as long as the embedding dimension satisfies \labelcref{eq:embedding_dimension} and
\begin{equation}
d(\pdico, \dico) \leq \frac{\Delta_S - 2\jldistorterm}{9\sqrt{B}\left(\frac{1}{4} + \sqrt{\log\left(\frac{58 K^2 (B+1)}{\tau \snrinputsig \statinputone(\Delta_S - 2\jldistorterm)}\right)}\right)}
\label{eq:condition_eps_convergence_radius_more}	
\end{equation}
To control $t_2$, $t_3$, and $s$, we use Lemmata B.6-B.8 in \cite{sc15}.\\
For $\sparsity \leq \min\left\{ \frac{\natoms}{98\frameupbound}, \frac{1}{98 \nsigma^2} \right\}$, $\eps \leq \frac{1}{32 \sqrt{\sparsity}}$, $\eps_\isometryconst \leq \frac{1}{48 (\frameupbound + 1)}$, and $v_2 > 0$, the second term $t_2$ can be estimated as 
\begin{align}
& \P \left(\! t_2 \!\geq\! \frac{\snrinputsig \statinputone}{\natoms} (0.381 \eps + v_2) \! \right)  \!\leq\!  K\exp  \left(   \frac{-v_2 \snrinputsig^2 \statinputone^2 \nsig}{8 \natoms  \max \left\{ \sparsity, \frameupbound \!+\! 1 \right\}} \min  \left\{ \frac{v_2}{5\eps^2 + \eps_\isometryconst ( 1 - \statinputotwo + \ddim \nsigma^2 )/32}, \frac{1}{3}  \right\} \!+\! \frac{1}{4} \right)\label{eq:probability_estimate_lemma_b8}	
\end{align}
and for $0 \leq v_3 \leq 1 - \statinputotwo + \ddim \nsigma^2$, the third term $t_3$ is estimated as 
\begin{align}
&\P \left( t_3 \geq \frac{\snrinputsig \statinputone}{\natoms} v_3 \right)
\leq K \exp \left( \frac{-v_3^2 \snrinputsig^2 \statinputone^2 \nsig}{ 8 \natoms \max \left\{ \sparsity, \frameupbound + 1 \right\} (1 - \statinputotwo + \ddim \nsigma^2  ) } + \frac{1}{4} \right).
\label{eq:probability_estimate_lemma_b7}	
\end{align}
Finally, for $v_0 > 0$ we estimate the last term $s_k$ as follows
\begin{align}
& \P \left( s \leq (1-v_0) \frac{\snrinputsig \statinputone}{\natoms} \right) 
\leq K \exp \left( \frac{- \nsig v_0^2 \snrinputsig^2 \statinputone^2 }{ 2 \natoms \left( 1 + \frac{\sparsity \frameupbound}{\natoms} + \sparsity \nsigma^2 + v_0 \snrinputsig \statinputone \sqrt{\frameupbound + 1}/3 \right) } \right).
\label{eq:probability_estimate_lemma_b6}	
\end{align}
Collecting the concentration of measure results, we have
\begin{equation}
t/s \leq \frac{\epsopt + \paramprobestimate \eps + v_1 + 0.381 \eps + v_2 + v_3}{1-v_0},
\label{eq:bound_tk_sk}
\end{equation}
except with probability given by the sum of the right-hand sides of \labelcref{eq:prob_estimate_tk1,eq:probability_estimate_lemma_b8,eq:probability_estimate_lemma_b7,eq:probability_estimate_lemma_b6}. We can see from \labelcref{eq:bound_tk_sk} that to have the error decreased in one iteration with high probability, it suffices to choose the constants $v_0$ to $v_3$, and $\paramprobestimate$. Assuming that the target error $\targeterror \geq 8 \epsopt$ and setting $v_0 = 1/50$, $v_1 = v_2 = \max \left\{ \targeterror, \eps \right\}/24$, $v_3 = \targeterror/8$, and $\paramprobestimate = 1/24$ we arrive at $t/s \leq 0.78 \cdot \max{\left\{ \targeterror, \eps \right\}}$. Since we assume $\targeterror, \eps \leq 1/(32 \sqrt{S})$ by \labelcref{eq:lemma_B10} we have $d(\ppdico, \dico) \leq 0.8 \cdot \max{\left\{ \targeterror, \eps \right\}}$,
%From \labelcref{eq:lemma_b_10}, we see that $d(\ppdico, \dico)$ can be made less than $1$ whenever $t_k/s_k \leq \sqrt{3}/2$. Therefore, it suffices to choose the constants $\paramprobestimate$, $\jlprobtermnew$, $v_0$, $v_1$, $v_2$, and $v_3$ in \labelcref{eq:bound_tk_sk,prob_t_k_sk} such that the error is decreased with high probability. 
except with probability 
\begin{align}
\zeta :=& \jlprobterm + \exp \left( \frac{- \snrinputsig \statinputone \nsig \max \left\{ \targeterror, \eps \right\} }{ 432 \natoms \sqrt{\frameupbound + 1} } \right) + 2 \natoms \exp \left( \frac{-\snrinputsig^2 \statinputone^2 \nsig \max \left\{ \targeterror, \eps \right\}^2}{576 \natoms \max \left\{ \sparsity, \frameupbound + 1 \right\} \left( \max \left\{ \targeterror, \eps \right\} + 1 - \statinputotwo + \ddim \nsigma^2 \right)} \right)\nonumber \\
& \hspace{1.3em} + 2 \natoms \exp \left( \frac{-\snrinputsig^2 \statinputone^2 \nsig \targeterror^2}{512 \natoms \max \left\{ \sparsity, \frameupbound + 1 \right\} \left( 1 - \statinputotwo + \ddim \nsigma^2 \right)} \right) +\natoms \exp \left( \frac{-\snrinputsig^2 \statinputone^2 \nsig}{\natoms \left( 5103 + 34 \snrinputsig \statinputone \sqrt{ \frameupbound + 1 } \right)} \right).
\label{eq:failure_probability_one_iteration_ictkm}
\end{align}
Further substituting the value for the constant $\paramprobestimate$ in \labelcref{eq:condition_eps_convergence_radius_more}, we arrive at the estimate for the convergence radius in \labelcref{eq:converg_radius}.\\

Lastly, we need to ensure that the target error is actually reached. Since for each iteration we obtain $\jlmatrix$ by redrawing $\ripmatrix$ and $\jlembeddingdim$ rows from an orthogonal matrix, and $\trainmatrix$ by redrawing $\nsig$ training signals $\trainsignal_n$, it follows that after $L$ iterations the error will satisfy $d(\ppdico, \dico) \leq \max \left\{ \targeterror, 0.8^L \right\}$ except with probability at most $L \cdot \zeta$. Thus, to reach the target error we set $L=5 \lceil \log(\targeterror^{-1}) \rceil$. Using the fact that the exponential terms in \labelcref{eq:failure_probability_one_iteration_ictkm} are dominated by the second exponential term when replacing $\max \left\{ \targeterror, \eps \right\}$ with $\targeterror$, we can bound the failure probability $\zeta$ by 
$$
 \jlprobterm + 6 \natoms \exp \left( \frac{- \snrinputsig^2 \statinputone^2 \nsig \targeterror^2}{576 \natoms \max \left\{ \sparsity, \frameupbound + 1 \right\} \left( \targeterror + 1 - \statinputotwo + \ddim \nsigma^2 \right)} \right),
$$ 
which leads to the final estimate in \labelcref{eq:ictkm_failure_prob_final_estimate}.
\end{proof}	
With the general Theorem in hand it is straightforward to arrive at the featured \Cref{th:ictkm_convergence_o_notation}.

\begin{proof}[Proof of \Cref{th:ictkm_convergence_o_notation}]
The proof is a direct application of \Cref{th:ictkm_convergence} to dictionaries with $\frameupbound = O(\natoms/\ddim)=O(1)$ and exactly $\sparsity$-sparse signals with $\coeffsequence(k)= 1/\sqrt{\sparsity}$ for $k \leq \sparsity$ and $\coeffsequence(k) = 0$ for $k > \sparsity$ and Gaussian noise with variance $\nsigma^2$. In this case $\csum_{\sparsity} = 1/\sqrt{\sparsity}$, $\Pert_{\sparsity} = 1$, $\statinputone = \sqrt{S}$, $\statinputotwo =1$ and $\snrinputsig = O(1)$. Therefore we get
\begin{equation*}
\epsopt %= \frac{8\natoms^2 \sqrt{\frameupbound + 1}}{\snrinputsig \sqrt{S}} \exp \left( \frac{-\left(1 - 2 \jldistorterm \right)^2}{98 S \left( \coherence + \jldistorterm \right)^2  } \right) 
= O\left(\natoms^2 \exp \left( \frac{-\left(1 - 2 \jldistorterm \right)^2}{98 \sparsity \max\{ (\coherence + \jldistorterm )^2, \nsigma^2 + \nsigma^2 \ddim \jldistorterm^2\}  } \right) \right)
\end{equation*}
and to have a target error $\targeterror = \natoms^{2-\ell}$ satisfying $\targeterror \geq \epsopt $, we get the following bound on the admissible sparsity level and embedding distortion 
\begin{align}
\max\{ (\coherence + \jldistorterm)^2, \nsigma^2 (1+ \ddim \jldistorterm^2) \} \leq O\left(\frac{1}{ \sparsity \log(\natoms^2/\targeterror) }\right), \notag
\end{align}
which simplifies to \labelcref{eq:admissible_sparsity_ictkm_convergence}.
Further, the above bound and \labelcref{eq:admissible_sparsity_ictkm_convergence} ensure the condition on $\eps_\isometryconst$ and imply that $\jldistorterm \leq  O(1/\sqrt{\sparsity \log{\natoms}})$, which in turn means that the condition on the initial dictionary becomes \begin{align}
d(\pdico,\dico) \leq O(1/\sqrt{\max\{ \sparsity,\log{\natoms}\}}). \notag
\end{align}
To reach the target error $\targeterror$ we need $L=O(\log(1/\targeterror))$ iterations. Choosing an overall failure probability of order $O(L\jlprobterm)$, the bound in \labelcref{eq:ictkm_failure_prob_final_estimate} leads to the requirements on the embedding dimension $m$ and the sample size $N$. So we need an embedding dimension
\begin{equation}
\jlembeddingdim \geq O\left( \jldistorterm^{-2} \cdot \log^2(\jldistorterm^{-1}) \cdot \log(\natoms/\jlprobterm) \cdot \left[\log ( \jldistorterm^{-1})+ \log\log(\natoms/\jlprobterm) \right]^2 \cdot \log\ddim \right) \notag,
\end{equation}
which, using the bound on $\jldistorterm$ together with
\begin{align*}
O\left(\log^2(\jldistorterm^{-1}) \log(\natoms/\jlprobterm) \right) &\geq O\left(\log^2(S\log \natoms) \log(\natoms/\jlprobterm) \right) \geq O\left(\log^2(\log \natoms) \log(\natoms/\jlprobterm) \right) \geq O\left(\log^2\log(\natoms/\jlprobterm)\log \ddim \right),
\end{align*}
simplifies to \labelcref{eq:jl_embedding_o_theorem_complete}, and a sample size 
\begin{align}
\nsig \geq 576 \, C_r^{-2}\, \natoms \,\log(6\natoms/\jlprobterm)\, \frac{\targeterror + \ddim \nsigma^2}{ \targeterror^2}, 
\end{align}
which reduces to \labelcref{eq:sample_size}.
\end{proof}

%%%%%%%%%%%%%%%%%%%%%%%%%%%%%%%%%%%%%%%%%%%%%%%%%%%%
%%%%%%%%%%%%%%%%%%%%%APPENDIX: TECHNICAL%%%%%%%%%%%%%
%%%%%%%%%%%%%%%%%%%%%%%%%%%%%%%%%%%%%%%%%%%%%%%%%%%%

\section{Technicalities}\label{appendix_technicalities}

\begin{lemma}[Compressed-Thresholding/Oracle-Residuals Expected Difference] \label{lem:compressed_thresholding_residual_estimate} Assume that $\trainsignal_n$ follows the model in \labelcref{eq:training_signal_with_model} with coefficients that are $\sparsity$-sparse, and that have an absolute gap $\csum_{\sparsity}$ and a relative gap $\Pert_{\sparsity}$. Further assume that $S\leq K/ (98B)$, and that $\jlmatrix$ is a \ac{jl} embedding based on an orthogonal transform as in \Cref{def:fast_jl_embedding}. We have for $\paramprobestimate, v, \jlprobterm > 0$
\begin{align}
& \P \Biggl(\exists k : \frac{1}{\nsig} \Big\| \sum_n \left[ R^{ct}(\pdico, \trainsignal_n,k) \!-\! R^{o}(\pdico,\trainsignal_n,k) \right] \Big\|_2  \geq \frac{\snrinputsig \statinputone}{\natoms} \left( \epsopt \! +  \! \paramprobestimate \eps  \!+ \! v \right) \Biggr) \nonumber \\
& \hspace{22em} \leq \jlprobterm + \exp \left( \frac{-v^2 \snrinputsig \statinputone \nsig}{4 \natoms \sqrt{\frameupbound + 1} \left( \epsopt + \paramprobestimate \eps + v/2 \right)} \right)  \label{eq:prob_estimate_thresholding},
\end{align}
whenever
\begin{equation}
d(\pdico, \dico) \leq \frac{\Delta_S - 2\jldistorterm}{9\sqrt{B}\left(\frac{1}{4} + \sqrt{\log\left(\frac{58 K^2 (B+1)}{\tau \snrinputsig \statinputone(\Delta_S - 2\jldistorterm)}\right)}\right)}
\label{eq:condition_eps_convergence_radius}
\end{equation}
and
\begin{equation}
\jlembeddingdim \geq O\left( \jldistorterm^{-2} \cdot \log^2(\jldistorterm^{-1}) \cdot \log(\natoms/\jlprobterm) \cdot \log^2 ( \jldistorterm^{-1} \log(\natoms/\jlprobterm) ) \cdot \log\ddim \right).
\label{eq:embedding_dimension_conv_proof}
\end{equation}
\end{lemma}

\begin{proof} Our first observation is that when compressed thresholding succeeds, meaning it recovers the oracle support and oracle signs, then the two residuals will coincide for all $k$. So defining the event of thresholding failing as
\begin{align}
\thresholdbad: = \left\{ \trainsignal : \compthreshpdicosupp \neq \oraclesupp \right\}  \cup \{ \trainsignal: \signop ( \pdico_{\oraclesupp}^{*} \trainsignal) \neq \signsequence( \oraclesupp ) \}
\label{eq:event_thresholding_simplified}	
\end{align}
and using the orthogonality of the projections $\left[ \I_\ddim - \projop (\pdico_{\oraclesuppn}) + \projop(\patom_k) \right]$ and the bound $\left\| \trainsignal_n \right\|_2 \leq \sqrt{\frameupbound + 1}$, we have for all $k$
\begin{align}
\Big\| \sum_n [ R^{ct}(\pdico, \trainsignal_n,k) - R^{o}(\pdico,\trainsignal_n,k) ] \Big\|_2 \leq 2 \sqrt{\frameupbound + 1} \cdot \# \left\{ n : \trainsignal_n \in  \thresholdbad \right\}.
\end{align}
This means that we can estimate
\begin{align}
&\P\left( \exists k : \frac{1}{\nsig} \Big\| \sum_n [ R^{ct}(\pdico, \trainsignal_n,k) - R^{o}(\pdico,\trainsignal_n,k) ] \Big\|_2  > t\right)
 \leq \P\left( \# \left\{ n :\trainsignal_n \in  \thresholdbad \right\} > \frac{tN}{2 \sqrt{\frameupbound + 1}} \right),
\label{eq:residual_diff_bound}	
\end{align}
and that it suffices to bound the number of signals for which thresholding fails with high probability.
Clearly this number depends on the properties of the chosen embedding and how well it preserves the inner products between the atoms of the current dictionary and the atoms of the generating dictionary making up the signals. Recalling the decomposition of the perturbed atoms $\patom_k = \alpha_k \atom_k + \omega_k \unitatom_k$, we define for a given distortion $\jldistorterm$
\begin{align}
\jlgood : = \big\{ \jlmatrix \:|\: \forall j,k : |\ip{ \jlmatrix \atom_k}{ \jlmatrix \atom_j }| \leq |\ip{\atom_k}{\atom_j }| + \jldistorterm& \:\wedge \: \|\jlmatrix \atom_k\|_2^2 \geq (1-\delta) \notag \\
\wedge |\ip{ \jlmatrix \atom_k}{ \jlmatrix \unitatom_j }| \leq |\ip{\atom_k}{\unitatom_j }| + \jldistorterm &\:\wedge \: \|\jlmatrix \unitatom_k\|_2^2 \leq (1+\delta) \big\}.
 \label{def:jlgood}
\end{align}
We then have 
\begin{align}
 \P\left( \# \left\{ n : \trainsignal_n \!\in\!  \thresholdbad \right\} \geq t \right)
\leq \P\left(\# \left\{ n : \trainsignal_n  \!\in\! \thresholdbad \right\} \geq t \big| \jlmatrix  \!\in\!  \jlgood \right) + \P(\jlmatrix \!\notin\!  \jlgood).
\label{eq:condition_jlgood}
\end{align}
%%%%%%
We first relate the probability of drawing an embedding with too high distortion to the embedding dimension $m$.
We define the set of $p=4 \natoms^2$ vectors $\mathcal{X} = \mathcal{X}_{\atom^{+}} \cup \mathcal{X}_{\atom^{-}} \cup \mathcal{X}_{\unitatom^{+}} \cup \mathcal{X}_{\unitatom^{-}} \cup \mathcal{X}_{\unitatom} $, where 
\begin{align*}
 \mathcal{X}_{\atom^{+}} = \left\{\atom_k + \atom_j : k, j \in \K  \right\}, &\quad \mathcal{X}_{\atom^{-}} = \left\{\atom_k - \atom_j : k, j \in \K, k\neq j \right\}, \\
 \mathcal{X}_{\unitatom^{+}} = \left\{ \atom_k + \unitatom_j : k, j \in \K  \right\},\: &\quad  \mathcal{X}_{\unitatom^{-}} = \left\{\atom_k -\unitatom_j : k,j \in \K \right\}\\
 \mbox{and}\:& \quad \mathcal{X}_{\unitatom} = \{\unitatom_k : k \in \K\}.
\end{align*}
From \Cref{def:fast_jl_embedding} we know that for a DCT or Fourier based \ac{jl} embedding with embedding dimension $\jlembeddingdim$ satisfying
\begin{equation}
\jlembeddingdim \geq O\left( \jldistorterm^{-2} \cdot \log^2(\jldistorterm^{-1}) \cdot \log(p/\jlprobterm) \cdot \log^2 ( \jldistorterm^{-1} \log(p/\jlprobterm) ) \cdot \log\ddim \right),	
\label{eq:embedding_K2points}
\end{equation}
which is ensured by \labelcref{eq:embedding_dimension_conv_proof},
we have with probability at least $(1-\jlprobterm)$, that for all $\bm{x}\in \mathcal X$
\begin{align}
\| \jlmatrix \bm{x}\|_2^2 \lessgtr (1 \pm \jldistorterm)  \|\bm{x}\|_2^2.
\end{align}
This means that the squared norms of $\unitatom_k$ and $\atom_k$ are preserved up to $\jldistorterm$, while setting $\bm{u}_j = \atom_j$ and $\bm{u}_j = \unitatom_j$ in the polarization identities further yield
\begin{align*}
 |\ip{ \jlmatrix \atom_k}{ \jlmatrix \bm{u}_j  }| &= \frac{1}{4} \Big| \|\jlmatrix(\atom_k + \bm{u}_j) \|_2^2 -  \|\jlmatrix(\atom_k - \bm{u}_j  )\|_2^2\Big|\notag \\
 &\leq \frac{1}{4} \big[ ( 1 + \jldistorterm) (2 + 2|\ip{\atom_k }{\bm{u}_j }|) -  ( 1 - \jldistorterm) (2 - 2|\ip{\atom_k }{\bm{u}_j }|)\big]\\
 &\leq |\ip{\atom_k }{\bm{u}_j }| + \delta.
\end{align*}
Looking back at the definition of $\jlgood$ in \labelcref{def:jlgood}, we then see that $\P (\jlmatrix \notin \jlgood) \leq \jlprobterm$ as long as the embedding dimension $m$ satisfies \labelcref{eq:embedding_dimension_conv_proof}.\\
%
%%%%%%%%%
Next we assume that we have drawn $\jlmatrix \in \jlgood$ satisfying \labelcref{eq:embedding_dimension_conv_proof} and estimate the number of signals for which the residuals do not coincide for all $k$ - now only depending on the draw of signals.
%Our first observation is that for a given index $k$ the number of signals for which the residuals do not coincide, is bounded by the number of signals for which compressed thresholding does not recover the oracle support, $\compthreshpdicosupp \neq \oraclesupp$, or that the empirical sign patterns using $\pdico$ are different from the oracle sign patterns, $\signop( \langle \jlmatrix \patom_k, \jlmatrix \trainsignal \rangle ) \neq \signsequence(k)$ for $k \in \oraclesupp$. 
%\begin{align}
%&\# \left\{ n : R^{ct}(\pdico,\trainsignal_n,k) \neq R^{o}(\pdico,\trainsignal_n,k) \right\}\\
%& \leq \# \left\{ n : \compthreshpdicosuppn \neq \oraclesuppn \: \mbox{or} \: \exists k \in \oraclesuppn \:\mbox{with} \: \signop( \langle \jlmatrix \patom_k, \jlmatrix \trainsignal_n \rangle ) \neq \signsequence_n(k)  \right\}
%\end{align}
%
%Defining the event $\thresholdbad: = \left\{ \trainsignal :  R^{ct}(\pdico,\trainsignal,k) \neq R^{o}(\pdico,\trainsignal,k) \right\}$,
%we can rewrite 
First we rewrite 
\begin{align}
\# \left\{ n : \trainsignal_n \in \thresholdbad \right\} = \sum_n \indicator_{\thresholdbad}(\trainsignal_n).
\end{align}
Since the signals $\trainsignal_n$ are independent, so are the indicator functions $\indicator_{\thresholdbad}(\trainsignal_n)$ and applying Bernstein's inequality to their centered versions yields 
\begin{align}
\P \left(\sum_n  \indicator_{\thresholdbad}(\trainsignal_n) \geq N\P(\thresholdbad)+ Nv \right)  \leq \exp \left( \frac{ -v^2 \nsig }{ 2 \P (\thresholdbad ) + v  } \right), \label{eq:bernstein_indicator}
\end{align}
meaning we still have to bound the probability $\P (\thresholdbad )$. 
To ensure that compressed thresholding recovers the oracle support, i.e., $\compthreshpdicosupp = \oraclesupp$ we need to have
\begin{equation}
\min_{k \in \oraclesupp}{ \left| \langle \jlmatrix \patom_k, \jlmatrix \trainsignal \rangle \right| } > \max_{k \notin \oraclesupp}{ \left| \langle \jlmatrix \patom_k, \jlmatrix \trainsignal \rangle \right| }.
\label{eq:sucessfull_thresholding_condition_psi}
\end{equation}
Expanding the inner products using the decomposition $\patom_k = \alpha_k \atom_k + \omega_k \unitatom_k$
we get
\begin{spreadlines}{-.3\baselineskip}
\begin{align}
\ip{ \jlmatrix \patom_k}{ \jlmatrix \trainsignal } \!=\! \frac{1}{\sqrt{1 \!+\! \left\| \noise \right\|_2^2}} & \Biggl[  \dicodecompfstvectorcoord_k \left\| \jlmatrix \atom_k \right\|_2^2 \signsequence(k) \coeffsequence(\permmapping(k)) \Biggl. +  \dicodecompfstvectorcoord_k \sum_{j \neq k} \signsequence(j) \coeffsequence(\permmapping(j)) \langle \jlmatrix \atom_k, \jlmatrix \atom_j \rangle \bigg. \nonumber \\ 
& \hspace{6em} \Biggl. + \dicodecompsndvectorcoord_k \sum_j{ \signsequence(j) \coeffsequence(\permmapping(j)) \langle \jlmatrix \unitatom_k, \jlmatrix \atom_j \rangle } \Biggl. + \dicodecompfstvectorcoord_k \langle \jlmatrix \atom_k, \jlmatrix \noise \rangle  + \dicodecompsndvectorcoord_k \langle \jlmatrix \unitatom_k, \jlmatrix \noise \rangle \Biggr],
\label{eq:inner_products_expanded}
\end{align}
\end{spreadlines}
from which we obtain the following lower (upper) bounds for $k \in \oraclesupp$ ($k \notin \oraclesupp$)
\begin{spreadlines}{-.3\baselineskip}
\begin{align}
\left| \langle \jlmatrix \patom_k, \jlmatrix \trainsignal \rangle \right| \lessgtr \frac{1}{\sqrt{1 \! + \! \left\| \noise \right\|_2^2}} &\Bigg[ \pm \big| \sum_{j \neq k} \signsequence(j) \coeffsequence(\permmapping(j)) \langle \jlmatrix \atom_k, \jlmatrix \atom_j \rangle \big|  
 \Biggl. \pm \dicodecompsndvectorcoord_k \big|\sum_j{ \signsequence(j) \coeffsequence(\permmapping(j)) \langle \jlmatrix \unitatom_k, \jlmatrix \atom_j \rangle } \big| \bigg. \nonumber \\
& \hspace{6em} \Biggl. \pm \big| \ip{\jlmatrix \atom_k}{ \jlmatrix \noise}  \big| \pm  \dicodecompsndvectorcoord_k\big| \ip{ \jlmatrix \unitatom_k}{\jlmatrix \noise} \big| \Biggl. + \alpha_k \left\| \jlmatrix \atom_k \right\|_2^2  \begin{cases} \coeffsequence(\sparsity\!+\!1), & k \!\notin\!\oraclesupp\\
 \coeffsequence(\sparsity),  & k \!\in\! \oraclesupp  \end{cases} \Biggr].
 \notag%\label{eq:thresholding_inner_product_patom_bound}
\end{align}	
\end{spreadlines}
Substituting the above into \labelcref{eq:sucessfull_thresholding_condition_psi}, and using the norm-preserving property of $\jlmatrix$, meaning $1-\delta \leq  \| \jlmatrix \atom_k \|^2_2 \leq 1+\delta$ for all $k$ we arrive at the following sufficient condition for having $\compthreshpdicosupp = \oraclesupp$, 
\begin{spreadlines}{-.3\baselineskip}
\begin{align}
 &\big| \sum_{j \neq k}{ \signsequence(j) \coeffsequence(\permmapping(j)) \langle \jlmatrix \atom_k, \jlmatrix \atom_j \rangle } \big| + \big| \langle \jlmatrix \atom_k, \jlmatrix \noise \rangle \big| + \dicodecompsndvectorcoord_k \big| \sum_j{\signsequence(j) \coeffsequence(\permmapping(j)) \langle \jlmatrix \unitatom_k, \jlmatrix \atom_j \rangle } \big| + \dicodecompsndvectorcoord_k \big| \langle \jlmatrix \unitatom_k, \jlmatrix \noise \rangle \big| \notag\\
& \hspace{22em} \leq \frac{1}{2} \big[\coeffsequence(\sparsity) (1-\delta)\left( 1-\tfrac{\eps^2}{2} \right) - \coeffsequence(\sparsity+1)(1+\delta)\big] \quad \forall k .
%\label{eq:condition_sucsessfull_thresholding_psi_noisy}	
\end{align}
\end{spreadlines}
Note that the condition above already implies that the embedded inner products have the correct sign. However, we are using the unembedded inner products to determine the sign in the algorithm, so we still have to analyse the second event $\{ \trainsignal : \signop ( \pdico_{\oraclesupp}^{*} \trainsignal) \neq \signsequence( \oraclesupp) \}$. From \labelcref{eq:inner_products_expanded} with $\jlmatrix =  \I $ we can see that it suffices to ensure that for all $k \in \oraclesupp$.
\begin{align}
\coeffsequence(\sparsity) &> 
\big| \sum_{j \neq k} \signsequence(j) \coeffsequence(\permmapping(j))\ip{\atom_k}{ \atom_j}\big|  + |\ip{\atom_k}{\noise}| + \frac{ \dicodecompsndvectorcoord_k}{\dicodecompfstvectorcoord_k} \big| \sum_j \signsequence(j) \coeffsequence(\permmapping(j))\ip{\unitatom_k}{\atom_j}\big|  + \frac{ \dicodecompsndvectorcoord_k}{\dicodecompfstvectorcoord_k} |\ip{ \unitatom_k}{ \noise }|,
%\label{eq:inner_products_expanded}
\end{align}
Thus, the event of thresholding failing is contained in the events
$
\thresholdbad \subseteq \threshsetatom \cup \threshsetunitatom \cup \signsetatom \cup \signsetunitatom
$, where  
\begin{align}
\allowdisplaybreaks
& \threshsetatom  := \Big\{ \trainsignal \:\big | \:\exists k  : \big| \sum_{j \neq k}\signsequence(j) \coeffsequence(\permmapping(j)) \ip{ \jlmatrix \atom_k}{ \jlmatrix \atom_j } \big| \geq u_{1} \vee \big| \ip{ \jlmatrix \atom_k }{\jlmatrix \noise } \big| \geq u_{2}   \Big\}  \notag\\
& \threshsetunitatom := \Big\{ \trainsignal \:\big | \:\exists k  : \dicodecompsndvectorcoord_k \big| \sum_j\signsequence(j) \coeffsequence(\permmapping(j)) \ip{\jlmatrix \unitatom_k }{\jlmatrix \atom_j} \big|  \geq u_{3} \vee  \dicodecompsndvectorcoord_k \big| \ip{ \jlmatrix \unitatom_k }{ \jlmatrix  \noise} \big| \geq u_4 \Big\} \notag\\&\mbox{with} \quad 2\left( u_{1} + u_{2} + u_{3} + u_{4} \right) \leq  \coeffsequence(\sparsity) (1\!-\!\delta)\left( 1-\tfrac{\eps^2}{2} \right) - \coeffsequence(\sparsity\!+\!1)(1+\delta),
\end{align}
and in analogy for $v_{1} + v_{2} + v_{3} + v_{4}  \leq  \coeffsequence(\sparsity)$
\begin{align}
\allowdisplaybreaks
&\signsetatom  := \Big\{ \trainsignal \:\big | \:\exists k \in \oraclesupp  : \big| \sum_{j \neq k}\signsequence(j) \coeffsequence(\permmapping(j)) \ip{\atom_k}{\atom_j } \big| \geq v_{1} \vee \big| \ip{\atom_k }{\noise } \big| \geq v_{2}   \Big\}  \notag\\
&\signsetunitatom := \Big\{ \trainsignal \:\big | \:\exists k \in \oraclesupp : \frac{ \dicodecompsndvectorcoord_k}{\dicodecompfstvectorcoord_k}\big| \sum_j\signsequence(j) \coeffsequence(\permmapping(j)) \ip{\unitatom_k }{\atom_j} \big|  \geq v_{3} \vee \frac{ \dicodecompsndvectorcoord_k}{\dicodecompfstvectorcoord_k} \big| \ip{ \unitatom_k }{  \noise} \big| \geq v_4 \Big\} \notag
%&\hspace{3cm} \mbox{with} \quad v_{1} + v_{2} + v_{3} + v_{4}  \leq  \coeffsequence(\sparsity). 
\end{align}
We first bound $\P(\threshsetatom)$ using a union bound over $k$, as well as Hoeffding's inequality and the subgaussian property of the noise%
%For a draw of $\jlmatrix$ we consider the event where it preserves the inner products of the embedded atoms from the generating dictionary $\dico$ and the perturbation dictionary $\pertdico$ up to a distortion $\jldistorterm$, namely, let $\jlgood$ denote the event given by
%\begin{equation}
%\jlgood \!:=\! \big\{ \jlmatrix \!:\! \forall j,k \  \ip{ \jlmatrix \atom_k}{ \jlmatrix \atom_j } \!\lessgtr\! \ip{\atom_k}{\atom_j } \!\pm\! \jldistorterm \ \text{and} \  \ip{ \jlmatrix \atom_k}{ \jlmatrix \unitatom_j } \!\lessgtr\! \ip{\atom_k}{\unitatom_j } \!\pm\! \jldistorterm \big\}
%\label{eq:jl_event} 
%\end{equation}
%and note that from the construction of $\jlmatrix$ based on a scaled orthogonal transform and the inner-product preservation property in \labelcref{eq:jl_event}, we have for $\jlmatrix\in \jlgood$,
%\begin{align}
%\left\| \jlmatrix^{*} \jlmatrix \fsttestvector_k \right\|_2^2 = \frac{\ddim}{\jlembeddingdim} \left\| \jlmatrix \fsttestvector_k \right\|_2^2 \leq \frac{\ddim}{\jlembeddingdim} (1 + \jldistorterm), \qquad \text{for} \quad \fsttestvector_k = \atom_k, \unitatom_k.
%\label{eq:bound_jl_matrix_noise_terms}	
%\end{align}
%We further assume that the embedding dimension is bounded by the noise level and the distortion as
%\begin{equation}
%\jlembeddingdim  \geq 8 \jldistorterm^{-2} \ddim \nsigma^2 .
%\label{eq:jl_embedding_noise_term}	
%\end{equation}
%Using a union bound over all $k$, as well as Hoeffding's inequality and the subgaussian property of the noise, we get for $\P (\threshsetatom)$,
\begin{align}
\allowdisplaybreaks
\P (\threshsetatom) &\leq \sum_{k} \P \Big(\big| \sum_{j \neq k}{ \signsequence(j) \coeffsequence(\permmapping(j)) \ip{ \jlmatrix \atom_k}{ \jlmatrix \atom_j }  } \big| \geq u_1\Big) + \sum_k \P \Big( \left| \ip{ \jlmatrix\atom_k }{\jlmatrix \noise} \right| \geq u_2  \Big) \notag \\
&  \leq \sum_k 2 \exp  \left( \frac{-u_1^2} {2 \sum_{j \neq k}\coeffsequence(\permmapping(j))^2 | \ip{\jlmatrix  \atom_k}{\jlmatrix \atom_j }|^2} \right) +\sum_k 2 \exp  \left( \frac{ - u_2^2 }{2 \nsigma^2 \left\| \jlmatrix^* \jlmatrix \atom_k \right\|_2^2 } \right) \label{eq:union_bound_threshsetatom}
\end{align}
Since we fixed $\jlmatrix \in \jlgood$ we have $|\ip{\jlmatrix  \atom_k}{\jlmatrix \atom_j }| \leq |\ip{\atom_k}{\atom_j }| +\delta \leq \mu+\delta$. Further, 
$\jlmatrix$ was constructed based on an orthogonal transform, so we have
$\| \jlmatrix^* \jlmatrix \atom_k \|_2^2 \leq \frac{d}{m}\| \jlmatrix \atom_k\|_2^2 \leq  \frac{d}{m} (1+\delta)$.
Using these estimates in the bound above, we get
\begin{align}\label{eq:bound_threshsetatom}
\P (\threshsetatom) \leq 2K \exp \left(\frac{-u_1^2} { 2 (\mu +\delta)^2}\right)+ 2K \exp  \left( \frac{ - u_2^2 }{2 \nsigma^2 d \cdot (1+\delta)/m}\right)
\end{align}
and repeating the same steps for $\signsetatom$ we arrive at 
\begin{align}\label{eq:bound_signsetatom}
\P (\signsetatom) \leq 2\sparsity \exp \left(\frac{-v_1^2} { 2 \mu^2}\right)+ 2\sparsity \exp  \left( \frac{ - v_2^2 }{2 \nsigma^2}\right).
\end{align}
For the second event $\threshsetunitatom$ we also use a union bound over all $k$, as well as Hoeffding's inequality and the subgaussian property of the noise to arrive at 
\begin{align}
\allowdisplaybreaks
\P ( \threshsetunitatom) &\leq \sum_k \P \Big( \dicodecompsndvectorcoord_k \big| \sum_j \signsequence(j) \coeffsequence(\permmapping(j)) \ip{\jlmatrix \atom_j}{\jlmatrix \unitatom_k}\big| \geq  u_3\Big) + \sum_k \P \Big( \dicodecompsndvectorcoord_k| \ip{ \jlmatrix \noise_n}{\jlmatrix \unitatom_k }| \geq u_4 \Big)\notag\\
&\leq \sum_k 2 \exp\left(\frac{-u_3^2 }{2\dicodecompsndvectorcoord^2_k \sum_j \coeffsequence(\permmapping(j))^2 |\ip{\jlmatrix \atom_j}{\jlmatrix \unitatom_k}|^2} \right) +\sum_k 2 \exp  \left( \frac{ - u_4^2 }{2 \dicodecompsndvectorcoord^2_k \nsigma^2 \| \jlmatrix^* \jlmatrix \unitatom_k\|_2^2 } \right).
\end{align}
The term $\| \jlmatrix^* \jlmatrix \unitatom_k\|_2^2$ is again bounded by $ \frac{d}{m} (1+\delta)$, while for the sum in the denominator involving the embedded inner products we get,
\begin{align*}
\sum_j \coeffsequence(\permmapping(j))^2 |\ip{\jlmatrix \atom_j}{\jlmatrix \unitatom_k}|^2& \leq \sum_j \coeffsequence(\permmapping(j))^2 \big(|\ip{\atom_j}{\unitatom_k}|^2+2|\ip{\atom_j}{\unitatom_k}| \jldistorterm+ \jldistorterm^2\big)\\
& \leq \coeffsequence(1)^2 \sum_j |\ip{\atom_j}{\unitatom_k}|^2 + 2 \jldistorterm\coeffsequence(1)\sum_j \coeffsequence(\permmapping(j))|\ip{\atom_j}{\unitatom_k}| + \jldistorterm^2\\
& \leq \coeffsequence(1)^2 \|\dico^*\unitatom_k\|^2_2 + 2\delta \coeffsequence(1)\|\dico^*\unitatom_k\|_2 + \jldistorterm^2,
\end{align*}
where for the last step we have used the Cauchy-Schwarz inequality. Since $\|\dico^*\unitatom_k\|^2_2 \leq \|\dico\|^2_{2,2}\|\unitatom_k\|^2_2 \leq B$ and $\dicodecompsndvectorcoord_k\leq \eps$ we get
\begin{align}\label{eq:bound_threshsetunitatom}
\P ( \threshsetunitatom)\leq  2K \exp\left(\frac{-u_3^2 }{ 2\eps^2 (\coeffsequence(1) \sqrt{B} + \delta)^2} \right)+  2K \exp  \left( \frac{ - u_4^2 }{2 \eps^2 \nsigma^2 d \cdot  (1+\delta)/m }\right).
\end{align}
As before we repeat the steps above for $\signsetunitatom$ and, using the bound $\alpha_k/\omega_k \geq  (1-\frac{\eps^2}{2})/\eps$, arrive at
\begin{align}\label{eq:bound_signsetunitatom}
\P ( \signsetunitatom)&\leq  2S \exp\left( \frac{-v_3^2 (1-\frac{\eps^2}{2})^2 }{ 2\eps^2 \coeffsequence(1)^2 B} \right)+  2S\exp  \left( \frac{ - v_4^2 (1-\frac{\eps^2}{2})^2 }{2 \eps^2 \nsigma^2 }\right)
%& \leq  2S \exp\left( \frac{-v_3^2 (1-\frac{1}{32})^2 }{ 2\eps^2 \coeffsequence(1)^2 B} \right)+  2S\exp  \left( \frac{ - 8 v_4^2 (1-\frac{1}{32})^2 }{2 \eps^2 \nsigma^2 }\right)\\
\end{align}
Combining \labelcref{eq:bound_threshsetatom}/\labelcref{eq:bound_signsetatom} with \labelcref{eq:bound_threshsetunitatom}/\labelcref{eq:bound_signsetunitatom} and choosing
\begin{align*}
&u_1 = u_2 = \big(\coeffsequence(\sparsity) (1-\jldistorterm)- \coeffsequence(\sparsity+1)(1+ \jldistorterm)\big)/6 , \nonumber \\
&u_3 = u_1 - 2\eps^2 \coeffsequence(\sparsity)(1-\jldistorterm)/3, \quad u_4 = \eps^2 u_1, \quad \mbox{and} 
\quad v_i = \coeffsequence(\sparsity)/4
\end{align*}
we arrive at
\begin{align*}
\P(\thresholdbad)& \leq 2K \exp \left(\frac{-\big(\coeffsequence(\sparsity) (1-\jldistorterm)- \coeffsequence(\sparsity+1)(1+ \jldistorterm)\big)^2} { 72 (\mu +\delta)^2}\right) + 4K \exp  \left( \frac{ - \big(\coeffsequence(\sparsity) (1-\jldistorterm)- \coeffsequence(\sparsity+1)(1+ \jldistorterm)\big)^2 }{72 \nsigma^2 d \cdot (1+\delta)/m }\right)\\
&\hspace{1em}+ 2e^{\frac{1}{324}} K \exp\left(\frac{- \big(\coeffsequence(\sparsity) (1-\jldistorterm)- \coeffsequence(\sparsity+1)(1+ \jldistorterm)\big)^2 }{ 72\eps^2 (\coeffsequence(1) \sqrt{B} + \delta)^2} \right) + 2S \exp \left(\frac{-\coeffsequence(\sparsity)^2} {32 \mu^2}\right) + 2S \exp \left(\frac{-\coeffsequence(\sparsity)^2} {32 \nsigma^2}\right)  \\ 
& \hspace{2em} + 2S \exp\left( \frac{-\coeffsequence(\sparsity)^2}{ 35\eps^2 \coeffsequence(1)^2 B} \right)+  2S\exp  \left( \frac{ - \coeffsequence(\sparsity)^2 }{35 \eps^2 \nsigma^2 }\right),
\end{align*}
where we have used that $1-\frac{\eps^2}{2} \geq 31/32$ for $\eps<1/4$. We first observe that each exponential with prefactor $\sparsity$ is dominated by an exponential with prefactor $K$ and that we assumed $S\leq K/(98 B) \leq K/ 98$. Next note that we chose the embedding dimension according to \labelcref{eq:embedding_K2points}, thus $m\geq 2\delta^{-2}$, meaning that $(1+\delta)/m \leq \delta^{2} \leq \delta^{2} + \frac{1}{d}$, where the last bound has the advantage of remaining true for $m=\ddim$, when $\jldistorterm = 0$. Further, for $\jldistorterm \geq \coeffsequence(1)/\sqrt{72 \log K}$ already the first term on the right hand side in the inequality above is larger than one, so the bound is trivially true without the third term. On the other hand for $\jldistorterm < \coeffsequence(1)/\sqrt{72 \log K}$ we can bound the denominator in the exponential of the third term by $\coeffsequence(1)^2 (\sqrt{B} + \sqrt{B/72 \log K})^2$, leading to
\begin{align*}
\P(\thresholdbad)& \leq \frac{99}{49} K \exp \left(\frac{-\big(\coeffsequence(\sparsity) - \coeffsequence(\sparsity+1) - \jldistorterm \coeffsequence(\sparsity)- \jldistorterm\coeffsequence(\sparsity+1)\big)^2} { 72 (\mu +\jldistorterm)^2}\right)\\
&\hspace{6em} + \frac{198}{49} K\exp  \left( \frac{ - \big[\coeffsequence(\sparsity) - \coeffsequence(\sparsity+1) - \jldistorterm \coeffsequence(\sparsity)- \jldistorterm\coeffsequence(\sparsity+1)\big]^2 }{72 ( \nsigma^2 +\jldistorterm^2 \ddim \nsigma^2 ) }\right)\\
&\hspace{12em} \qquad+3K \exp\left(\frac{- \big[\coeffsequence(\sparsity) - \coeffsequence(\sparsity+1) - \jldistorterm \coeffsequence(\sparsity)- \jldistorterm\coeffsequence(\sparsity+1)\big]^2 }{ \eps^2 \coeffsequence(1)^2 B \big(\sqrt{72} + 1/\sqrt{\log K}\big)^2} \right).\\
\end{align*}
Using the definitions of $\beta_S$ and $\Delta_S$, the fact that $\coeffsequence(\sparsity+1)<\coeffsequence(\sparsity)\leq \min\{\coeffsequence(1),1/\sqrt{S}\}$ and the assumption $K \geq 98B S \geq 98$ we get
\begin{align*}
\P(\thresholdbad)
&\leq 6.5 K \exp \left(\frac{-\big(\beta_S - \frac{2\jldistorterm}{\sqrt{S}}\big)^2} { 72 \max\{ (\mu +\jldistorterm)^2, \nsigma^2 +\jldistorterm^2\ddim \nsigma^2 \}}\right) + 3K \exp\left(\frac{- \big(\Delta_S - 2\jldistorterm \big)^2 }{ 81 B \eps^2} \right).
\end{align*}
From Lemma A.3 in \cite{sc14} we know that for $a,b,\eps>0$ we have $a \exp(-b^2/\eps^2) \leq \eps$ whenever $\eps \leq b /(\frac{1}{4} + \sqrt{\log(a e^{1/16}/b)})$. Thus \labelcref{eq:condition_eps_convergence_radius} implies that 
\begin{align}
3K \exp\left(\frac{- \big(\Delta_S - 2\jldistorterm \big)^2 }{ 81B \eps^2 } \right) \leq \frac{\snrinputsig \statinputone}{2 \natoms \sqrt{\frameupbound + 1}}  \paramprobestimate \eps
\end{align}
and looking at the definition of $\epsopt$ we get 
\begin{align}
\P(\thresholdbad) \leq  \frac{\snrinputsig \statinputone}{2 \natoms \sqrt{\frameupbound + 1}} ( \epsopt + \paramprobestimate \eps).
\end{align}
Substituting this bound into $\labelcref{eq:bernstein_indicator}$ we get for a fixed $\jlmatrix \in \jlgood$
\begin{align*}
&\P\left(\# \left\{ n : \trainsignal_n \in \thresholdbad \right\} > \frac{\nsig \snrinputsig \statinputone}{2 \natoms \sqrt{\frameupbound + 1}} ( \epsopt + \paramprobestimate \eps + v )\right) \leq \exp \left(  \frac{\nsig \snrinputsig \statinputone v^2}{ 4 K \sqrt{B+1} (\epsopt + \paramprobestimate \eps + v/2)}\right),
\end{align*}
which combined with \labelcref{eq:condition_jlgood} and \labelcref{eq:residual_diff_bound} yields the statement of the Lemma.

\end{proof}
As already mentioned, the proof can be amended to get slightly weaker results also for fast \ac{jl}-embeddings based on the circulant matrices. We will sketch the necessary steps in the following remark.
\begin{remark}[Circulant \ac{jl}-Constructions]
Note that in the case of \ac{jl}-embeddings based on the circulant matrices the operator norm bound $\|\jlmatrix\|^2_{2,2}\leq d/m$, we have used to control the noise terms is no longer valid. However, we can show that with
high probability the operator norm of a circulant matrix, built from a normalized Rademacher vector $\sndtestvector$, which is equivalent to the supremum norm of the \ac{dft} of $\sndtestvector$, is of order $O(\sqrt{\log d})$. This means for $\fsttestvector_k = \unitatom_k, \atom_k$ we have 
\begin{equation}
\left\| \jlmatrix^{*} \jlmatrix \fsttestvector_k \right\|_2^2 \leq \|\jlmatrix\|_{2,2}^2 \|\jlmatrix \fsttestvector_k  \|^2_2 \leq  \frac{(1+\delta) \ddim }{\jlembeddingdim } \cdot O \left(\log \ddim \right),
\label{eq:bound_jl_matrix_noise_terms_rademacher}
\end{equation}
which reduces the admissible noise level by a factor $O \left(\log \ddim \right)$. \\
Further, for the circulant construction from \Cref{def:fast_jl_embedding} the bound in \labelcref{eq:embedding_K2points} becomes 
\begin{equation}
\jlembeddingdim \geq O\left( \jldistorterm^{-2} \cdot \log(\jlnumpoints/\jlprobterm) \cdot \log^2 \bigl( \log(\jlnumpoints/\jlprobterm) \bigr) \cdot \log^2 {\ddim} \right),
\end{equation}
and accordingly the bound in \labelcref{eq:embedding_dimension_conv_proof} is replaced by
\begin{equation}
\jlembeddingdim \geq O\left( \jldistorterm^{-2} \cdot \log(\natoms/\jlprobterm) \cdot \log^2 \bigl( \log(\natoms/\jlprobterm) \bigr) \cdot \log^2 {\ddim} \right).\notag
\end{equation}
\end{remark}

%%%%%%%%%%%%%%%%%%%%%%%%%%%%%%%%%%%%%%%%%%%%%%%%%%%%
%%%%%%%%%%%%%%%%%%%%%%%%%%%%%%%%%%%%%%%%%%%%%%%%%%%%
\bibliography{karinbibtex}

\begin{thebibliography}{10}

\bibitem{trodar11}
E.~A.{-}Castro and Y.~C. Eldar.
\newblock Noise folding in compressed sensing.
\newblock {\em CoRR}, abs/1104.3833, 2011.

\bibitem{aganjaneta13}
A.~Agarwal, A.~Anandkumar, P.~Jain, P.~Netrapalli, and R.~Tandon.
\newblock Learning sparsely used overcomplete dictionaries via alternating
  minimization.
\newblock In {\em COLT 2014 (arXiv:1310.7991)}, 2014.

\bibitem{ahelbr06}
M.~Aharon, M.~Elad, and A.M. Bruckstein.
\newblock {K}-{S}{V}{D}: An algorithm for designing overcomplete dictionaries
  for sparse representation.
\newblock {\em IEEE Transactions on Signal Processing.}, 54(11):4311--4322,
  November 2006.

\bibitem{aich09}
N.~Ailon and B.~Chazelle.
\newblock The fast {J}ohnson-{L}indenstrauss transform and approximate nearest
  neighbours.
\newblock {\em SIAM J. Computing}, 39(1):302--322, 2009.

\bibitem{aili08}
N.~Ailon and E.~Liberty.
\newblock Fast dimension reduction using {R}ademacher series on dual {BCH}
  codes.
\newblock {\em Discrete \& Computational Geometry}, 42(4):615--630, 2008.

\bibitem{argemamo15}
S.~Arora, R.~Ge, T.~Ma, and A.~Moitra.
\newblock Simple, efficient, and neural algorithms for sparse coding.
\newblock In {\em COLT 2015 (arXiv:1503.00778)}, 2015.

\bibitem{argemo13}
S.~Arora, R.~Ge, and A.~Moitra.
\newblock New algorithms for learning incoherent and overcomplete dictionaries.
\newblock In {\em COLT 2014 (arXiv:1308.6273)}, 2014.

\bibitem{bakest14}
B.~Barak, J.A. Kelner, and D.~Steurer.
\newblock Dictionary learning and tensor decomposition via the sum-of-squares
  method.
\newblock In {\em STOC 2015 (arXiv:1407.1543)}, 2015.

\bibitem{cata05}
{E}.~{J}. {C}and{\`e}s and {T}. {T}ao.
\newblock Decoding by linear programming.
\newblock {\em IEEE Transactions on Information Theory}, 51(12):4203--4215,
  2005.

\bibitem{enaahu99}
K.~Engan, S.O. Aase, and J.H. Husoy.
\newblock Method of optimal directions for frame design.
\newblock In {\em ICASSP99}, volume~5, pages 2443 -- 2446, 1999.

\bibitem{olsfield96}
D.J. Field and B.A. Olshausen.
\newblock Emergence of simple-cell receptive field properties by learning a
  sparse code for natural images.
\newblock {\em Nature}, 381:607--609, 1996.

\bibitem{mandar10}
S.~Gleichman and Y.~C. Eldar.
\newblock Blind compressed sensing.
\newblock {\em CoRR}, abs/1002.2586, 2010.

\bibitem{ma04}
M.~Goto.
\newblock Development of the {R}{W}{C} music database.
\newblock In {\em Proc. 18th International Congress on Acoustics (ICA 2004)},
  pages 553--556, 2004.

\bibitem{bagrje14}
R.~Gribonval, R.~Jenatton, and F.~Bach.
\newblock Sparse and spurious: dictionary learning with noise and outliers.
\newblock {\em {I}{E}{E}{E} {T}ransactions on {I}nformation {T}heory},
  61(11):6298--6319, 2015.

\bibitem{grsc10}
R.~Gribonval and K.~Schnass.
\newblock Dictionary identifiability - sparse matrix-factorisation via
  $l_1$-minimisation.
\newblock {\em {IEEE} {T}ransactions on {I}nformation {T}heory},
  56(7):3523--3539, July 2010.

\bibitem{hare17}
I.~Haviv and O.~Regev.
\newblock {\em The restricted isometry property of subsampled fourier
  matrices}, volume 2169 of {\em Lecture Notes in Mathematics}, pages 163--179.
\newblock Springer Verlag, Germany, 2017.

\bibitem{hskazh11}
D.~Hsu, S.M. Kakade, and T.~Zhang.
\newblock A tail inequality for quadratic forms of subgaussian random vectors.
\newblock {\em Electronic Communications in Probability (arXiv:1110.2842)},
  17(14), 2012.

\bibitem{joli84}
W.B. Johnson and J.~Lindenstrauss.
\newblock Extensions of {L}ipschitz mappings into a {H}ilbert space.
\newblock {\em Contemporary Mathematics}, 26:189--206, 1984.

\bibitem{krmera14}
F.~Krahmer, {S}. {M}endelson, and {H}. {R}auhut.
\newblock Suprema of chaos processes and the {R}estricted {I}sometry
  {P}roperty.
\newblock {\em Communications on Pure and Applied Mathematics},
  67(11):1877--1904, 2014.

\bibitem{krwa11}
F.~Krahmer and R.~Ward.
\newblock New and improved {J}ohnson-{L}indenstrauss embeddings via the
  restricted isometry property.
\newblock {\em SIAM Journal on Mathematical Analysis}, 43(3):1269--1281, 2011.

\bibitem{kreutz03}
K.~Kreutz-Delgado, J.F. Murray, B.D. Rao, K.~Engan, T.~Lee, and T.J. Sejnowski.
\newblock Dictionary learning algorithms for sparse representation.
\newblock {\em Neural Computations}, 15(2):349--396, 2003.

\bibitem{leeung99}
D.~D. Lee and H.~S. Seung.
\newblock Learning the parts of objects by nonnegative matrix factorization.
\newblock {\em Nature}, 401, 1999.

\bibitem{lese00}
M.~S. Lewicki and T.~J. Sejnowski.
\newblock Learning overcomplete representations.
\newblock {\em Neural Computations}, 12(2):337--365, 2000.

\bibitem{mabapo12}
J.~Mairal, F.~Bach, and J.~Ponce.
\newblock Task-driven dictionary learning.
\newblock {\em IEEE Transactions on Pattern Analysis and Machine Intelligence},
  34(4):791--804, 2012.

\bibitem{mabaposa10}
J.~Mairal, F.~Bach, J.~Ponce, and G.~Sapiro.
\newblock Online learning for matrix factorization and sparse coding.
\newblock {\em Journal of Machine Learning Research}, 11:19--60, 2010.

\bibitem{memathva16}
A.~Mensch, J.~Mairal, B.~Thirion, and G.~Varoquaux.
\newblock Dictionary learning for massive matrix factorization.
\newblock In {\em ICML (arXiv:1605.00937)}, volume~48, pages 1737--1746, 2016.

\bibitem{memathva18}
A.~Mensch, J.~Mairal, B.~Thirion, and G.~Varoquaux.
\newblock Stochastic subsampling for factorizing huge matrices.
\newblock {\em IEEE Transactions on Signal Processing}, 66(1):113--128, 2018.

\bibitem{nasc18}
V.~Naumova and K.~Schnass.
\newblock Fast dictionary learning from incomplete data.
\newblock {\em EURASIP Journal on Advances in Signal Processing}, 2018(12),
  2018.

\bibitem{leylahmerath06}
M.~D. Plumbley, S.~A. Abdallah, T.~Blumensath, and M.~E. Davies.
\newblock Sparse representations of polyphonic music.
\newblock {\em Signal Process.}, 86(3):417--431, March 2006.

\bibitem{leyathdetvalies10}
M.~D. Plumbley, T.~Blumensath, L.~Daudet, R.~Gribonval, and M.~E. Davies.
\newblock {Sparse Representations in Audio and Music: from Coding to Source
  Separation}.
\newblock {\em {Proceedings of the IEEE.}}, 98(6):995--1005, June 2010.

\bibitem{rubrel10}
R.~Rubinstein, A.~Bruckstein, and M.~Elad.
\newblock Dictionaries for sparse representation modeling.
\newblock {\em Proceedings of the IEEE}, 98(6):1045--1057, 2010.

\bibitem{sc14}
K.~Schnass.
\newblock On the identifiability of overcomplete dictionaries via the
  minimisation principle underlying {K-SVD}.
\newblock {\em Applied Computational Harmonic Analysis}, 37(3):464--491, 2014.

\bibitem{sc14b}
K.~Schnass.
\newblock Local identification of overcomplete dictionaries.
\newblock {\em Journal of Machine Learning Research (arXiv:1401.6354)},
  16(Jun):1211--1242, 2015.

\bibitem{sc15imn}
K.~Schnass.
\newblock A personal introduction to theoretical dictionary learning.
\newblock {\em Internationale Mathematische Nachrichten}, 228:5--15, 2015.

\bibitem{sc15}
K.~Schnass.
\newblock Convergence radius and sample complexity of {ITKM} algorithms for
  dictionary learning.
\newblock {\em accepted to Applied and Computational Harmonic Analysis}, 2016.

\bibitem{sken10}
K.~Skretting and K.~Engan.
\newblock Recursive least squares dictionary learning algorithm.
\newblock {\em {IEEE} {T}ransactions on {S}ignal {P}rocessing},
  58(4):2121--2130, 2010.

\bibitem{dis04}
P.~Smaragdis.
\newblock Non-negative matrix factor deconvolution; extraction of multiple
  sound sources from monophonic inputs.
\newblock In {\em Independent Component Analysis and Blind Signal Separation},
  pages 494--499, 2004.

\bibitem{spwawr12}
D.~Spielman, H.~Wang, and J.~Wright.
\newblock Exact recovery of sparsely-used dictionaries.
\newblock In {\em COLT 2012 (arXiv:1206.5882)}, 2012.

\bibitem{suquwr17a}
J.~Sun, Q.~Qu, and J.~Wright.
\newblock Complete dictionary recovery over the sphere~{I}: Overview and
  geometric picture.
\newblock {\em IEEE Transactions on Information Theory}, 63(2):853--884, 2017.

\bibitem{suquwr17b}
J.~Sun, Q.~Qu, and J.~Wright.
\newblock Complete dictionary recovery over the sphere~{II}: Recovery by
  riemannian trust-region method.
\newblock {\em IEEE Transactions on Information Theory}, 63(2):885--915, 2017.

\bibitem{nen07}
T.~Virtanen.
\newblock {Monaural Sound Source Separation by Nonnegative Matrix Factorization
  With Temporal Continuity and Sparseness Criteria}.
\newblock {\em IEEE Transactions on Audio, Speech and Language Processing},
  15(3):1066--1074, March 2007.

\end{thebibliography}
\bibliographystyle{plain}
%%%%%%%%%%%%%%%%%%%%%%%%%%%%%%%%%%%%%%%%%%%%%%%%%%%%
%%%%%%%%%%%%%%%%%%%%%%%%%%%%%%%%%%%%%%%%%%%%%%%%%%%%

\end{document}